\documentclass[11pt]{article}
\usepackage[inline]{enumitem}
\usepackage{yhmath}

\usepackage{amsthm,amsmath,amssymb,amsbsy,bbm,mathrsfs,supertabular,
eurosym,graphicx,enumitem,xcolor}
\usepackage{mathtools}

\newtheoremstyle{BBstyle0}  {}{}{\itshape}{}{\bfseries}{}{6pt}{}
\newtheoremstyle{BBstyle1}  {3pt}{3pt}{\rmfamily}{}{\itshape}{: }{3pt}{}
\newtheoremstyle{BBstyle2}  {3pt}{3pt}{\itshape}{}{\bfseries\large}{}{0pt}{}
\newtheoremstyle{BBstyle3}  {}{}{\itshape}{}{\bfseries}{: }{3pt}{}
\newtheoremstyle{BBstyle4}  {}{}{\rmfamily}{}{\bfseries}{}{6pt}{}
  
\usepackage[normalem]{ulem}

\newtheorem{ass}{Assumption}

\theoremstyle{definition}

\usepackage[english]{babel}

\newcommand{\norm}[1]{\left\|{#1}\right\|}
\newcommand{\cro}[1]{\left[{#1}\right]}

\newcommand{\argmin}{\mathop{\rm argmin}}

\newcommand{\Var}{\mathop{\rm Var}\nolimits}

\newcommand{\E}{{\mathbb{E}}}

\newcommand{\N}{{\mathbb{N}}}
\renewcommand{\P}{{\mathbb{P}}}
 
\newcommand{\R}{{\mathbb{R}}}

\DeclareMathAlphabet{\mathscrbf}{OMS}{mdugm}{b}{n}

\newcommand{\cA}{{\mathcal{A}}}
\newcommand{\cB}{{\mathcal{B}}}
\newcommand{\cC}{{\mathcal{C}}}
\newcommand{\cD}{{\mathcal{D}}}
\newcommand{\cE}{{\mathcal{E}}}
\newcommand{\cF}{{\mathcal{F}}}
\newcommand{\cG}{{\mathcal{G}}} 
\newcommand{\cH}{{\mathcal{H}}}

\newcommand{\cJ}{{\mathcal{J}}}
\newcommand{\cK}{{\mathcal{K}}}
\newcommand{\cL}{{\mathcal{L}}} 
\newcommand{\cM}{{\mathcal{M}}}
\newcommand{\cN}{{\mathcal{N}}}
\newcommand{\cO}{{\mathcal{O}}}
\newcommand{\cP}{{\mathcal{P}}}
 
\newcommand{\cR}{{\mathcal{R}}}
 
\newcommand{\cT}{{\mathcal{T}}}

\newcommand{\cX}{{\mathcal{X}}}
 
\newcommand{\cZ}{{\mathcal{Z}}}

\newcommand{\gT}{{\mathbf{T}}}

\newcommand{\gW}{{\mathbf{W}}}

\newcommand{\bs}[1]{\boldsymbol{#1}}

\newlist{lista}{enumerate}{1}
\setlist[lista,1]{label=\alph*),ref=\alph*)}

\newlist{listi}{enumerate}{1}
\setlist[listi,1]{label=(\roman*),ref=(\roman*),align=left}

\newcommand{\1}{1\hskip-2.6pt{\rm l}}

\newcommand{\eps}{{\varepsilon}}

\usepackage[left=2.8cm,right=2.8cm, top=4cm, bottom=4cm]{geometry}
\usepackage[utf8]{inputenc}
\usepackage[T1]{fontenc}
\usepackage{lmodern}
\usepackage{yhmath}
\usepackage{subcaption}
\usepackage{xurl}
\usepackage{amsthm,amsmath,amssymb,amsbsy,bbm,mathrsfs,supertabular,
eurosym,graphicx,xcolor}
\usepackage{bm}
\usepackage{mathtools}
\usepackage{float}
\usepackage{tikz}
\usetikzlibrary{arrows.meta, positioning, calc, fit, backgrounds}
\usepackage{tabularx}
\usepackage{booktabs}
\usepackage{anyfontsize}
\usepackage{pifont}
\usepackage{xcolor}
\newcommand{\xmark}{\ding{55}}
\RequirePackage[colorlinks,citecolor=blue,urlcolor=blue]{hyperref}
\newtheoremstyle{BBstyle0}  {}{}{\itshape}{}{\bfseries}{}{6pt}{}
\newtheoremstyle{BBstyle1}  {3pt}{3pt}{\rmfamily}{}{\itshape}{: }{3pt}{}
\newtheoremstyle{BBstyle2}  {3pt}{3pt}{\itshape}{}{\bfseries\large}{}{0pt}{}
\newtheoremstyle{BBstyle3}  {}{}{\itshape}{}{\bfseries}{: }{3pt}{}
\newtheoremstyle{BBstyle4}  {}{}{\rmfamily}{}{\bfseries}{}{6pt}{}
\usepackage[normalem]{ulem}
\usepackage[authoryear]{natbib}

\newtheorem{Theorem}{Theorem}
\newtheorem{theorem}{Theorem}
\newtheorem{lemma}[Theorem]{Lemma}
\newtheorem{Corollary}[Theorem]{Corollary}

\newtheorem{proposition}[Theorem]{Proposition}

\def\bE{\mathbb E}
\def\R{\mathbb R}
\def\bA{\mathbf A}
\def\bD{\mathbf D}
\def\bX{\mathbf X}
\def\bS{\mathbf S}
\def\bT{\mathbf T}

\def\bW{\mathbf W}
\def\bZ{\mathbf Z}
\def\bY{\mathbf Y}

\title{Semi-Supervised Learning on Graphs\\ using Graph Neural Networks\\[1.5cm]}
\author{
Juntong Chen$^*$\quad 
Claire Donnat$^\dagger$\quad 
Olga Klopp$^\ddagger$\quad
Johannes Schmidt-Hieber$^\S$ \\[0.5cm]
$^*$School of Mathematical Sciences, Xiamen University \\
$^\dagger$Department of Statistics, University of Chicago \\
$^\ddagger$ESSEC Business School \\
$^\S$Department of Applied Mathematics, University of Twente
}
\date{}
\begin{document}
\maketitle

\begin{abstract}
Graph neural networks (GNNs) work remarkably well in semi-supervised node regression, yet a rigorous theory explaining when and why they succeed remains lacking. To address this gap, we study an aggregate-and-readout model that encompasses several common message passing architectures: node features are first propagated over the graph then mapped to responses via a nonlinear function.  For least-squares estimation over GNNs with linear graph convolutions and a deep ReLU readout, we prove a sharp non-asymptotic risk bound that separates approximation, stochastic, and optimization errors. The bound makes explicit how performance scales with the fraction of labeled nodes and graph-induced dependence. Approximation guarantees are further derived for graph-smoothing followed by smooth nonlinear readouts, yielding convergence rates that recover classical nonparametric behavior under full supervision while characterizing performance when labels are scarce. Numerical experiments validate our theory, providing a systematic framework for understanding GNN performance and limitations.
\end{abstract}
\section{Introduction}
Graph Neural Networks (GNNs) have become the default tool for semi-supervised prediction on graphs: given a graph $\mathcal{G}=(\mathcal{V},\mathcal{E})$ on $n=|\mathcal{V}|$ nodes with features $X_i$, we observe a response variable $Y_i$ on a subset of nodes, and aim to predict the rest \citep{ma2019flexible, song2022graph,stojanovic2015semi,zhou2019graph}. A central assumption in this setting is that the graph specifies how information propagates across nodes, which, if efficiently leveraged, can substantially boost prediction. GNNs have achieved strong performance for node-level prediction on interaction graphs such as social or hyperlink networks where they predict outcomes like website traffic, future engagement, or satisfaction from partial labels \citep{berg2017graph,deng2022recommender,ma2011recommender}. They are also increasingly used in spatially resolved omics, where nodes are spots or cells connected by spatial neighborhoods and the goal is to predict expensive assays (e.g., gene or protein measurements) from observed modalities by propagating local context through the graph \citep{han2022semi,jiang2021graph}; we present real-data case studies in Section~\ref{experiment-sec}.

Semi-supervised learning on graphs has a rich and long history in the statistics literature.
Classical graph semi-supervised learning is dominated by (i) Laplacian-based regularization \citep{Belkin} and (ii) label propagation \citep{xiaojin2002learning, zhou2003learning}. While both of these approaches have been extensively studied and benefit from solid theoretical guarantees, they focus on spatial regularization, often ignoring node features. 

By contrast, modern GNNs inject node features into propagation and deliver strong empirical performance across domains. One of the key ingredients of their success lies in the use of message-passing layers that perform a localized averaging of features, effectively acting as a learnable low-pass filter on the graph signal \cite{defferrard2016}. The propagation rule enables the model to learn representations that are smooth across the graph and discriminative in their features. \cite{Max} empirically showed that graph convolutional networks (GCNs), one of the earliest types of GNNs, significantly outperform manifold regularization and transductive SVMs, establishing GNNs as the dominant paradigm for graph semi-supervised learning. While a variety of GNN architectures have been proposed (e.g. \cite{pmlr-v97-wu19e,zhu2020beyond}), they largely share the same spirit: all involve an initial aggregation of node information through a message-passing algorithm before its synthesis into an output via a readout step (see Appendix~\ref{app:related_works} for an extended discussion of related works).

Despite their success in various practical applications, a rigorous statistical foundation for GNNs in the semi-supervised regime remains elusive. This theoretical gap is particularly striking given the paradoxical empirical behavior of these models: while GNNs can achieve stable performance with limited labels \cite{oono2019graph}, they simultaneously exhibit high sensitivity to structural perturbations \cite{geisler2021robustness, zugner2018adversarial}. To address this, we study nonparametric semi-supervised node regression under a compositional data-generating mechanism that leverages message passing. We focus on two key problems: (i) how graph-induced propagation affects the effective complexity of the predictor, and (ii) how well the GNN class approximates the underlying regression function when it admits a propagation–nonlinearity compositional form. The main contributions of this work are threefold:
\begin{enumerate}[label=(\roman*)]
\item  \textbf{A sharp oracle inequality for general estimators in the semi-supervised graph regression setting}  (Theorem \ref{main}) that decomposes prediction error into
optimization error $+$ approximation error $+$ stochastic error. Crucially, our analysis explicitly characterizes how these errors depend on the unmasked proportion and the graph’s topology through a single parameter.
\item \textbf{Approximation theory for message passing:} We analyze the approximation capabilities of GNNs with linear graph convolution layers followed by deep ReLU networks. We show that this architecture can approximate functions formed by the composition of a graph-induced propagation step and a Hölder-smooth synthesizing (readout) function (Lemma \ref{overall-approx}).
    \item \textbf{Rates that expose label-scarcity and graph effects.} Combining (i) and (ii), we derive explicit convergence rates for the least-squares GNN estimator (Theorem \ref{converge-rate}). With a properly chosen architecture (depth and width scaling with graph size), the estimator achieves a convergence rate governed by the smoothness of the underlying regression function and its intrinsic input dimension.
\end{enumerate}
Our analysis builds on the oracle-inequality and approximation-theoretic framework developed for sparse deep ReLU networks in classical nonparametric regression, most notably \cite{Schmidt-Hieber}. The semi-supervised graph setting, however, introduces two core challenges that require novel analytical tools. First, responses are observed only on a random subset of nodes; we incorporate this missingness mechanism directly into the risk decomposition to quantify the effect of limited supervision. Second, graph propagation induces nontrivial statistical dependencies because predictions at each node rely on overlapping neighborhoods; we introduce a bounded receptive field assumption to model this graph-structured dependency and, via graph coloring, develop concentration arguments tailored to localized interactions. Combined with new metric entropy bounds for the graph-convolutional component and approximation guarantees for compositions of graph filters and Hölder-smooth functions, this framework yields explicit non-asymptotic risk bounds and convergence rates for least-squares estimation over GNN classes. Notably, this general rate recovers the optimal minimax rate for standard (non-graph) regression in the special case of full supervision and purely local node responses.

This article is structured as follows. In Section~\ref{statis-setting}, we introduce the semi-supervised regression setting with graph-structured data. We analyze GNN-based estimation in Section~\ref{esti-gnns}, where, under a locality condition, we establish an oracle inequality for any given estimator and explicitly characterize the convergence rate of the least squares estimator in terms of the proportion of labeled nodes and the graph’s receptive field. Numerical experiments on both synthetic and real-world datasets are presented in Section~\ref{experiment-sec}. We conclude in Section~\ref{conclusion-sec} and all the related proofs are provided in the Appendix.

{\textit {Notation}}: We set $\mathbb{N}_0=\{0,1,2,\ldots\},$ $\mathbb{N}=\{1,2,\ldots\}$, $\R_{+}=(0,\infty),$ and $[m] =\{1, \ldots, m\}$. In this paper, vectors and matrices are denoted by bold lowercase and uppercase letters, respectively. In specific contexts, we may use bold lowercase letters for fixed matrices to distinguish them from their random counterparts. For a $d_1 \times d_2$ matrix $\mathbf{M}$, define the entry-wise maximum norm\vspace{5pt}
$$\|\mathbf{M}\|_{\infty}=\max_{i\in[d_1],j\in[d_2]}|\mathbf{M}_{i,j}|,$$
the row-sum norm\vspace{5pt}
$$\|\mathbf{M}\|_{1,\infty}=\max_{i\in[d_1]}\sum_{j\in[d_2]}|\mathbf{M}_{i,j}|,$$
and the Frobenius norm 
$$\|\mathbf{M}\|_{\operatorname{F}}=\sqrt{\sum_{i\in[d_1],j\in[d_2]}\mathbf{M}_{i,j}^2}.$$ We denote the $i$-th row of $\mathbf{M}$ by $\mathbf{M}_{i,\cdot}\,$. For $p$-dimensional row or column vectors $\mathbf{v}=(v_1,\ldots,v_p)$, we define $|\mathbf{v}|_{\infty} = \max_{1 \leq i \leq p} |v_i|$ and $|\mathbf{v}|_{1} = \sum_{i=1}^p |v_i|$. For any $a\in\R$, $\lfloor a \rfloor$ denotes the largest integer strictly less than $a$ and $\lceil a \rceil$ the smallest integer greater than or equal to $a$. We use $\log_2$ for the binary logarithm and $\log$ for the natural logarithm. For a set $\cC$, we denote its cardinality by $|\cC|$. For two nonnegative sequences $(\alpha_n)_n$ and $(\beta_n)_n$, we write $\alpha_n \lesssim \beta_n$ if there exists a constant $c$ such that $\alpha_n \leq c\beta_n$ holds for all $n$, and we write $\alpha_n \asymp \beta_n$ when $\alpha_n \lesssim \beta_n\lesssim\alpha_n$. If the function value $h(\bs{x})$ is  a matrix, we denote the $i$-th row of $h(\bs{x})$ by $(h(\bs{x}))_i$ or, if no ambiguity arises, by $h_i(\bs{x})$. For any two real-valued functions $f, g:\cX\to\R$, their sup-norm distance is defined as $\|f - g\|_{L^\infty(\cX)}=\sup_{{\bs x} \in \cX} \, |f({\bs x}) - g({\bs x})|$. For any two vector-valued functions $f=(f_1,\ldots,f_m)^{\top},\ g=(g_1,\ldots,g_m)^{\top}$ where each component function satisfies $f_i, g_i: \mathcal{X}\to \mathbb{R}$, the sup-norm distance between $f$ and $g$ is defined as
$$\|f - g\|_{L^\infty(\cX)}= \max_{1 \leq i \leq m} \, \sup_{{\bs x} \in \cX} \, \big|f_i({\bs x}) - g_i({\bs x})\big|.$$ When the domain $\mathcal{X}$ is clear from context, we simply write $\|\cdot\|_{\infty}$.

\section{Statistical setting}\label{statis-setting}
Consider a graph $\cG=(\mathcal{V},\mathcal{E})$ with vertex set $\mathcal{V}$ and edge set $\mathcal{E}.$ Let $n = |\mathcal{V}|$ denote the number of vertices. We encode the graph structure through its adjacency matrix $\mathbf{A} = (A_{i,j}) \in \{0,1\}^{n \times n}$, defined by
\begin{equation*}
A_{i,j} = \begin{cases}
1, & \text{if } j \in \mathcal{N}(i), \\
0, & \text{otherwise},
\end{cases}
\end{equation*}
where $\mathcal{N}(i) \subseteq \mathcal{V}$ denotes the set of neighbors of node $i$. We define $\widetilde{\mathbf{A}} = \mathbf{A} + \mathbf{I}_n$ as the adjacency matrix with self-loops, and let $\mathbf{D}$ and $\widetilde{\mathbf{D}}$ be the corresponding diagonal degree matrices of $\mathbf{A}$ and $\widetilde{\mathbf{A}}$, respectively. Each node $i \in\{1,\ldots,n\}$ is assumed to have a feature vector $X_i \in \mathcal{X}$, where $\cX$ is a compact subset of $\mathbb{R}^d$. For simplicity, we assume $\mathcal{X} = [0,1]^d$ throughout the paper. We assume that the feature vectors $X_1, \ldots, X_n$ are independent draws from distributions $P_{X_i}$ on $[0,1]^d$. The feature matrix $\mathbf{X} \in [0,1]^{n \times d}$ contains $X_i^\top$ in its $i$-th row. 

In the semi-supervised framework, we observe a random subset of nodes $\Omega \subseteq \mathcal{V}$. For each node $i \in \Omega$, both the feature vector $X_i$ and the response $Y_i \in \mathbb{R}$ are observed. For nodes in the complement $\Omega^c = \mathcal{V} \setminus \Omega$, only the features are available. We assume nodes are included in $\Omega$ independently with probability $\pi \in (0, 1]$, and denote by $$\omega_i \sim \operatorname{Bernoulli}(\pi)$$ the corresponding Bernoulli indicator. The objective is to predict the response values for all nodes in $\Omega^c$. This means that the (training) dataset is
\begin{align}
    \big\{X_1,\ldots, X_n\big\} \cup \{Y_i\}_{i\in \Omega}.
\end{align}
We consider a node-level regression problem on a graph, where responses associated with individual nodes depend not only on their own features but also on the features of neighboring nodes through the graph structure. To capture such structural dependencies, we adopt a compositional modeling framework that separates feature propagation from local prediction.

Formally, for each node $i \in \{1,\ldots,n\}$, the response $Y_i$ is assumed to follow the statistical model
\begin{equation}\label{model}
Y_i = \varphi^{*}\big(\psi_{\mathbf{A},i}^{*}(\mathbf{X})\big) + \varepsilon_i,
\end{equation}
for independent $\varepsilon_i \sim \mathcal{N}(0,1)$ that are also independent of $\mathbf{X}$. The resulting (overall) regression function is given by
\begin{equation}\label{setting-for}
f^* = (f_1^*, \ldots, f_n^*)^{\top}
\quad\mbox{with}\quad
f_i^* = \varphi^{*} \circ \psi_{\mathbf{A},i}^{*},
\end{equation}
and is assumed to admit a two-stage structure: the response is obtained by applying a shared nonlinear map $\varphi^{*}:\mathbb{R}^d\to\mathbb{R}$
to the propagated features. The high-dimensional inner (regression) function
\[
\psi^*_{\mathbf{A}}: [0,1]^{n \times d} \to [-M,M]^{n \times d},
\]
maps to matrices with entries bounded in absolute value by a chosen constant \( M\geq1 \); the subscript \(\mathbf{A}\) indicates the dependence on the graph adjacency matrix. This function aggregates information across the graph to produce propagated node features. Here, $\psi_{\mathbf{A},i}^* :[0,1]^{n \times d} \to [-M,M]^{1 \times d}$ denotes the restriction of $\psi^*_{\mathbf{A}}$ to its $i$-th row. In many GNN architectures, feature propagation is implemented by applying a linear graph filter to node features. Common filters are polynomials in a graph operator, such as the Laplacian $\mathbf{L} = \mathbf{D} - \mathbf{A}$, the adjacency matrix or normalized variants like $\widetilde{\mathbf{D}}^{-1/2} \widetilde{\mathbf{A}} \widetilde{\mathbf{D}}^{-1/2}$ and $\widetilde{\mathbf{D}}^{-1} \widetilde{\mathbf{A}}$ \cite{chung97,Max,10.5555/3504035.3504468}. Further examples are discussed in \cite[Section~6]{Max} and the survey \cite{wu2019survey}. Motivated by these constructions, we assume that the inner regression function  $\psi^*_{\mathbf{A}}$ is close to a matrix polynomial with respect to a chosen graph propagation operator $\bS_\mathbf{A}$. More precisely, let $\mathcal{P}_k(\beta,\bS_{\mathbf{A}})$ denote the set of functions $\psi:[0,1]^{n \times d} \to [-M,M]^{n \times d}$ of the form
\begin{equation}\label{p-k-s-def}
\psi({\bs{x}}) = \sum_{j=1}^k \theta_j \bS^j_{\mathbf{A}}{\bs{x}}, 
\end{equation}
where $\bS_{\mathbf{A}}$ is an $n \times n$ graph propagation operator that leverages the connectivity encoded in $\mathbf{A}$, and the coefficients $\theta_j$ satisfy $|\theta_j| \leq \beta$ for some constant $\beta > 0$. Typically, $\bS_{\mathbf{A}}$ exhibits a low-dimensional structure, such as row-wise sparsity \cite{hwang2023grow,yan2020hygcn}. The $j$-th power of $\bS_{\mathbf{A}}$ characterizes feature propagation over $j$ hops, while the coefficients $\theta_j$ quantify the influence of $j$-hop neighborhoods. To accommodate settings where the propagation may only be approximately polynomial, due to nonlinearities or model mismatch, we introduce the $\rho$-neighborhood
\begin{equation}\label{delta-sum}
\mathcal{F}_{\rho}(\beta,k,\bS_{\mathbf{A}}) = \left\{ f :\cro{0,1}^{n \times d} \to \cro{-M_{\rho},M_{\rho}}^{n \times d} \Big|\inf_{g \in \mathcal{P}_k(\beta,\bS_{\mathbf{A}})} \sup_{{\bs{x}} \in \mathcal{X}^n}\|f({\bs{x}}) - g({\bs{x}})\|_{\infty} \leq \rho \right\},
\end{equation}
with $M_{\rho}=M+\rho$. We then assume that the inner (regression) function satisfies
$$\psi^*_{\mathbf{A}}\in\mathcal{F}_{\rho}(\beta,k,\bS_{\mathbf{A}}).$$
In the special case where $\bS_{\mathbf{A}}$ is symmetric with eigendecomposition $\bS_{\mathbf{A}} = \mathbf{U}^{\top}\mathbf{\Lambda}\mathbf{U}$, any $f \in \mathcal{P}_k(\beta,\bS_{\mathbf{A}})$ can be expressed as
$$f({\bs{x}}) =\mathbf{U}^{\top}\left( \sum_{j=1}^k \theta_j\mathbf{\Lambda}^j \right) \mathbf{U}{\bs{x}}.$$ Consequently, a function $g(\bs{x}) = \mathbf{U}^\top \mathbf{Q} \mathbf{U}\bs{x}$ lies within distance $\rho$ (in sup-norm) of some $f \in \mathcal{P}_k(\beta,\bS_{\mathbf{A}})$ provided that the spectral filter $\mathbf{Q}$ satisfies $\|\mathbf{Q} - P(\mathbf{\Lambda})\|_{\operatorname{F}} \leq \rho/\sqrt{n}$, where $P(\mathbf{\Lambda})=\sum_{j=1}^k \theta_j \mathbf{\Lambda}^j$. 

In particular, when $\rho=0$, the inner regression function simplifies to the exact polynomial form. Consequently, model \eqref{model} resembles a multi-index model, a class of functions extensively studied in recent theoretical deep learning literature \cite{Hajjar2023, mousavi-hosseini2025learning,zbMATH08112234,doi:10.1137/24M1672158}. Let $\mathrm{vec}(\bX) \in \mathbb{R}^{nd}$ denote the vectorized feature matrix $\bX$, and let $\mathbf{V}_i\in\mathbb{R}^{nd\times d}$ denote the linear map induced by $\bS_{\mathbf{A}}$. We can then rewrite the model as
\[
Y_i =\varphi^*\big(\mathbf{V}_i^\top \mathrm{vec}(\mathbf{X})\big) + \varepsilon_i.
\]
This formulation reveals that each response depends on a low-dimensional projection of the global feature matrix, passed through a shared nonlinearity $\varphi^*$. 
However, unlike the classical multi-index model where the projection is typically fixed or unstructured, here the index maps $\{\mathbf{V}_i\}_{i=1}^n$ are node-specific and structurally constrained by the graph topology, reflecting the aggregation of local neighborhood information.

\section{Estimation using GNNs}\label{esti-gnns}
In this section, we employ graph neural networks to approximate the target function and derive a generalization bound for the resulting estimator. The architecture of the considered GNNs consists of two main components: graph convolutional layers followed by deep feedforward neural networks. The precise architecture is specified in detail as follows.

Graph convolutional networks (GCNs) provide a natural approach for approximating the target function $\psi^*_{\mathbf{A}}$ that propagates node feature information. A GCN typically takes the feature matrix ${\bs{x}}\in[0,1]^{n\times d}$ and a prespecified propagation matrix $\bT$ as the input and produces a transformed version of the input feature matrix \cite{Max}. Specifically, let $L$ represent the number of GCN layers. For each $\ell\in\{0,\ldots,L-1\}$, the output of the layers can be defined recursively via
\begin{equation}\label{gcn-def}
H^{(\ell+1)}_{\bT}({\bs{x}})=\sigma_{\operatorname{GCN}}\big(\bT H^{(\ell)}_{\bT}({\bs{x}})\bW_{\ell+1}\big),
\end{equation}
where $H^{(0)}_{\bT} = {\bs{x}}$, $\sigma_{\operatorname{GCN}}$ denotes the activation function, applied element-wise, and $\bW_{\ell+1}$ are $d \times d$ weight matrices, whose entries are learnable parameters. Empirical studies demonstrate that linear GCNs, that is, $\sigma_{\operatorname{GCN}} = \operatorname{id}$, achieve accuracy comparable to their nonlinear counterparts in various downstream tasks \cite{pmlr-v97-wu19e,wang2022powerful}. Theoretically, \cite{oono2019graph} shows that nonlinearity does not enhance GCN expressivity. The linear setting has been further explored under the name poly-GNN in \cite{poly-GNN} and compared to standard GCNs in \cite{gfnn,gfnnICPR}. In particular, for the linear feature propagation defined in \eqref{p-k-s-def}, the identity function is a well-motivated choice for the activation; we therefore set $\sigma_{\operatorname{GCN}} = \operatorname{id}$. Drawing inspiration from residual connections, we aggregate the outputs of all convolutional layers through a weighted sum rather than using only the last layer. This approach of incorporating connections from previous layers preserves multi-scale representations from different neighborhood levels, thus mitigating the over-smoothing phenomenon \cite{chen2020simple,li2021training,zhou2025model}. Accordingly, $\mathcal{G}(L,\bT)$ is defined as the class of functions of the form $\sum_{\ell=1}^{L}\gamma_{\ell} H^{(\ell)}_{\bT}$, that is
\begin{equation}\label{gcn-l-t}
\mathcal{G}(L,\bT) = \left\{\ \bs{x} \mapsto \sum_{\ell=1}^{L}\gamma_{\ell}\bT^\ell \bs{x}\bW_1 \cdots\bW_\ell \;\middle|\;\bW_{\ell} \in [-1,1]^{d\times d}\ \text{and}\ \gamma_{\ell} \in [-1,1]\right\},
\end{equation}
where the coefficients $\gamma_{\ell}$ are reweighting parameters.

Deep feedforward neural networks are commonly applied after graph convolutional layers to synthesize the propagated feature vectors and produce final node-level outputs \cite{Zhou2020GraphNN,wu2019survey,xu2019powerful}. Here we follow the same paradigm and refer to this component as deep neural networks (DNNs). Let the width vector $\bs{p} = (p_0, p_1, \ldots, p_{L+1}) \in \N^{L+2}$ satisfy $p_0 = d$ and $p_{L+1} = 1$, and let $\operatorname{ReLU}(u) = \max\{0, u\}$ denote the ReLU activation function. We define $\cF{(L,\bs{p})}$ as the class of deep ReLU neural networks of depth $L \in \N_0$ and width vector $\bs{p}$, comprising all functions $f:\R^{d}\rightarrow\R$ of the form
\begin{equation}\label{neural-network}
f({\bm{x}})={\bs\Theta}_{L}\circ\operatorname{ReLU}\circ \hspace{2pt}{\bs\Theta}_{L-1}\circ\cdots\circ\operatorname{ReLU}\circ\hspace{2pt}{\bs\Theta}_{0}({\bs{x}}),
\end{equation}
where, for $\ell=0,\ldots,L$, $${\bs\Theta}_{\ell}({\bs{y}})={\mathbf M}_{\ell}{\bs{y}}+{\bs b}_{\ell}.$$ Here, ${\mathbf M}_{\ell}$ is a $p_{\ell}\times p_{\ell+1}$ weight matrix, ${\bs b}_{\ell}$ is a bias vector of size $p_{\ell+1}$, and the $\operatorname{ReLU}$ activation function is applied component-wise to any given vector. Again, we assume that all entries of the weight matrices and bias vectors lie within $[-1,1]$. In practice, sparsity in neural networks is often encouraged through techniques such as regularization or specialized architectures \cite{goodfellow2016deep}. A notable example is dropout, which promotes sparse activation patterns by randomly deactivating units during training, thereby ensuring that each neuron is active only for a small subset of the training data \cite{2014dropout}. In line with the framework introduced by \cite{Schmidt-Hieber}, we explicitly enforce parameter sparsity by restricting the network to utilize only a limited number of non-zero parameters. More precisely, let $\|{\mathbf M}_{\ell}\|_0$ and $|{\bs b}_{\ell}|_0$ denote the number of non-zero entries in the weight matrix ${\mathbf M}_{\ell}$ and bias vector ${\bs b}_{\ell}$, respectively. For any $s\geq1$ and $F>0$, we define the class of $s$-sparse DNNs, truncated to the range $[-F,F]$, as
\begin{equation}
\cF(L,{\bs p},s,F)=\left\{(f\vee-F)\wedge F:\;f\in\cF(L,{\bs p}),\ \sum_{\ell=0}^L\|{\mathbf M}_{\ell}\|_0+|{\bs b}_{\ell}|_0\leq s\right\}.
\end{equation}

In conclusion, the considered graph neural networks class is given by
\begin{equation}\label{gnn-models}
\cF(\bT,L_1,L_2,{\bs{p}},s,F)=\Big\{f:\;f_i(\bs{x})=h\big(g_i(\bs{x})\big),\ h\in\mathcal{F}{(L_2,{\bs{p}},s,F)},\ g\in\cG{(L_1,\bT)}\Big\}, 
\end{equation}
where $L_1,L_2\in\mathbb{N}_0$ are the respective depths of the convolutional and feedforward layers, and $\bT$ is a pre-specified convolutional operator. Figure~\ref{fig:gnn-graph} provides a schematic explanation of the constituent blocks. 
\begin{figure}[t]
\centering
\resizebox{\linewidth}{!}{%
\begin{tikzpicture}[
  every node/.style={font=\small},
  line cap=round, line join=round,
  >=Stealth,
  ink/.style={draw=black!70},
  faint/.style={draw=black!20},
  softfill/.style={fill=black!3},
  gedge/.style={ink, line width=0.85pt},
  hnode/.style={draw=cyan!80!black, fill=cyan!20, line width=1.2pt},
  hedge/.style={draw=cyan!80!black, line width=1.2pt},
  dnodeO/.style={circle, draw=orange!80!black, fill=orange!40, line width=1pt, minimum size=7.5mm},
  dnodeH/.style={circle, draw=orange!70!black, fill=orange!15, line width=0.8pt, minimum size=7.5mm},
  dnodeX/.style={circle, draw=orange!60!black, fill=orange!5,  line width=0.8pt, minimum size=7mm},
  dedge/.style={draw=orange!50!black, line width=0.6pt, opacity=0.6},
  outerbox/.style={ink, rounded corners=12pt, line width=0.9pt, softfill},
  innerbox/.style={draw=black!30, rounded corners=9pt, line width=0.7pt, densely dotted},
  gnodeL/.style={circle, ink, line width=0.9pt, minimum size=7.5mm, fill=black!15},
  gnodeU/.style={circle, ink, line width=0.9pt, minimum size=7.5mm, fill=black!5},
  lab/.style={black!85},
  smalllab/.style={font=\footnotesize, black!80, inner sep=0.5pt}
]

\def\gshift{3.35}   
\def\dshift{4.15}   

\def\pA{(-1.5, 1.3)} \def\pB{(1.5, 1.1)} \def\pC{(-0.5, 0.2)}
\def\pD{(1.4,-0.3)}  \def\pE{(-1.5,-1.3)} \def\pF{(0.1,-0.9)}
\def\pG{(2.25, 1.75)}

\begin{scope}[shift={(-\gshift,0)}, local bounding box=origGraph]
  \node[gnodeL, label={[smalllab, shift={(0.1,0.1)}]above left:$(X_3, Y_3)$}] (r1) at \pA {};
  \node[gnodeL, label={[smalllab, xshift=-1mm]left:$(X_1, Y_1)$}] (r3) at \pC {};

  \node[gnodeU, label={[smalllab]above left:$X_2$}] (r2) at \pB {};
  \node[gnodeU, label={[smalllab]right:$X_4$}]      (r4) at \pD {};
  \node[gnodeU, label={[smalllab]below left:$X_5$}]  (r5) at \pE {};
  \node[gnodeU, label={[smalllab]below right:$X_6$}]       (r6) at \pF {};

  \node[gnodeU, label={[smalllab]above right:$X_7$}] (r7) at \pG {};

  \draw[gedge] (r3)--(r1) (r3)--(r2) (r3)--(r4) (r3)--(r6) (r1)--(r2) (r5)--(r3);
  \draw[gedge] (r2)--(r7); 
\end{scope}

\node[innerbox, fit=(origGraph), inner sep=6pt] (gcnRight) {};

\begin{scope}[shift={(+\gshift,0)}, local bounding box=propGraph]
  \node[gnodeU] (l5) at \pE {};
  \node[gnodeL, hnode] (l3) at \pC {};
  \node[smalllab, left=1.8mm of l3, black, font=\bfseries] {$(X'_1, Y_1)$};

  \node[gnodeL] (l1) at \pA {};
  \node[gnodeU] (l2) at \pB {};
  \node[gnodeU] (l4) at \pD {};
  \node[gnodeU] (l6) at \pF {};

  \node[gnodeU, faint] (l7) at \pG {};
  \draw[gedge, faint] (l2)--(l7); 

  \draw[hedge, <-] (l3) -- (l1);
  \draw[hedge, <-] (l3) -- (l2);
  \draw[hedge, <-] (l3) -- (l4);
  \draw[hedge, <-] (l3) -- (l6);
  \draw[hedge, <-] (l3) -- (l5);

  \draw[gedge, faint] (l1) -- (l2);
\end{scope}

\node[innerbox, fit=(propGraph), inner sep=6pt] (gcnLeft) {};

\draw[line width=1.25pt, -Stealth, black!60]
  (gcnRight.east) -- (gcnLeft.west);

\begin{pgfonlayer}{background}
  \node[outerbox, fit=(gcnRight)(gcnLeft), inner sep=9pt] (gcnBox) {};
\end{pgfonlayer}

\node[lab, font=\bfseries] at (gcnBox.south) [below=2mm]
  {GCN Component (Feature Propagation)};

\begin{scope}[shift={($(gcnBox.east)+(\dshift,0)$)}, local bounding box=dnnContent]
  \coordinate (dX) at (0,0);
  \coordinate (dH) at (2.15,0);
  \coordinate (dO) at (4.3,0);

  \node[dnodeO] (fhat) at (dO) {};
  \node[lab, right=2mm of fhat] {$\widehat{Y}_1$};

  \foreach \y/\name in {1.35/h1, 0.45/h2, -0.45/h3, -1.35/h4} {
    \node[dnodeH] (\name) at ($(dH)+(0,\y)$) {};
    \draw[dedge] (\name) -- (fhat);
  }

  \node[dnodeX] (x1) at ($(dX)+(0,1.65)$) {};
  \node[dnodeX] (x2) at ($(dX)+(0,0.65)$) {};
  \node[dnodeX] (x3) at ($(dX)+(0,-0.65)$) {};
  \node[dnodeX] (xd) at ($(dX)+(0,-1.65)$) {};

  \node[lab, left=1mm of x1] {$X'_{11}$};
  \node[lab, left=1mm of xd] {$X'_{1d}$};
  \node at ($(dX)+(0,0.1)$) {$\vdots$};

  \foreach \h in {h1,h2,h3,h4} {
    \foreach \x in {x1,x2,x3,xd} {
      \draw[dedge] (\x) -- (\h);
    }
  }
\end{scope}

\begin{pgfonlayer}{background}
  \node[outerbox, fit=(dnnContent), inner sep=9pt] (dnnBox) {};
\end{pgfonlayer}

\node[lab, font=\bfseries] at (dnnBox.south) [below=2mm]
  {DNN Component (Prediction)};

\draw[
  line width=1.15pt,
  double distance=2.2pt,
  -Stealth,
  black
] (gcnBox.east) -- (dnnBox.west);

\node[lab, font=\footnotesize\bfseries, align=center]
  at ($(gcnBox.east)!0.5!(dnnBox.west)$) [above=2pt]
  {Input Vector\\$X'_1$};

\end{tikzpicture}%
}
\caption{Graph feature propagation followed by a nonlinear readout: a linear message-passing block generates propagated features, which are then mapped to node-level predictions via a ReLU DNN.}
\label{fig:gnn-graph}
\end{figure}
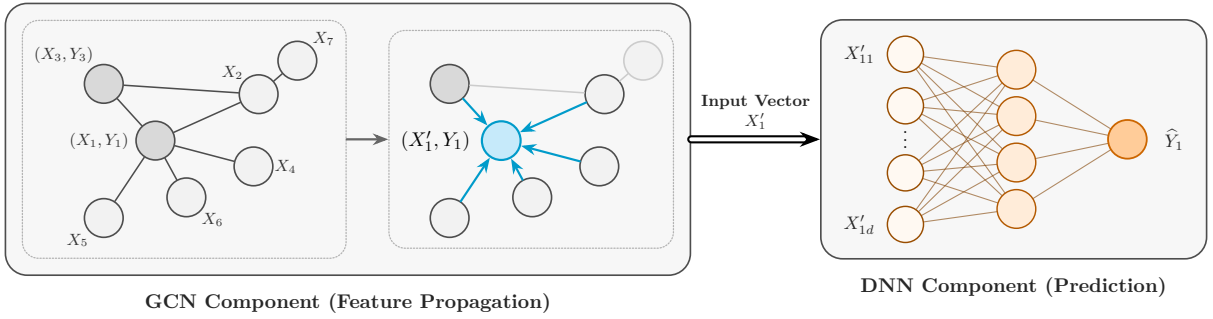

Any estimator of the regression function is based on the training data $\{X_1,\ldots,X_n\}\cup\{Y_i\}_{i\in \Omega}$ with $\Omega$ the random set of unmasked response variables. The statistical performance of an estimator $\widetilde f=(\widetilde f_1,\ldots,\widetilde f_n)^{\top}:[0,1]^{n\times d} \to \R^n$ is measured by the prediction error
\begin{equation}
\cR\left(\widetilde f,f^*\right)=\E\cro{\frac{1}{n}\sum_{i=1}^n\left(\widetilde{f}_i(\bX')-\varphi^{*}\big(\psi_{\mathbf{A},i}^*(\bX')\big)\right)^2},    
\end{equation}
where $\mathbf {X'}$ is an independent copy of $\mathbf {X}$ and the expectation is taken over all the randomness in the model, that is, the randomness induced by $\Omega\cup\{(X_i,X_i',\eps_i)\}_{i=1}^{n}$. 

Among various estimators, we are primarily interested in analyzing the performance of the least-squares estimator over the class $\cF(\bT,L_1,L_2,{\bs{p}},s,F)$ constructed from the unmasked nodes in $\Omega\subseteq\{1,\ldots,n\}$, that is, 
\begin{equation}\label{lse-def}
\widehat f\in\argmin_{f\in\cF(\bT,L_1,L_2,{\bs{p}},s,F)}\frac{1}{n}\sum_{i\in \Omega}\big(Y_{i}-f_i(\mathbf X)\big)^{2}.
\end{equation}
By the definition of $\cF(\bT,L_1,L_2,{\bs{p}},s,F)$, this estimator admits a decomposition satisfying for all $i\in\Omega,$
\begin{equation*}
\widehat f_i(\bX)=\widehat{\varphi}\circ\widehat{\psi}_{i}(\bX),
\end{equation*}
where $\widehat{\varphi}\in\mathcal{F}{(L_2,{\bs{p}},s,F)}$ and $\widehat{\psi}\in\cG{(L_1,\bT)}.$

\subsection{Oracle-type inequality}
We begin by introducing the prerequisites. Let ${\bs\varepsilon}=(\varepsilon_1,\ldots,\varepsilon_n)^{\top}$ and ${\bs\omega}=(\omega_1,\ldots,\omega_n)^{\top}$. For any estimator $\widetilde{f}$ returning a network in $\mathcal{F}$ with output $\widetilde f_i$ at the $i$-th node, the optimization error is defined as
\begin{align}
\Delta_{n}^{\cF}\left(\widetilde{f}, f^{*}\right)=\E_{{\bs\varepsilon},{\bs\omega},\bX}\left[\frac{1}{n} \sum_{i\in \Omega}\Big(Y_{i}-\widetilde{f}_i(\mathbf X)\Big)^{2}-\inf_{f \in \mathcal{F}} \frac{1}{n} \sum_{i\in\Omega}\Big(Y_{i}-{f}_i(\mathbf X)\Big)^{2}\right].\label{def-delta-em}
\end{align}
It measures the expected difference between the training loss of $\widetilde{f}$ and the training loss of the global minimum over all network fits in the class $\cF$, where the data $\{(X_i,Y_i)\}$ are generated according to the true regression function $f^{*}$ in \eqref{setting-for}. The optimization error is nonnegative and vanishes if $\widetilde f$ is an
empirical risk minimizer over $\cF$. 

The complexity of the vector-valued function class $\cF$ is measured by its metric entropy. Recall that for any two functions $f, g: \mathcal{X}^n\to \mathbb{R}^n$, with $f = (f_1, \ldots, f_n)^\top$ and $g = (g_1, \ldots, g_n)^\top$, their sup-norm distance is given by $\|f - g\|_{\infty} =\max_{1 \leq i \leq n} \, \|f_i-g_i\|_{\infty}.$ For any $\delta>0$, we say a class of functions $\mathcal{F}'$ is a $\delta$-cover of $\mathcal{F}$ with respect to the sup-norm, if for any $f\in\mathcal{F}$, there exists a function $g\in\mathcal{F}'$ such that $\|f - g\|_{\infty}\leq\delta$. We denote the smallest $\delta$-cover of $\cF$ by $\mathcal{F}_{\delta}$. The $\delta$-covering number $\mathcal{N}(\delta, \mathcal{F},\|\cdot\|_{\infty})$ is the cardinality of $\mathcal{F}_{\delta}$. If no finite $\delta$-cover exists, the covering number is defined to be infinite. The metric entropy (or simply entropy) of $\mathcal{F}$ at scale $\delta>0$ is $\log\mathcal{N}(\delta, \mathcal{F}, \|\cdot\|_{\infty})$. For brevity, we simplify this notation in some contexts to $\log\mathcal{N}_{\delta}$. A key insight from statistical learning theory is that the stochastic error depends critically on this measure of model complexity \cite{van-exp,Bartlett_2005,Schmidt-Hieber}. 

A key technical difference with standard nonparametric regression is that, even if the node
features $\{X_i\}_{i=1}^n$ are sampled independently and the noises $\{\varepsilon_i\}_{i=1}^n$
are independent, the empirical loss is not a sum of independent terms once prediction uses
graph propagation. Indeed, for a general graph-based predictor $f\in\mathcal F$, each nodal
output $f_i(\bX)$ may depend on the features of many nodes through message passing, so the squared error contributions from each node $i$,
\[
s_i^{(f)}(\bX)= \bigl(f_i(\bX)-f_i^*(\bX)\bigr)^2
\]
can be statistically coupled through shared coordinates of $\bX$. To quantify this dependence in the analysis, we formalize a bounded-influence property: each nodal loss $s_i^{(f)}(\bX)$ depends on features from only a limited number of nodes, while the features from any given node affect at most a bounded number of loss terms. The following assumption captures this via a single parameter $m$, which will appear as a multiplicative factor in the oracle inequality and can be interpreted as an upper bound on the effective receptive-field size under graph propagation. 

We now state this locality condition formally as an assumption on the true regression function $f^*$ and the base class $\cF$.
\begin{ass}\label{ass-m}
Let $m$ be a positive integer. For any $f= (f_1, \ldots, f_n)^{\top}\in \cF \subset \{g:\R^{n \times d}\to \R^n\}$, $f^*= (f_1^*, \ldots, f_n^*)^{\top},$ and any argument $\bs{x} \in \R^{n \times d}$, assume that 
\begin{enumerate}[label=(\roman*), itemsep=0pt, topsep=2pt, parsep=0pt]
\item each function $\bs{x}\mapsto s_i^{(f)}(\bs{x}) = \big(f_i(\bs{x}) - f^*_i(\bs{x})\big)^2$ with $i\in [n]$ depends on at most $m$ coordinates of $\bs{x} = (x_1, \ldots, x_n)^{\top}$;
\item each coordinate $x_j$ influences at most $m$ functions $s_i^{(f)}$.
\end{enumerate}
\end{ass}
This assumption naturally arises from common graph propagation mechanisms and GNN architectures. For instance, consider a $k$-step polynomial filter $f({\bs x})=\sum_{j=1}^k \theta_{j} \bS^{j} {\bs x}$, where the propagation matrix $\bS$ is symmetric and corresponds to an undirected graph of maximum degree $\Delta$ (with self-loops). In this case, each output $f_i(x)$ depends only on features within $k$ hops of node $i$, and one may choose $m=\Delta^k$. Similarly, for a depth-$L_1$ linear GCN with propagation matrix $\bT$ (unit diagonal) and with at most $m_{\bT}$ non-zero entries per row and per column, each nodal output depends on at most $m_{\bT}^{L_1}$
input feature vectors; hence we may take $m=m_{\bT}^{L_1}$.

Building upon the existing node graph, Assumption~\ref{ass-m} allows us to define a new dependency graph for $\{s_i^{(f)}\}_{i=1}^n$. An edge between nodes $i$ and $j$ in this dependency graph exists when $s_i^{(f)}$ and $s_j^{(f)}$ share a common input coordinate. Under Assumption~\ref{ass-m}, this graph has maximum degree at most $m(m-1)$. A direct implication of this bounded degree is the existence of a (disjoint) partition $\mathcal{P}_1 \cup \cdots \cup \mathcal{P}_r = \{1,\ldots,n\}$ with $r \leq m(m-1)+1$. Here, each part $\mathcal{P}_\ell$ comprises indices such that for any distinct $i, j \in \mathcal{P}_\ell$, the losses $s_i^{(f)}(\mathbf{X})$ and $s_j^{(f)}(\mathbf{X})$ share no covariate vector. This key insight makes it possible to derive tight concentration bounds for functions exhibiting this type of sparse local dependence. 

The following theorem states an oracle-type inequality that holds for any estimator $\widetilde{f}$ constructed from a general model class $\mathcal{F}$.

\begin{theorem}\label{main}
Assume Assumption~\ref{ass-m} holds with $m\geq1$. For $0<\delta \leq 1$, let $\mathcal{F}_\delta$ be a $\delta$-cover of $\mathcal{F}$ whose entropy satisfies $\log\mathcal{N}_\delta \geq 1$. Suppose that there exists a constant $F\geq1$ such that $\|f_i\|_{\infty} \leq F$ for all $i\in[n]$ and all $f \in \mathcal{F} \cup \mathcal{F}_{\delta} \cup\{f^*\}$. Then, for any $\varepsilon \in (0,1]$,
\begin{align*}
(1-\varepsilon)^2{\Delta}_n^{\mathcal{F}}&(\widetilde f,f^*)-E_1(\varepsilon, n, \delta)\leq\mathcal{R}(\widetilde{f}, f^*)\\
& \leq 2(1+\varepsilon)^2 \left( \inf_{f \in \mathcal{F}} \mathbb{E}\cro{\frac{1}{n}\sum_{i=1}^n\big(f_i(\bX) - f^*_i(\bX)\big)^2} + \frac{{\Delta}_n^{\mathcal{F}}(\widetilde f,f^*)}{\pi} +E_2(\varepsilon, n,\pi, \delta) \right),
\end{align*}
where
\begin{align*}
E_1(\varepsilon, n, \delta) &=C_1\left(\frac{m^2F^2 \log\mathcal{N}_{\delta}}{\varepsilon n}+\delta F\right), \\
E_2(\varepsilon, n, \pi, \delta) &=C_2\cro{\frac{(1+\varepsilon)m^2F^2 \log\mathcal{N}_{\delta}}{\varepsilon n\pi} + \frac{F\delta}{\sqrt{\pi}}+\frac{F^2}{\mathcal{N}_{\delta}}},
\end{align*}   
with $C_1,C_2>0$ universal constants.
\end{theorem}
The proof of Theorem~\ref{main} is postponed to Section~\ref{main-proof}. The result indicates that the upper bound for the prediction error of an arbitrary estimator $\widetilde{f}$ based on model $\cF$ can be decomposed into three main terms: the approximation error between $\mathcal{F}$ and the target $f^*$, the stochastic error $E_2(\varepsilon, n,\pi, \delta)$ governed by the complexity of $\cF$, and the optimization error ${\Delta}_n^{\cF}(\widetilde{f}, f^*)$, which quantifies the training error discrepancy between the chosen estimator and the empirical risk minimizer. When $\pi=1$ (full supervision) and provided the covering number $\mathcal{N}_{\delta}$ is not too small, our result recovers the same dependence on $n$ as in standard regression \cite{Schmidt-Hieber} (Theorem~2). The bounds differ by a multiplicative factor of $m^2$, which accounts for the graph-induced dependency.

The distinct roles of $m$ and $\pi$ in Theorem~\ref{main} shed light on the seemingly contradictory phenomena observed in practice. On one hand, the upper bound is scaled by the expected proportion of observed sample pairs $\pi$; when $\pi$ remains constant, even if small, the prediction error increases only by a constant multiplicative factor. This explains why GNNs can achieve satisfactory generalization with limited supervision \citep{10.5555/3504035.3504468,Max}. On the other hand, structural perturbations such as edge rewiring or dropping can fundamentally alter the graph geometry, resulting in performance instability \citep{geisler2021robustness, zugner2018adversarial}. In our bound, this sensitivity is quantified by the receptive field size $m$, which depends intricately on the network's expansion properties. For instance, in a ring lattice, $m$ grows only linearly with the number of propagation steps. However, adding a few ``shortcut'' edges triggers a phase transition to a small-world regime, where the neighborhood size $m$ may shift from linear to exponential expansion \citep{watts1998collective}. Such structural shifts significantly relax the generalization bounds, thereby accounting for the high sensitivity of GNNs to topological perturbations that bridge distant clusters. 

In what follows, we provide a more precise analysis focusing on the proposed GNN class $\mathcal{F}(\bT,L_1,L_2,\bs{p},s,F)$. The result below establishes an upper bound for the metric entropy of $\mathcal{F}(\bT,L_1,L_2,\bs{p},s,\infty)$. Because the inequality $$\big\|(f\vee-F)\wedge F-(g\vee-F)\wedge F\big\|_{\infty}\leq\|f-g\|_{\infty}$$ holds, the same entropy bound immediately extends to the bounded class $\mathcal{F}(\bT,L_1,L_2,\bs{p},s,F)$.

Recall that for a matrix $\mathbf{M}=(M_{i,j})$, the row-sum norm is $\|\mathbf{M}\|_{1,\infty}=\max_i \sum_j |M_{i,j}|.$

\begin{proposition}\label{cover-whole}
Let $L_1\geq1$. For any $0<\delta\leq1$,
\begin{align*}
&\log\mathcal{N}\big(\delta,\mathcal{F}(\bT, L_1, L_2, {\bs p}, s,\infty),\|\cdot\|_{\infty}\big)\hspace{-1pt}\leq\hspace{-1pt} \big(d^2L_1+L_1+s+1\big)\log\left(\hspace{-1pt}\frac{\cL(L_1,L_2)}{\delta}\prod_{k=0}^{L_2+1}(p_{k}+1)^2\hspace{-2pt}\right),
\end{align*}
where $\cL(L_1,L_2)=2L_1(L_1+L_2+2)(\|\bT\|_{1,\infty}\vee1)^{L_1}d^{L_1}$.
\end{proposition}
The proof of Proposition~\ref{cover-whole} is deferred to Section~\ref{proof-covering}. Observe that the considered GNNs have a total of $(d^2+1)L_1 + s$ trainable parameters. In practice, GCNs employ few layers, so $L_1$ is typically independent of $n$, and the feature dimension $d$ is also commonly assumed to be finite \cite{song2021scgcn,Ritter2025NePSTA}. When the widths $p_k$ of the deep ReLU networks are bounded by order $n^{\gamma}$ and $n$ is sufficiently large, the logarithmic term simplifies to $\lesssim L_2 \log[(n^{\gamma}\vee L_2)/\delta]$ for any $\delta\leq1$. Consequently, up to a constant factor depending on $\|\bT\|_{1,\infty}$, $d$, and $L_1$, the metric entropy of $\mathcal{F}(\bT, L_1, L_2, \bs{p}, s,F)$ at scale $\delta\leq1$ can be bounded by $\lesssim sL_2\log[(n^{\gamma}\vee L_2)/\delta]$. Therefore, the depth and sparsity of the DNN component play an essential role in determining the entropy. Based on Theorem~\ref{main} and Proposition~\ref{cover-whole}, we obtain the following result.

\begin{Corollary}\label{kappa-b}
Suppose Assumption~\ref{ass-m} holds with $m\geq1$, and the true regression function $f^*$ satisfies $\|f^*_i\|_{\infty}\leq F$ for all $i\in[n]$ and some $F\geq1$. Let $\widetilde f$ be any estimator in the class $\overline\cF=\cF(\bT, L_1, L_2, \bs{p}, s,F),$ where $s \geq 2$ and $L_1,L_2 \geq 1.$ Assume $\cN\big(1/n, \overline\cF, \norm{\cdot}_{\infty}\big) \gtrsim n.$
Define
\[
\kappa_n = (d^2 L_1 + s)\cro{\log\big( n L_1(L_1 + L_2)\big)+(L_1+1)\log\big(\|\bT\|_{1,\infty}\vee d\big)+L_2\log s}.
\]
Then, for any $\varepsilon \in (0,1]$, there exists $C_\varepsilon > 0$ depending only on $\varepsilon$ such that 
\begin{align*}
(1-\varepsilon)^2&\Delta_n^{\overline\cF}(\widetilde{f},f^*)-C_{\varepsilon}m^2F^2\frac{\kappa_{n}}{n}\leq \mathcal{R}(\widetilde{f}, f^*)\\ 
&\leq 2(1+\varepsilon)^2\cro{\inf_{f \in \mathcal{F}(\bT, L_1, L_2, {\bs p}, s,F)} \|f - f^*\|_{\infty}^2+\frac{\Delta_n^{\overline\cF}(\widetilde{f},f^*)}{\pi}}+ C_{\varepsilon}\frac{m^2F^2}{\pi}\frac{\kappa_{n}}{n}.
\end{align*}
\end{Corollary}
The proof of Corollary~\ref{kappa-b} is postponed to Section~\ref{main-proof}. Theorem~\ref{main} and Corollary~\ref{kappa-b} provide not only an upper bound but also a lower bound
on the prediction error in terms of the optimization error
$\Delta^{\overline\cF}_n(\widetilde f,f^*)$.
In particular, Corollary~2 implies that for any $\varepsilon\in(0,1]$,
\[
\cR(\widetilde f,f^*)\ \geq\ (1-\varepsilon)^2\,\Delta^{\overline\cF}_n(\widetilde f,f^*)
\;-\; C_\varepsilon\, m^2 F^2\frac{\kappa_n}{n},
\]
for a positive constant $C_\varepsilon$.
Consequently, the optimization error acts as a floor for the achievable prediction
accuracy: unless the training algorithm returns an approximate empirical risk minimizer with
$\Delta^{\overline\cF}_n(\widetilde f,f^*)\lesssim m^2F^2\kappa_n/n$, the prediction risk cannot achieve the statistical rate implied by the model complexity. This highlights an explicit computational--statistical trade-off in our setting: to benefit from
the statistical guarantees, the optimization error must be driven below the intrinsic
``statistical resolution'' of the function class.

When $\widetilde{f} = \widehat{f}$ is the empirical minimizer defined as in \eqref{lse-def}, the optimization error $\Delta_n^{\overline\cF}(\widehat{f}, f^*)$ is zero and the error upper bound reveals a bias-variance trade-off between the approximation error of $f^*$ based on model $\mathcal{F}(\bT, L_1, L_2, \bs{p}, s,F)$ and the variance terms depending on the complexity of $\mathcal{F}(\bT, L_1, L_2, \bs{p}, s,F)$. For both GCN and DNN components, stacking too many layers worsens the upper bound, as this is reflected in the term $\kappa_{n}$. 
\subsection{Approximation with GNNs}
This section analyzes the approximation capability of the class $\mathcal{F}(\bT,L_1,L_2,\bs{p},s,F)$ defined in \eqref{gnn-models}, which ultimately provides the approximation error $\inf_{f \in \mathcal{F}(\bT, L_1, L_2, \bs{p}, s,F)}\|f - f^*\|_{\infty}$ in Corollary~\ref{kappa-b}.

We first consider the approximation of the target class $\mathcal{F}_{0}(\beta,k,\bT)$ by the GCN class $\mathcal{G}(L_1,\bT)$.

\begin{lemma}\label{gcn-approx}
Let $\cF_{0}(\beta,k,\bT)$ be the function class defined in \eqref{delta-sum}. If $L_1\geq k,$ and $\beta\leq1$, then $\cF_{0}(\beta,k,\bT)\subseteq \cG(L_1,\bT).$
\end{lemma}

The proof of Lemma~\ref{gcn-approx} is deferred to Section~\ref{appro-proofs}. The above result demonstrates that, when the propagation operators coincide, GCNs with sufficient depth and parameters constrained to $[-1,1]$ contain the function class $\mathcal{F}_{0}(\beta,k,\bT)$ with normalized coefficients. Additionally, for any function $f\in\cF_{\rho}(\beta,k,\bT)$ with $\rho>0$, Lemma~\ref{gcn-approx} implies that there exists $g\in\cG(L_1,\bT)$ satisfying $\|f-g\|_{\infty}\leq\rho$. In the case where $\beta > 1$, any $f \in \mathcal{F}_0(\beta,k,\bT)$ can be expressed as $f = \beta h$, where $h \in \mathcal{F}_0(1,k,\bT)$. Thus, the original approximation problem reduces to approximating the normalized function class $\mathcal{F}_0(1,k,\bT)$, while the scaling factor $\beta$ will be handled separately by the DNN component.

A key feature of the GNN class defined in~\eqref{gnn-models} is that it produces predictions through successive nonlinear transformations applied to the propagated features: the DNN ``readout'' is itself a composition of simple maps (affine transformations and pointwise nonlinearities). Motivated by this architectural structure, we model $\varphi^*$ as a finite composition of smooth, low-dimensional building blocks, allowing complex dependencies in $d$ variables to be assembled hierarchically from simpler interactions.
Beyond matching the network structure, this assumption is also statistically meaningful: if each intermediate component depends only on a small number of variables, then approximation and estimation rates are governed by the corresponding intrinsic dimensions rather than the ambient dimension $d$. This helps mitigate the curse of dimensionality, as shown e.g.\ in \cite{Schmidt-Hieber,juditsky2009nonparametric,baraud2014estimating}.

Concretely, we assume that $\varphi^*$ admits the representation
\begin{equation*}
\varphi^* = g_q^* \circ g_{q-1}^* \circ \cdots \circ g_0^*,
\end{equation*}
where each $g_i^*=(g_{i,1}^*,\ldots,g_{i,d_{i+1}}^*)^\top$ maps $[a_i,b_i]^{d_i}$ to $[a_{i+1},b_{i+1}]^{d_{i+1}}$.
Furthermore, for every $j\in\{1,\ldots,d_{i+1}\}$, we assume that $g_{i,j}^*$ is $\alpha_i$-H\"older smooth and depends on at most $t_i$ variables.
We now formalize this notion by recalling H\"older balls and then defining the resulting compositional function class.

For any $\alpha>0$, we say a function $f$ is of $\alpha$-H\"older smoothness if all its partial derivatives
up to order $\lfloor \alpha \rfloor$ exist and are bounded, and the partial derivatives of order $\lfloor \alpha \rfloor$
are $(\alpha-\lfloor \alpha \rfloor)$-H\"older continuous.
Given $t\in\mathbb{N}$ and $\alpha>0$, define the $\alpha$-H\"older ball with radius $K\geq0,$ denoted by $\cH_t^\alpha(\cD,K),$ as the collection of functions $f:\cD\subset\mathbb{R}^t\to\mathbb{R}$ such that
$$\sum_{\substack{{\bs{\beta}}=(\beta_{1},\ldots,\beta_{t})\in\N_0^{t}\\ \sum_{j=1}^{t}\beta_{j}<\alpha}}\|\partial^{\bs{\beta}}f\|_{\infty}+\sum_{\substack{{\bs{\beta}}\in\N_0^{t}\\ \sum_{j=1}^{t}\beta_{j}=\lfloor\alpha\rfloor}}\sup_{\substack{{\bs{x}},{\bs{y}}\in\cD\\{\bs{x}}\not={\bs{y}}}}\frac{\left|\partial^{\bs{\beta}}f({\bs{x}})-\partial^{\bs{\beta}}f({\bs{y}})\right|}{|{\bs{x}}-{\bs{y}}|_{\infty}^{\alpha-\lfloor\alpha\rfloor}}\leq K,$$ where for any ${\bs{\beta}}=(\beta_{1},\ldots,\beta_{t})\in\N_0^{t}$, $\partial^{\bs{\beta}}=\partial^{\beta_{1}}\cdots\partial^{\beta_{t}}$. 

Within the framework, we assume that each constituent function $g^*_{i,j} \in \mathcal{H}^{\alpha_i}_{t_i}([a_i, b_i]^{t_i}, K)$. Consequently, the underlying compositional function class is given by
\begin{align}
\cG(q,{\bs{d}},{\bs{t}},{\bm{\alpha}},K)=&\left\{g_{q}\circ\cdots\circ g_{0}:\;g_{i}=(g_{i,j})_{j}:\cro{a_{i},b_{i}}^{d_{i}}\rightarrow\cro{a_{i+1},b_{i+1}}^{d_{i+1}},\right.\nonumber\\
&\quad\quad\quad\quad\quad\quad\quad\quad\left.g_{i,j}\in\cH^{\alpha_{i}}_{t_{i}}(\cro{a_{i},b_{i}}^{t_{i}},K),\;\mbox{for some } |a_{i}|, |b_{i}|\leq K\right\},\label{composite-def}
\end{align}
where $q\in\mathbb{N}_0$, ${\bs{d}}=(d_{0},\ldots,d_{q+1})\in\mathbb{N}^{q+2}$ with $d_{0}=1$ and $d_{q+1}=1$, ${\bs{t}}=(t_{0},\ldots,t_{q})\in\mathbb{N}^{q+1}$, ${\bm{\alpha}}=(\alpha_{0},\ldots,\alpha_{q})\in\R_{+}^{q+1}$, and $K\geq0$. The $i$-th entry $t_i$ of the vector $\bs{t}$ indicates the effective input dimension for every $g_{i,j}$, with $j=1,\ldots,d_{i+1}$. 

For any function $f=g_{q}\circ\cdots\circ g_{0}\in\cG(q,{\bs{d}},{\bs{t}},{\bm{\alpha}},K)$, where each component $g_i$ possesses certain smoothness properties, the overall composition $f$ exhibits a specific level of smoothness determined by its constituents. The classical $d$-variate H\"older class corresponds to the special case where $q=0$ and $t_0=d$. In contrast, for the case $q = 1$ with $\alpha_0,\alpha_1 \leq 1$ and $d_0 = d_1 = t_0 = t_1 = 1$, the composite function $f = g_1 \circ g_0$ achieves a smoothness of order $\alpha_0\alpha_1$, as established in \cite{2017regularity,juditsky2009nonparametric,baraud2014estimating}. For general compositions, the effective smoothness parameters are defined for $i=0,\ldots,q-1$ by
\begin{equation}\label{effect-smooth}
\alpha_i^*=\alpha_i \prod_{\ell=i+1}^q (\alpha_\ell \wedge 1),    
\end{equation}
with $\alpha_q^*=\alpha_q$. These parameters influence the convergence rate of the network estimator.

Previous studies have established that deep ReLU networks are capable of effectively approximating compositional smooth function classes \cite{Schmidt-Hieber,chen2024}. Combined with Lemma~\ref{gcn-approx}, this leads to the following approximation error bound, proving that the target regression function class $\mathcal{F}_{\rho}(\beta,k,\bS_{\mathbf{A}})$ can be well approximated by $\mathcal{F}(\bS_{\mathbf{A}},L_1,L_2,\bs{p},s,F)$ with appropriate architecture.
\begin{lemma}\label{overall-approx}
Let $\varphi^{*} \in \cG(q,{\bs{d}},{\bs{t}},{\bm{\alpha}},K)$ with $K\geq1$ be defined in \eqref{composite-def}, and let $\psi_{\mathbf{A}}^*\in\mathcal{F}_{\rho}(\beta,k,\bS_{\mathbf{A}})$ with $\rho<1$ be defined in \eqref{delta-sum}. Set $Q_0=1$, $Q_i=(2K)^{\alpha_i}$ for $i \in [q-1]$, and $Q_q=K(2K)^{\alpha_q}$.
For any $N_i \in \mathbb{N}$ such that $N_i \geq (\alpha_i+1)^{t_i} \vee (Q_i+1)e^{t_i}$, there exists $f\in \cF(\bS_{\mathbf{A}},L_1,L_2,{\bs p},s,F)$ satisfying $$L_1\geq k,\quad L_2\leq C_{3} \log_2 n,\quad N=\max_{i=0,\ldots,q}N_i,\quad s \leq C_{4} N \log_2 n$$ $${\bs{p}}=\left(d,3\left\lceil\frac{\beta}{M}\right\rceil d,C_{5}N,\ldots,C_{5}N,1\right),\quad\mbox{and}\quad F\geq K,$$ such that for every $f^*_j = \varphi^{*} \circ \psi_{\mathbf{A},j}^{*}$, $j = 1,\dots,n$,
$$\|f_j-f^*_j\|_{\infty}\leq C_{6}\cro{\sum_{i=0}^{q}\left(N_i^{-\frac{\alpha_i}{t_i}}+N_in^{-\frac{\alpha_i+t_i}{2\alpha_i^*+t_i}}\right)^{\prod_{\ell=i+1}^{q}(\alpha_{\ell}\wedge1)}\hspace{-10pt}+\rho^{\prod_{i=0}^{q}(\alpha_{i}\wedge1)}},$$
where $C_{3},C_{4},C_{5},C_{6}$ are numerical constants that do not depend on $n$.
\end{lemma}
The proof is postponed to Section~\ref{appro-proofs}. Lemma~\ref{overall-approx} decomposes the approximation error into two terms associated with approximating $\cF_\rho(\beta,k,\bS_{\mathbf{A}})$ and $\cG(q,{\bs{d}},{\bs{t}},{\bm{\alpha}},K)$. The effective smoothness $\alpha_i^*$ quantify the approximation rate. In particular, when $N_i \asymp n^{t_i/(2\alpha_i^*+t_i)}$ holds for all $i$, the first term in the bracket becomes
$$\sum_{i=0}^{q}\left(N_i^{-\frac{\alpha_i}{t_i}}+N_in^{-\frac{\alpha_i+t_i}{2\alpha_i^*+t_i}}\right)^{\prod_{\ell=i+1}^{q}(\alpha_{\ell}\wedge1)}\lesssim q\max_{i=0,\ldots,q}n^{-\frac{\alpha_i^*}{2\alpha_i^*+t_i}}.$$ This scenario arises when we later incorporate the model complexity from Proposition~\ref{cover-whole} to determine an appropriate network architecture for estimation.

The preceding analysis characterizes the approximation error associated with using the true propagation operator $\mathbf{S}_{\mathbf{A}}$. For practical situations where the implemented operator $\bT$ deviates from $\bS_{\mathbf{A}}$, the next result shows that this deviation introduces an additional error term beyond $\rho$.
\begin{lemma}\label{t-s-app-error}
Let $\bT$ and $\bS_{\mathbf{A}}$ be $n \times n$ matrices with respective row sparsity at most $d_{\bT}$ and $d_{\bS_{\mathbf{A}}}$. If $\|\bT-\bS_{\mathbf{A}}\|_{\operatorname{F}} \leq \tau$, then for any positive integer $L_1$, any row index $j \in\{1,\ldots,n\}$, and any ${\bs x} \in [0,1]^{n \times d}$, we have
\[
\left|\left(\sum_{i=1}^{L_1}\theta_i(\bT^i-\bS_{\mathbf{A}}^i){\bs x}\right)_{j,\cdot}\right|_{\infty}\leq\tau \sqrt{d_{\bT}+d_{\bS_{\mathbf{A}}}}\cro{\sum_{i=1}^{L_1}i|\theta_i| \big(\|\bT\|_{1,\infty}\vee\|\bS_{\mathbf{A}}\|_{1,\infty}\big)^{i-1}}.
\]
\end{lemma}
The proof is given in Section \ref{appro-proofs}. Lemma~\ref{t-s-app-error} shows that the additional
approximation error induced by using an implemented propagation operator $\bT$ instead of the
target operator $\bS_{\mathbf{A}}$ is controlled by three factors: (i) their Euclidean distance  
$\tau=\|\bT-\bS_{\mathbf{A}}\|_{\operatorname{F}}$, (ii) the local sparsity levels $d_{\bT},d_{\bS}$ (number of nonzeros per row), and
(iii) an amplification term $\sum_{i=1}^{L_1}i |\theta_i|(\|\bT\|_{1,\infty}\vee\|\bS_{\mathbf{A}}\|_{1,\infty})^{i-1}$ arising from repeated
propagation. In particular, if $L_1$ and $\{|\theta_i|\}_{i=1}^{L_1}$ are treated as constants, the mismatch
contribution scales as $\lesssim\tau\sqrt{d_{\bT}+d_{\bS_{\mathbf{A}}}}$.

For sparse ``real-life'' networks, it is natural to assume that the number of
connections per node grows slowly with $n$; a common regime, as shown in \cite{bonato2014dimensionality,xue2004number}, is
\[
d_{\bT},d_{\bS_{\mathbf{A}}}\lesssim\log n.
\]
Under this logarithmic-degree scaling, Lemma~\ref{t-s-app-error} yields a mismatch term of order
$\tau\sqrt{\log n}$ provided $\|\gT\|_{1,\infty}\vee\|\bS_{\mathbf{A}}\|_{1,\infty}$ remains $\cO(1)$, a condition satisfied by commonly used normalized propagation operators such as row stochastic matrices. If instead $\bT$ is
an unnormalized adjacency-type operator with $\cO(1)$ weights, then $\|\gT\|_{1,\infty}$ typically
scales like the maximum degree and hence also $\lesssim \log n$ in this regime, yielding an
at-most polylogarithmic dependence on $n$ through
$\sum_{i=1}^{L_1} |\theta_i|\, i\,(\log n)^{i-1}$
(and therefore $\lesssim \tau(\log n)^{L_1-1/2}$ when $L_1$ is fixed).

\subsection{Convergence rate of the least-squares estimator}
Throughout this section, we assume that the propagation function $\psi_{\mathbf{A}}^*\in\cF_0(\beta,k,\bS_{\mathbf{A}})$ and the outer function $\varphi^{*}\in\cG(q,{\bs{d}},{\bs{t}},{\bm{\alpha}},K)$. The aim is to examine the rate of convergence of the least-squares estimator $\widehat f$ defined in \eqref{lse-def} over the class $\cF(\bS_{\mathbf{A}},L_1,L_2,{\bs p},s,F)$. 

By definition \eqref{def-delta-em}, the optimization error $\Delta^{\cF(\bS_{\mathbf{A}},L_1,L_2,{\bs p},s,F)}_n
(\widehat f, f^*)$ vanishes if $\widehat f$ is the empirical risk minimizer. Therefore, Corollary~\ref{kappa-b} shows that the prediction error of $\widehat f$ admits a bias-variance decomposition consisting of a bias term plus a variance term that depends on the network class. Proposition~\ref{cover-whole} and Lemma~\ref{overall-approx} demonstrate that using network models with more parameters increases the variance while typically reducing the bias. In the next result, we show that when the network architecture is well-chosen, the least-squares estimator $\widehat f$ converges to the true regression function.

\begin{theorem}\label{converge-rate}
Suppose Assumption~\ref{ass-m} holds for some $m\geq1$ and that the unknown regression function $f^*$ has $i$-th component of the form $f^*_i=\varphi^{*}\circ \psi_{\mathbf{A},i}^*$, where $\varphi^{*}\in\cG(q,{\bs{d}},{\bs{t}},{\bm{\alpha}},K)$ with $K\geq1$ and $\psi_{\mathbf{A}}^*\in\cF_0(\beta,k,\bS_{\mathbf{A}})$. Set $N_i=\lceil n^{t_i/(t_i+2\alpha_i^*)}\rceil$, $N=\max_{i=0,\ldots,q}N_i$, and let $\widehat f$ be the least-squares estimator over the network class $\mathcal{F}(\bS_{\mathbf{A}}, L_1, L_2, \bs{p}_n, s_n,F)$, where the network architecture satisfies
\begin{enumerate}[label=(\roman*), itemsep=0pt, topsep=2pt, parsep=0pt]
\item $L_1 \geq k$;
\item $1\leq L_2 \leq C_{7} \log_2 n$;
\item ${\bs{p}}_n=\left(d,3\lceil\frac{\beta}{M}\rceil d,C_{8}N,\ldots,C_{8}N,1\right)$;
\item $2\leq s_n \leq C_{9} N\log_2 n$;
\item $F \geq K$.
\end{enumerate}
Then, for all sufficiently large $n$, 
\begin{align*}
\mathcal{R}(\widehat{f}, f^*)\leq C_{10}\frac{m^2\log^3 n}{\pi}\max_{i=0,\ldots,q}n^{-\frac{2\alpha_i^*}{2\alpha_i^*+t_i}}.
\end{align*} 
Here, $C_7,C_8,C_9,C_{10}$ are numerical constants independent of $n$.
\end{theorem}
The proof follows from a consequence of Lemma~\ref{overall-approx} and Corollary~\ref{kappa-b}, and is deferred to Section~\ref{main-proof}. Theorem~\ref{converge-rate} provides an explicit non-asymptotic convergence rate for the least-squares
estimator over the GNN class \(\mathcal{F}(\bS_{\mathbf{A}},L_1,L_2,p_n,s_n,F)\) under the compositional
model \(f_i^*=\varphi^*\circ \psi_{{\mathbf{A}},i}^*\).
The bound exhibits the following salient features.

The convergence rate is governed by the intrinsic regularity of $\varphi^*$, specifically by its effective smoothness parameters $\alpha_i^*$ defined in \eqref{effect-smooth} and the associated intrinsic dimensions $t_i$. Ultimately, the overall rate is determined by the bottleneck \(\max_{i=0,\dots,q} n^{-2\alpha_i^*/(2\alpha_i^*+t_i)}\). In particular, in the classical Hölder case (\(q=0\), \(t_0=d\)), Theorem~2 recovers the minimax-optimal rate
\(n^{-2\alpha/(2\alpha+d)}\) in terms of the sample size $n$, up to logarithmic factors. For target functions \(\varphi^*\) with a genuinely
compositional structure, however, the convergence rate depends only on the smaller intrinsic dimensions \(t_i\), thereby mitigating the curse of dimensionality. 

While matching the minimax rate with respect to the sample size $n$, as noted below Theorem~\ref{main}, the factor \(m^2/\pi\) characterizes a key departure from standard regression, explicitly quantifying the interaction between graph-induced dependence and semi-supervised learning. For graphs of bounded degree with fixed propagation depth, \(m\) is essentially constant (up to polylog factors). When \(m\) grows rapidly with \(n\) (for instance, due to large depth \(L_1\) or the presence of high-degree hubs), the combined term \(m^2/\pi\) may dominate. This reveals an explicit bias-variance trade-off in message passing: while increasing \(L_1\) (or employing less localized propagation) reduces approximation bias for \(\psi_{{\mathbf{A}}}^*\), it simultaneously increases the dependence penalty through \(m\). Thus, $L_1$ should be kept near the target filter order $k$, and normalized propagation operators (with bounded $\|\bS_{\mathbf{A}}\|_{1,\infty}$) are recommended to control constants.

\section{Numerical experiments}\label{experiment-sec}
In this section, we empirically validate our theoretical findings through synthetic and real world experiments. Specifically, we seek to (i) confirm the convergence rate of the prediction error as a function of the sample size $n$ (as provided by Theorem \ref{converge-rate}), and (ii) evaluate how the graph topology, as captured by the average and maximum degree $\Delta$, affects the decay rate of the MSE.  To this end, we compare four neural network architectures:
\begin{itemize}
\item \textbf{MLP (no propagation).} Set $\bZ=\bX$ and predict $\widehat{\bY}=\widehat f(\bZ)$ with a ReLU MLP
$\widehat f\in \mathrm{DNN}_{L_{\mathrm D}}$, where $L_{\mathrm D}$ denotes the depth of the neural network.

\item \textbf{GCN (no skip connections).} Propagate features with $L_{\mathrm G}\in\{1,2,3\}$ linear GCN layers
$\widehat g\in\mathrm{GCN}_{L_{\mathrm G}}$ to obtain $\bZ=\widehat g(\bX)$ , then predict with an MLP head
$\widehat{\bY}=\widehat f(\bZ)$, $\hat f\in\mathrm{DNN}_{L_{\mathrm D}}$.

\item \textbf{GCN (skip connections).} Compute layerwise representations
$H^{(\ell)}_{\bS_{\bA}}(\bX)$ and form a convex combination
$\bZ=\sum_{\ell=1}^{L} w_\ell H^{(\ell)}_{\bS_{\bA}}(\bX)$ with
$w_\ell=\exp(\alpha_\ell)/\sum_{k=1}^{L}\exp(\alpha_k)$, then predict
$\widehat{\bY}=\widehat f(\bZ)$ with $\widehat f\in\mathrm{DNN}_{L_{\mathrm D}}$.

\item \textbf{MaGNet-inspired multi-scale model}. We adapt MaGNet to node regression by
aggregating multi-hop linear convolutions $H^{(\ell)}_{\bS}(\bX)=\bS^\ell \bX \bW$ with
$\bS=\bD^{-1/2}\bA\bD^{-1/2}$ (no self-loops), forming $\bZ=\sum_{\ell=1}^{L} w_\ell H^{(\ell)}_{\bS}(\bX)$,
and applying an MLP head node-wise. The fusion weights $w_\ell$ are either learned end-to-end or set via
the critic mode of \cite{zhou2025model}.
\end{itemize}

\subsection{Evaluating convergence rate}
\begin{figure}[t]
\centering
\includegraphics[width=\linewidth,height=9cm]{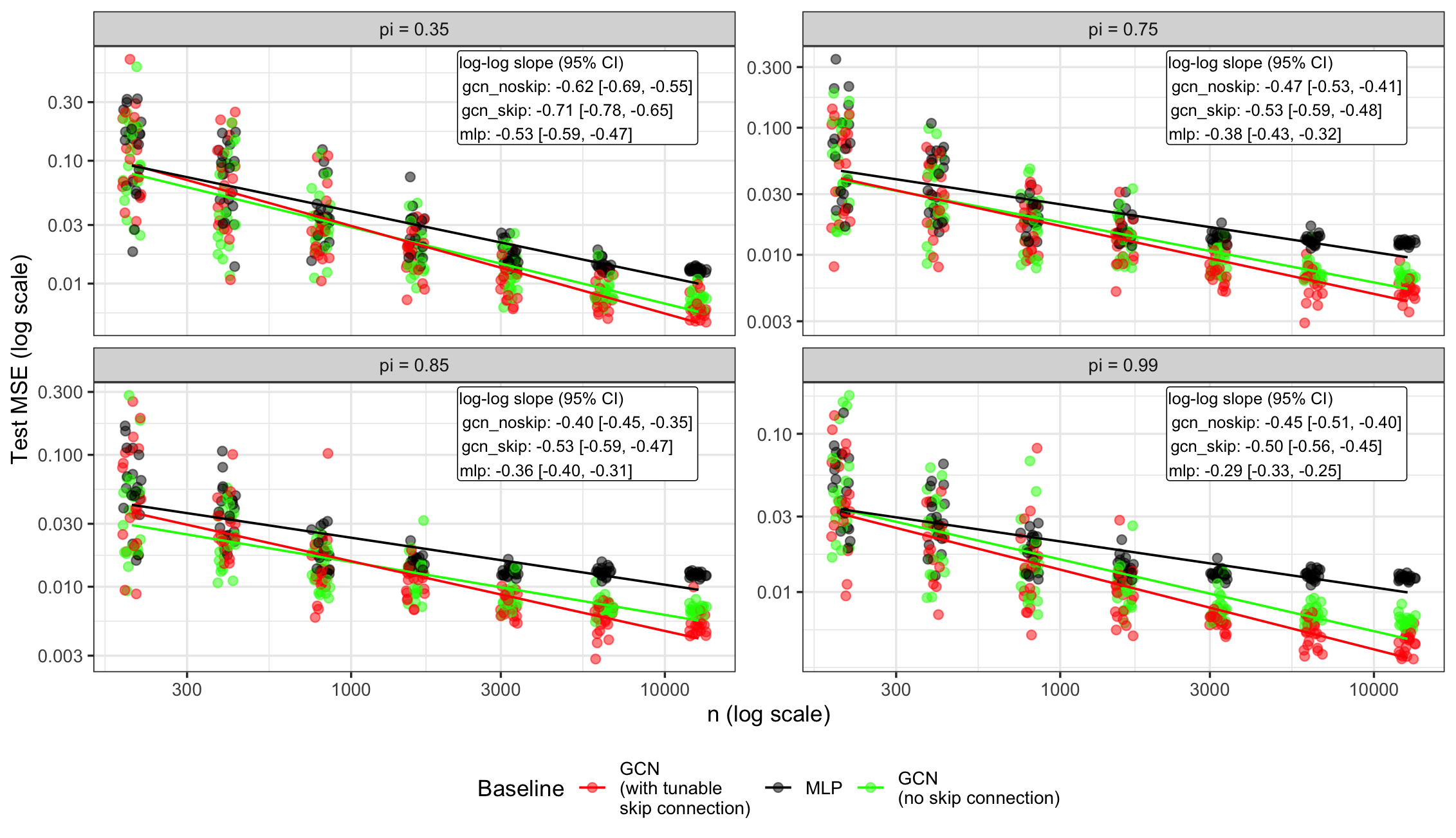}
\caption{MSE (over 20 trials) as a function of training samples $n$, with the unmasked proportion held constant at $\pi \in \{0.35, 0.75, 0.85, 0.95\}$. Estimators are distinguished by color: GCN with skip connections, GCN without skip connections, and the MLP baseline.}
\label{fig:slope1}
\end{figure}
To evaluate the results established in Theorem \ref{converge-rate}, we consider a cycle (ring) graph with $n$ nodes and a bounded degree of $\Delta = 3$ (including self-loops). We generate node features as $X_i \stackrel{\text{i.i.d.}}{\sim} \operatorname{Unif}[0,1]$ and construct propagated features as $$\mathbf{Z} = \sum_{j=1}^k \theta_j\bS_{\mathbf{A}}^j \bX,$$ where $\mathbf{S}_{\mathbf{A}}$ is taken as a neighborhood average and the coefficients $\theta_j$ are bounded. The responses are generated as $$Y_i = \varphi^{*}(\bZ_{i,\cdot})+\varepsilon_i\quad\mbox{with}\quad\varepsilon_i\stackrel{\text{i.i.d.}}{\sim} \mathcal{N}(0,1).$$ We set $$\varphi^*({\bs z}) = \operatorname{BM}\left(\operatorname{Sigmoid}\left(\frac{{\bs z}}{\text{scale}}\right)\right),$$ where BM denotes a sample path of a Brownian motion, obtained by discretizing the interval $[0,1]$ into $2^{12}$ equal-length subintervals. This ensures that $\varphi^*$ has H\"older regularity $\alpha \approx 1/2$ \citep{KleyntssensNicolay2022}. Node responses $Y_i$ are observed on a random subset (inclusion probability $\pi$). Models are trained and evaluated on an independent copy of the features $\mathbf{X}'$.

Results are compared across methods and for different values of $\pi$ in Figure~\ref{fig:slope1}. Recall that Theorem \ref{converge-rate} points to a learning rate of order $n^{-2\alpha/(2\alpha+t)}$. For the considered simulation setup, $t=1$ and $\alpha$ is close to $1/2$. Theorem \ref{converge-rate} thus postulates a convergence rate of approximately $n^{-1/2}$ in this setting.

\begin{figure}[t]
\centering
\includegraphics[width=0.85\textwidth,height=6cm]{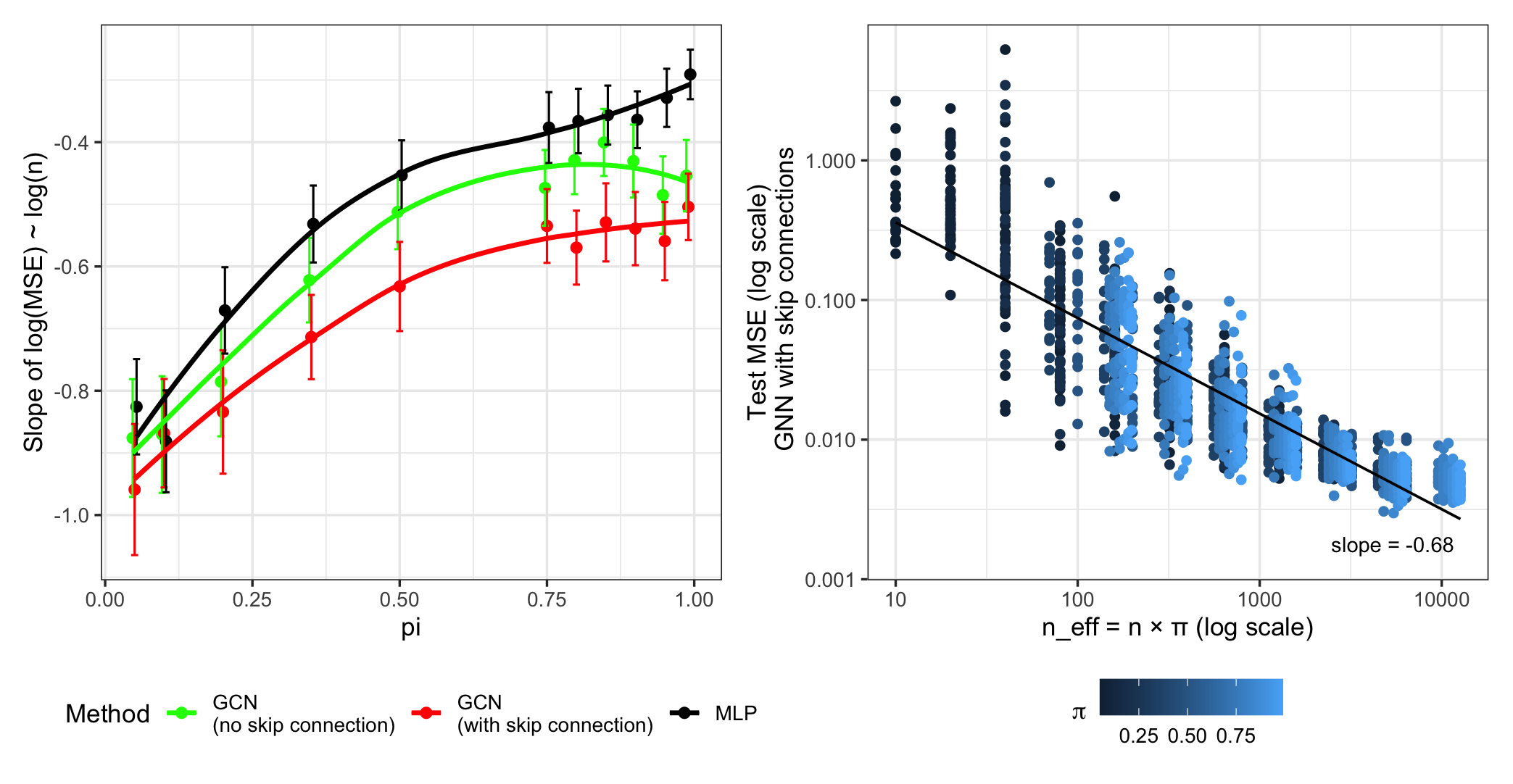}
\caption{\textbf{Left:} Fitted slopes of $\log(\mathrm{MSE})$ vs. $\log(n)$ as a function of connection probability $\pi$. \textbf{Right}: $\log(\operatorname{MSE})$ as a function of the effective sample size $n_{\text{eff}} = n \times \pi$ for the GCN with skip connections, as proposed in Equation~\eqref{gcn-l-t}.}
\label{fig:all_exp1}
\end{figure}

As shown in Figure~\ref{fig:slope1}, for large training sets (high $\pi$), the fitted slope is approximately $-1/2$, in good agreement with the theoretical prediction. We also note, however, that for very low $\pi$, the slope deviates from its expected value; This is particularly salient in Figure~\ref{fig:all_exp1}, where we plot the fitted slope as a function of the proportion of training data $\pi$, and observe significantly lower slopes for extremely small values of $\pi$ (for instance, around $-0.95$ for $\pi=0.01$). This effect could be explained by an increase in the relative contribution of the optimization error in low-data regimes (consistent with Corollary~\ref{kappa-b}, which posits that optimization error imposes a lower bound on prediction accuracy). The right subplot of Figure~\ref{fig:all_exp1}, which shows performance (as measured by the test MSE) against the effective number of training samples, highlights indeed a low-sample regime ($n_{\text{eff}} < 100$), characterized by higher MSE and a stagnation phase before improvement ($n_{\text{eff}} \geq 100$).

\begin{figure}[t]
\centering
\includegraphics[width=\linewidth]{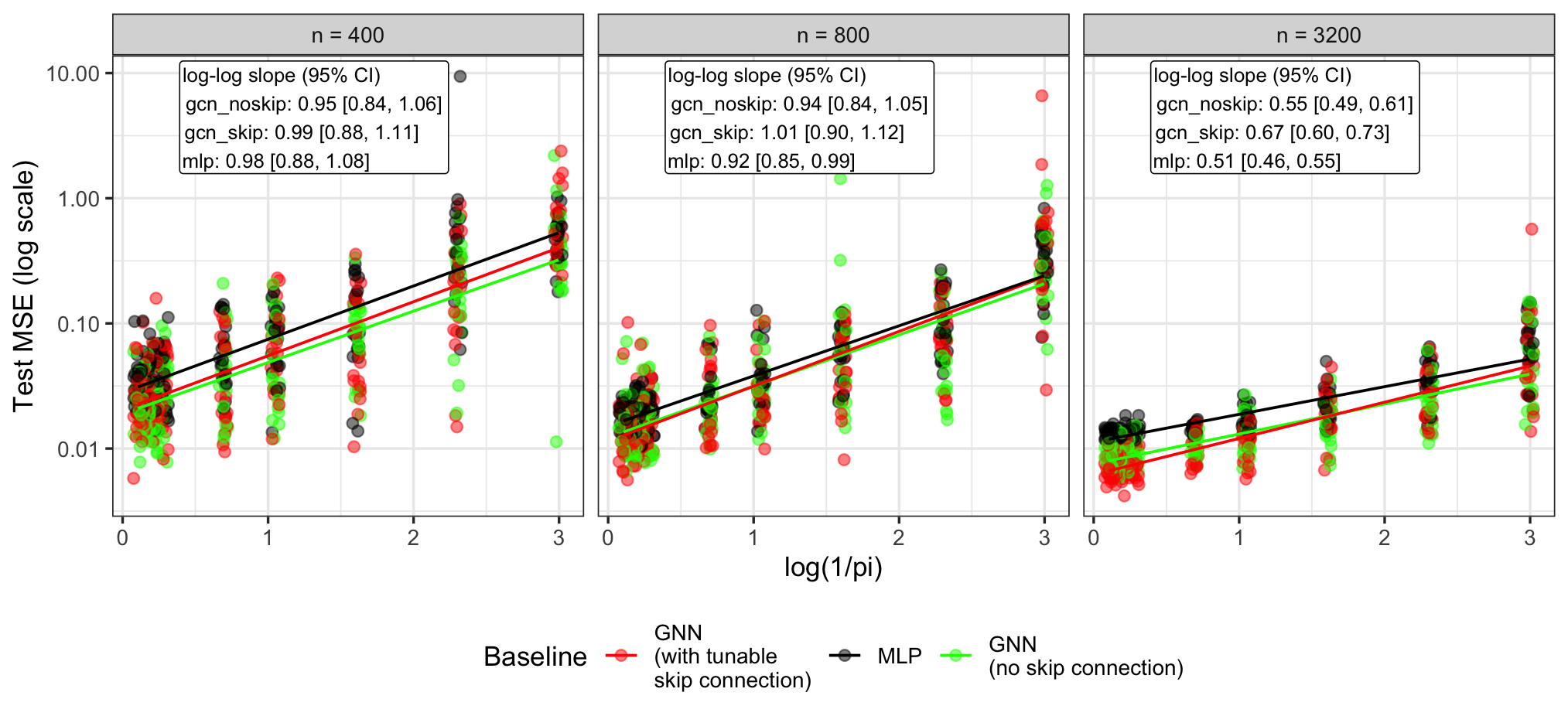}
\caption{MSE (over 20 trials) as a function of $\log(1/\pi)$ across different graph sizes. Estimators are distinguished by color: GCN with skip connections, GCN without skip connections, and the MLP baseline.}
\label{fig:slope_pi}
\end{figure}

Figure~\ref{fig:slope_pi} is a log-log plot of the MSE as a function of the inverse of the labeled nodes proportion $1/\pi$. Interestingly, for smaller datasets, the results seem to be in agreement with theory, with a linear increase for small values of $\pi$ (large values of $\log(1/\pi)$). As $n$ increases, the fitted slope decreases (e.g., to $0.67$ at $n=3,200$), which is consistent with a regime
in which $\pi$-independent components of the error (e.g. optimization effects) become non-negligible
relative to the $1/\pi$ stochastic term. This suggests that the worst-case $1/\pi$ dependence in the bound can be
conservative in large-$n$ settings.

\subsection{Assessing the impact of graph topology}
To validate the dependency of Theorem~\ref{converge-rate} on the receptive field size $m$, we generate synthetic graphs ($n=3000$) with a fixed average degree $\bar{\delta}$ across four distinct topologies: 
\begin{enumerate}[label=(\roman*)]
\item Erd\H{o}s--Rényi graphs with $\pi= \bar{\delta}/n$;
\item stochastic block models on 2 blocks of equal size, with intra-class probability $\pi_{\text{within}} = 0.55\bar{\delta}/n$ and inter-class probability $\pi_{\text{between}} = 0.055\bar{\delta}/n$;
\item random geometric graphs with radius $\tau = \sqrt{\bar{\delta}/(\pi n)}$);
\item Barabási–Albert graphs with parameter $m = \lfloor \bar{\delta}/2 \rfloor$.
\end{enumerate}

Node features are sampled as $X_i \overset{\text{i.i.d.}}{\sim} \mathcal{N}(0, \sigma^2 \mathbf{I}_d)$, with $\sigma^2=1$ by default. Propagated features are generated as $\mathbf{Z} = \sum_{j=1}^k \theta_j\bS_{\mathbf{A}}^j \bX$ for various choices of the graph convolution operator $\mathbf{S}_{\mathbf{A}}$. Each coefficient $\theta_j$ is sampled uniformly at random, followed by row-normalization of $\bs\theta = (\theta_1, \ldots, \theta_k)$. To ensure comparability across topologies, we standardize $\mathbf{Z}$ and control for Laplacian energy (a measure of signal smoothness). Let $\varphi^*$ be a fixed-architecture DNN of depth $L_2= 2$ (default) with random parameters, ReLU activations, and residual connections. We define the raw targets as $y_i = \varphi^*(\bZ_{i,\cdot})$, which are then standardized to $\tilde y_i=y_i-\overline y/s_{\bs y}$. The final responses are generated as $Y_i = \tilde y_i + \varepsilon_i, \ \varepsilon_i \sim \mathcal{N}(0, \sigma^2)$.

\begin{figure}[t]
\centering
\includegraphics[width=\linewidth]{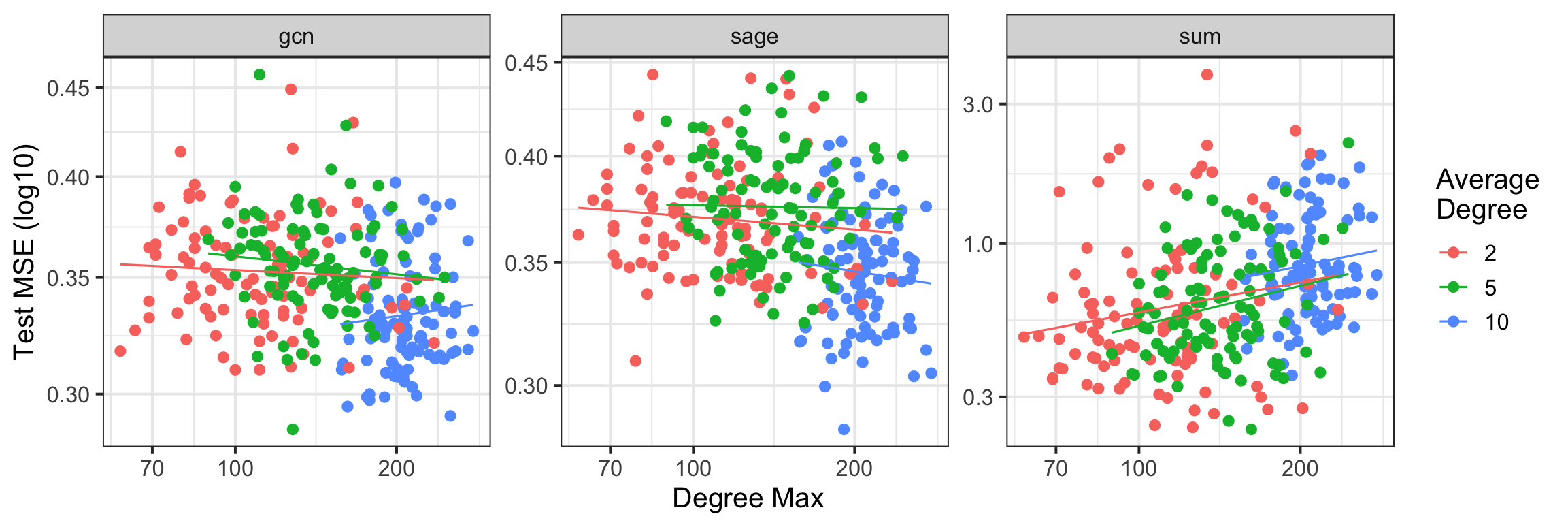}
\caption{Performance of the GCN (with skip connection) as a function of the maximum degree of the graph ($x$-axis), for different convolution types (columns) and values of $\bar \delta$ (colors) on a Barabási–Albert graph. }
\label{fig:effect_degree}
\end{figure}
\begin{figure}[t]
\centering
\includegraphics[width=\textwidth,height=5.2cm]{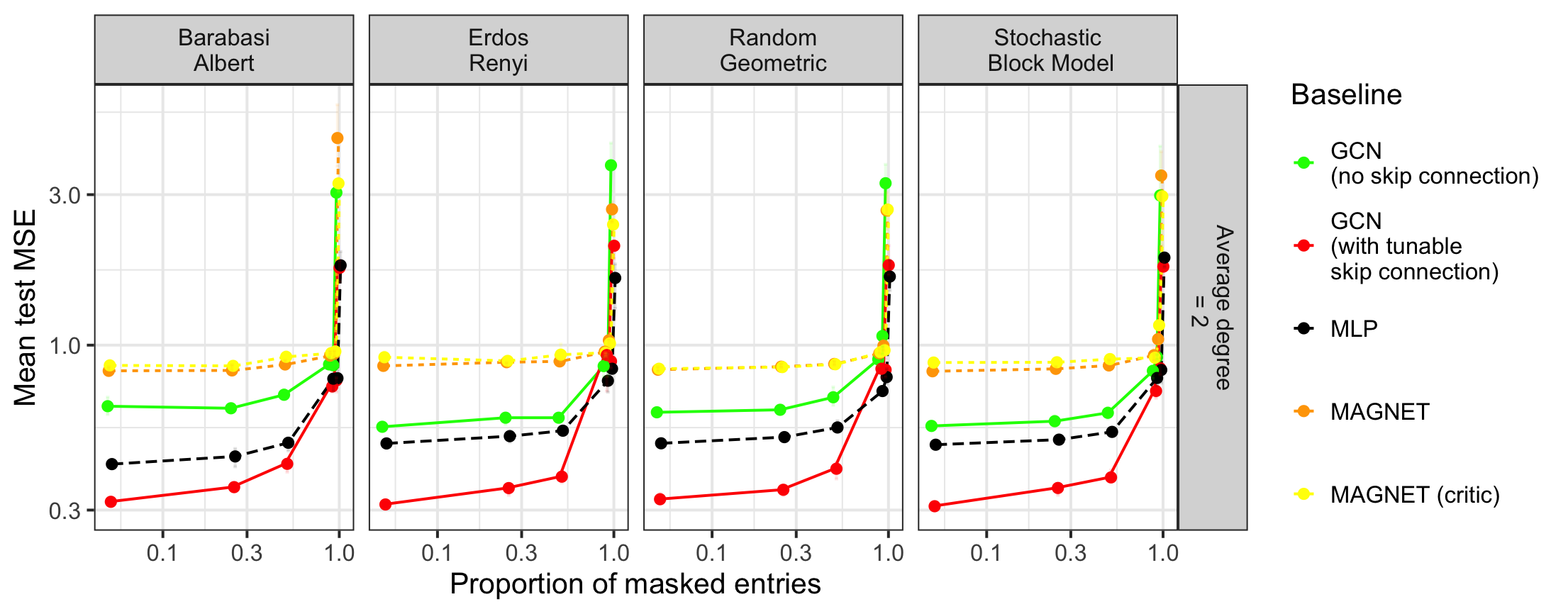}
\caption{Performance of the different baselines averaged over 5 experiments as a function of $\pi$ ($x$-axis) and the graph type (columns). Here, all GCN versions use $L=2$ layers, and the average degree was fixed to $\bar{\delta}=2$ across topologies.}
\label{fig:fig1}
\end{figure}

Figure~\ref{fig:effect_degree} illustrates the impact of the maximum degree in Barab\'asi--Albert graphs across different propagation operators $\mathbf{S}_{\mathbf{A}}$ (with matched data generation): the GCN convolution ($\mathbf{S}_{\mathbf{A}} = \widetilde{\mathbf{D}}^{-1/2} \widetilde{\mathbf{A}} \widetilde{\mathbf{D}}^{-1/2}$), Sage ($\mathbf{S}_{\mathbf{A}} = \widetilde{\mathbf{D}}^{-1} \widetilde{\mathbf{A}}$), and a GINE-style sum operator \citep{xu2019powerful} ($\mathbf{S}_{\mathbf{A}} = \widetilde{\mathbf{A}}$).
The sum convolution is particularly sensitive to the maximal degree --- as predicted by Theorem~\ref{converge-rate}, higher maximal degrees yield considerably worse MSE. This effect is however substantially mitigated for degree-averaging filters. Our discussion of Theorem~\ref{converge-rate} highlights the importance of choosing degree-averaging operators --- particularly to control constants in the rate. This experiment demonstrates that the bound in Theorem~\ref{converge-rate} is accurate for the sum operator but remains conservative for degree-averaging operators.

Figure~\ref{fig:fig1} exhibits the performance of various methods versus label fraction $\pi$ across topologies, using convolution with a fixed average degree $\bar{\delta}=2$. Notably, the performance of the GCN with skipped connections is stable across graph topologies. This confirms the fact that, with the effect of the maximal degree mitigated, as predicted by Theorem~\ref{converge-rate}, the primary driver of the error bounds lies in the size of the receptive field, rather than the connectedness of the graph (or other spectral properties), as is the case in other graph-regularization-based approaches \citep{pmlr-v49-huetter16, tran2025generalized}.

\subsection{Performance on real-world data}
We evaluate the performance of GNNs on two different datasets:
\begin{enumerate}[label=(\roman*)]
\item The California Housing Dataset \cite{pace1997sparse}\footnote{The California Housing Dataset can be found as part of the \href{https://scikit-learn.org/stable/modules/generated/sklearn.datasets.fetch_california_housing.html}{\texttt{sklearn} library}.}, a dataset of 20,640 observations of property attributes from the 1990 U.S. Census, grouped at the block group level (small geographical unit). The goal is to predict the median house value based on 9 attributes (median income, house age, average number of rooms, number of bedrooms per household, as well as size of the household, group population, and latitude and longitude). Data points are embedded within a $k$-NN graph based on their spatial coordinates.
\item The Wikipedia Chameleon dataset\footnote{The Wikipedia Chameleon Dataset can be found at: \url{https://snap.stanford.edu/data/wikipedia-article-networks.html}}, a graph of 2,277 Wikipedia pages about chameleons connected by 31,421 mutual hyperlinks; each node has sparse text-derived features indicating which “informative nouns” appear in the article, and the regression target is the page’s average monthly traffic (Oct 2017–Nov 2018). In contrast to the California Housing dataset, this dataset is much denser and higher dimensional.
\end{enumerate}

We further compare the graph neural networks and MLP baselines to Tikhonov and Laplacian smoothing. The optimal parameters (i.e. the number of GNN convolutions and the depth of the neural network head) were selected based on the performance of each method on held-out nodes in a calibration set. 

\begin{figure}[t]
\centering
\includegraphics[width=\linewidth,height=6cm]{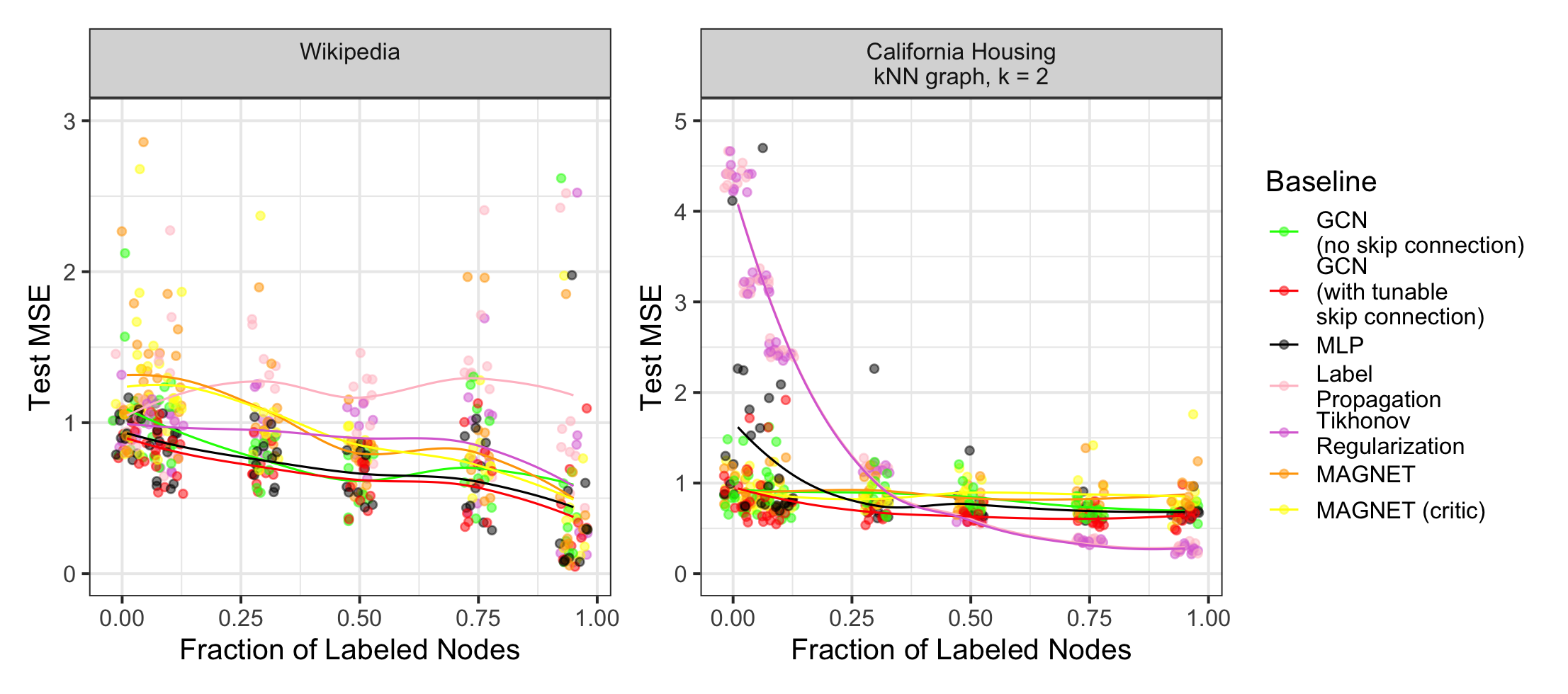}
\caption{Performance of the different baselines as a function of the proportion of nodes used for training for the California Housing and Wikipedia datasets. For Wikipedia, displayed MSE values are capped at 3 to keep the scale readable; some MAGNET and MAGNET-critic replicates exceed this value.}
\label{fig:cal}
\end{figure}

Figure~\ref{fig:cal} reports test performance across architectures. Importantly, this figure measures transductive generalization—error on previously unseen nodes in the same graph—whereas our theory focuses on the inductive risk. Despite this mismatch, the same qualitative picture emerges: the GNN model (with GCN diffusion) performs best overall, with one notable exception on California Housing, where Tikhonov regularization improves substantially as $\pi$ increases and becomes highly competitive. On the less homophilic Wikipedia dataset, neural architectures dominate the classical baselines. Overall, these examples reinforce the practical relevance of our function class: relative to a plain MLP, the GNN class consistently reduces error both in the transductive evaluation shown here and in the inductive regime studied throughout the paper.

\section{Conclusion}\label{conclusion-sec}
In this work, we addressed the problem of semi-supervised regression on graph-structured data. Inspired by the ``aggregate and readout'' mechanism, we introduced a natural statistical model where node responses are generated by a smooth graph-propagation operator followed by a multivariate nonlinear mapping. Under locality conditions on the receptive fields, we derived a general oracle inequality that decomposes the prediction error into optimization, approximation, and stochastic components. This bound explicitly quantifies the influence of critical factors such as the proportion of unmasked nodes and the underlying graph topology. Building upon this result, we further provided an analysis of the stochastic and approximation error, establishing a non-asymptotic convergence rate for the least-squares estimator when the outer function belongs to a composition of Hölder smoothness classes. These theoretical findings help explain the seemingly contradictory phenomenon observed in practice, whereby GNNs excel with limited labels yet may remain vulnerable to graph perturbations, and offer insights into the design of future GNN architectures. 

We believe this work opens several promising avenues for future research. First, while our results provide explicit bounds on the statistical risk of the least-squares estimator, the optimization error $\Delta_n^{\cF}$, which is governed by the training dynamics, remains to be analyzed. In practice, GNNs are trained via stochastic gradient descent (SGD) rather than global risk minimization. While some progress has been made in characterizing the implicit regularization of SGD in standard i.i.d. regression and classification settings \cite{gunasekar2017implicit, soudry2018implicit}, understanding these dynamics in the presence of graph-induced dependencies remains an open challenge. Investigating how the graph structure influences the optimization landscape and whether SGD induces an implicit regularization that controls the effective capacity of the network, thereby keeping $\Delta_n^{\cF}$ small, is a problem of significant value in its own right.

Second, the present analysis relies on the locality parameter $m$, as specified in Assumption~\ref{ass-m}, to quantify graph-induced dependence. While this effectively captures the behavior of message-passing architectures on bounded-degree graphs, it yields conservative bounds for dense graphs or architectures with global readouts. It is, therefore, interesting to explore in future work how spectral properties of the graph, such as the spectral gap or conductance, could be utilized to derive tighter concentration inequalities that do not rely solely on worst-case receptive field sizes. Notably, \cite{kirichenko2017estimating} has preliminarily investigated this direction within a Bayesian approach for graph denoising without node features.

Finally, the current framework assumes the graph structure is fixed and known; however, in many real-world applications, the observed graph may be noisy or incomplete. Extending the risk analysis to settings where the graph topology is learned jointly with the regression function (latent graph learning) or analyzing the minimax rates under adversarial edge perturbations would be a significant step towards understanding and improving the stability of GNNs.
\section*{Acknowledgements} 
C. Donnat acknowledges support by  the National Science Foundation (Award Number 2238616), as well as the resources provided by the University of Chicago’s Research Computing Center. The work of Olga Klopp was funded by CY Initiative (grant “Investissements d’Avenir” ANR-16-IDEX-0008) and Labex MME-DII (ANR11-LBX-0023-01). This work was partially done  while O. Klopp  and C. Donnat were visiting the Simons Institute for the Theory of Computing. J. S.-H. acknowledges support by the ERC grant A2B (grant agreement number 101124751).

\appendix
\section{Related works}\label{app:related_works}
Semi-supervised learning on graphs has a rich and long history in statistics and machine learning. Earlier work focused on spatial regularization, often ignoring the node features. Notably, building on the notion of smooth functionals on graphs, Belkin et al. \cite{Belkin} developed a regularization framework for semi-supervised learning. They proposed to minimize the squared error loss function augmented by a ridge-type graph smoothness penalty and  provide bounds on the empirical and generalization errors. Crucially, these results suggest that performance is driven by the graph's geometry (typically, through the second eigenvalue of the graph Laplacian), rather than its size $n$. Follow-up work considered the extension of this approach leveraging the $\ell_1$-penalty (also referred to as graph trend filtering), highlighting the strong influence of the graph in driving the error rate \cite{pmlr-v49-huetter16}.

Parallel to regularization approaches, label propagation algorithms were introduced to explicitly exploit graph structure. \cite{xiaojin2002learning} and \cite{zhou2003learning} proposed iterative algorithms where node labels are propagated to neighboring nodes based on edge weights. Theoretically, these methods can be viewed as computing a harmonic function on the graph that is constrained by the boundary conditions induced by the labeled data; specifically, the solution minimizes a quadratic energy function induced by the graph Laplacian. While computationally efficient and theoretically grounded in harmonic analysis, standard label propagation is inherently transductive and typically fails to incorporate the node features $X_i$, relying solely on the graph topology and observed labels.

More recently, Graph Convolutional Networks (GCN) \cite{Max} bridged the gap between feature-based learning and graph regularization. Unlike traditional label propagation, which operates solely on graph topology, GCNs integrate node features $X_i$ directly into the propagation mechanism. By approximating spectral graph convolutions to the first order \cite{Hammond2011, defferrard2016}, the GCN layer performs a localized feature averaging that functions as a learnable low-pass filter. This approach resolves a key limitation of traditional label propagation: it allows for inductive learning on unseen nodes and exploits feature correlations. A variety of GNN architectures continue to be proposed \cite{pmlr-v97-wu19e,zhu2020beyond}; they largely share a foundational spirit with GCNs, involving first the aggregation of node information and then its synthesis into an output via a readout step.

Existing theory on GNNs has progressed along several complementary axes. A first line of work studies expressivity, relating message-passing GNNs to the Weisfeiler–Lehman (WL) hierarchy: standard neighborhood-aggregation architectures are at most as powerful as the 1-WL test in distinguishing graph structures, which motivates more expressive higher-order or invariant/equivariant constructions \citep{xu2019powerful,morris2019weisfeiler,maron2019provably}. A second line develops generalization guarantees for node-level prediction in the transductive/semi-supervised regime, using stability, transductive Rademacher complexity, and PAC-Bayesian tools, and makes explicit how quantities such as graph filters, degree/spectrum, or diffusion operators control the generalization gap \citep{verma2019stability, garg2020generalization,esser2021learning,liao2020pac,ju2023generalization}. A third line clarifies the algorithmic role of propagation through spectral/graph-signal and dynamical-systems lenses: GCN-style layers act as low-pass (Laplacian-smoothing) operators, while repeated propagation can provably induce over-smoothing and loss of discriminative power as depth grows \citep{oono2019graph,keriven2022not}. Despite this progress, existing results often analyze either approximation/expressivity or statistical generalization in isolation, and frequently rely on linearized models or specific graph generative assumptions. This leaves open the need for a nonparametric, finite-sample theory that jointly accounts for message-passing approximation error and statistical complexity under partial labeling and graph-induced dependence.

\section{Scope of the work}\label{app:scope}
Table~\ref{tab:architectures} compares existing mainstream architectures with our compositional framework.
\begin{table}[t]
\centering
\renewcommand{\arraystretch}{0.6}
\begin{tabularx}{0.9\textwidth}{@{} l c X @{}} 
\toprule
\textbf{Architecture} & \textbf{Covered?} & \multicolumn{1}{c}{\textbf{Notes / Why (not)}}\\
\midrule
GCN \citep{Max} / SGC & \checkmark & {\fontsize{10.5}{9}\selectfont Linear propagation with a fixed graph operator (e.g., Laplacian) followed by a learned readout.} \\ \addlinespace[1.2em]

Polynomial GNNs & \checkmark & {\fontsize{10.5}{9}\selectfont Propagation is a polynomial in a fixed operator; fits $\sum_{\ell} \gamma_\ell\bT^\ell\bX\mathbf{W}_\ell$ and related forms.}\\
\addlinespace[1.2em]

APPNP / Diffusion & \checkmark & {\fontsize{10.5}{9}\selectfont Personalized PageRank-style diffusion is linear in features (fixed operator), followed by an MLP readout.}\\ \addlinespace[1.2em]

Skip-connected GCN & \checkmark & {\fontsize{10.5}{9}\selectfont Directly covered by weighted sums of multiple propagation depths.} \\ \addlinespace[1.2em]

GraphSAGE (mean) & \textit{partial} & {\fontsize{10.5}{9}\selectfont Mean aggregation is linear if aggregator weights are fixed; feature-dependent gating falls outside the scope.} \\ \addlinespace[1.2em]

GAT (Attention) & \xmark & {\fontsize{10.5}{9}\selectfont Attention makes the propagation operator data-dependent, violating the fixed-operator assumption.} \\ \addlinespace[1.2em]

MPNN variants & \xmark & {\fontsize{10.5}{9}\selectfont Message-passing functions depend on $(h_i, h_j, e_{ij})$; propagation is no longer a fixed linear operator.} \\
\bottomrule
\end{tabularx}
\caption{Coverage of common GNN architectures by our theoretical framework.}
\label{tab:architectures}
\end{table}
\section{Key supporting lemmas and proofs}
This section introduces the key technical results to establish the main theorem.

Recall from the main text that for any matrix-valued function evaluation $f(\bs{x})\in \mathbb{R}^{n \times d}$, both $f_i(\bs{x})$ and $(f(\bs{x}))_i$ denote the $i$-th row of $f(\bs{x}).$ For a function $f: \mathbb{R}^{n \times d} \rightarrow \mathbb{R}^{n}$, a matrix $\bX\in\R^{n \times d},$ and a binary vector $\bs{\omega}=(\omega_1,\ldots,\omega_n)^{\top}\in\{0,1\}^n$, we define the (semi)-norms
\begin{equation*}
\|f\|_{n,{\bs\omega}}^2=\frac{1}{n} \sum_{i=1}^{n} \omega_i\big(f_i(\bX)\big)^2\quad\mbox{and}\quad\|f\|_{n}^2=\frac{1}{n} \sum_{i=1}^{n}\big(f_i(\bX)\big)^2.
\end{equation*}

To set the stage for the established results, we first provide some necessary preliminary inequalities. 
\begin{lemma}[Bernstein's inequality, see e.g.\ Corollary~2.11 of \cite{boucheron2013concentration}]\label{bernstein}
Let $U_1, \ldots, U_n$ be independent random variables with $\mathbb{E}[U_i] = 0$ and $|U_i| \leq M$ almost surely for all $i$. Then, for any $t > 0$,
\[
\mathbb{P}\left(\left|\sum_{i=1}^n U_i \right|\geq t \right) \leq 2 \exp\left( -\frac{t^2/2}{\sum_{i=1}^n\Var(U_i)+ M t / 3} \right).
\]    
\end{lemma}
The next result is a variant of Talagrand's concentration inequality \cite{Talagrand_1996}. It follows from inversion of the tail bound in Theorem 3.3.16 of \cite{Ginenickl}.
\begin{Theorem}\label{Talagrand} Let $(S, \mathcal{S})$ be a measurable space. Let $X_1,\ldots,X_n$ be independent ${S}$-valued random variables and let $\mathcal{G}$ be a countable set of functions $f=(f_1,...,f_n):S\rightarrow [-K,K]^n$ such that $\mathbb{E}[f_k(X_k)]=0$ for all $f\in \mathcal{G}$ and $k=1,...,n$. Set $$\cZ=\sup_{f \in \mathcal{G}}  \, \sum_{k=1}^n f_k(X_k)$$
and define the variance proxy
$$V_n=2K\mathbb{E}[\cZ] + \sup_{f \in \mathcal{G}}  \, \sum_{k=1}^n \mathbb{E} \left [\big(f_k(X_k)\big)^2\right].$$
Then, for all $t \geq 0$,
\begin{equation*}
\mathbb{P} \left (\cZ - \mathbb{E}[\cZ]\geq t \right ) \leq \exp \Big(\frac{-t^2}{4V_n + (9/2)Kt}\Big).
\end{equation*}
\end{Theorem}

We now establish the following two lemmas, which bridge the semi-supervised empirical loss and the overall nodal prediction performance. 
\begin{lemma}\label{lm:isometry} 
Let $f^*$ be a function and $\mathcal{G}$ a countable class of functions with $\log|\mathcal{G}|\geq1$, all mapping $\R^{n\times d}$ to $\R^n$ and having the form $f(\bs{x}) = \big(f_1(\bs{x}), \ldots, f_n(\bs{x})\big)^{\top}$ with $\|f_i\|_{\infty} \leq F$ for all $i$ and some constant $F\geq1$. With probability at least $1-2/|\cG|$, for all $f\in\cG$ such that $\pi n\left\|f-f^*\right\|_{n}^{2}>3600F^2\log |\cG|$, we have 
\begin{equation}\label{st:isometry}
\left \Vert f-f^*\right \Vert^{2}_{n,{\bs\omega}}\geq\frac{\pi\left \Vert f-f^*\right \Vert^{2}_{n}}{2}.
\end{equation} 
\end{lemma}
\begin{proof}
We apply the peeling argument. Let $\nu =3600F^2\log|\cG|/\pi$ and set $\alpha=6/5$. The event 
$$\cB=\left\{\exists\,f\in\cF\,\text{such that}\,\ n\left\|f-f^*\right\|_{n}^{2}>\nu,\left\Vert f-f^*\right \Vert^{2}_{n,{\bs\omega}}<\frac{\pi\left \Vert f-f^*\right\Vert^{2}_{n}}{2}\right\}$$ is the complement of the event that we wish to analyze. For $\ell\in\N$, we define the sets
$$S_{\ell}=\left \{f\,:\,\alpha^{\ell-1}\nu< n\left \Vert f-f^*\right \Vert^{2}_{n}\leq \alpha^{\ell}\nu\right \}$$ and the corresponding events
$$\mathcal{B}_{\ell}=\left \{\exists\,f\in S_{\ell}\ \text{such that}\ \pi\left \Vert f-f^*\right\Vert^{2}_{n}-\left\Vert f-f^*\right \Vert^{2}_{n,{\bs\omega}}>\frac{\pi\nu\alpha^{\ell-1}}{2n}\right \}.$$
In fact, we can restrict the consideration to a finite $\ell$ since $\|f-f^*\|_n^2\leq4F^2.$ If the event $\mathcal{B}$ holds for some $f$, then $f$ belongs to some $S_{\ell}$ and $\mathcal{B} \subset \bigcup_{\ell=1}^{\infty} \mathcal{B}_{\ell}$. Lemma \ref{lm:sup} implies that $$\mathbb P\left (\mathcal{B}_{\ell}\right )\leq \exp \left (-5.6\alpha^{\ell}\log|\cG|\right ).$$ Applying the union bound, we obtain that
\begin{align*}
\mathbb P\left (\mathcal{B}\right )&\leq\sum_{\ell=1}^{\infty} \P\left (\mathcal{B}_{\ell}\right )\\
&\leq \sum_{\ell=1}^{\infty}\exp(-5.6\alpha^{\ell}\log|\cG|)\\
&\leq\sum_{\ell=1}^{\infty}\exp\left (-5.6\ell\log|\cG|\log\alpha\right)\\
&\leq \dfrac{\exp\left (-5.6\log|\cG|\log\alpha\right )}{1-\exp\left (-5.6\log|\cG|\log\alpha\right )}
\\&\leq \dfrac{\exp\left (-\log|\cG|\right )}{1-\exp\left (-\log|\cG|\right)}.
\end{align*} 
Given that $\log|\mathcal{G}| \geq 1$, the proof is complete.
\end{proof}

\begin{lemma}\label{lm:sup}
Suppose $f^*=(f_1^*,\ldots,f_n^*)^{\top}$ with $\|f_i^*\|_{\infty}\leq F$. Let $\mathcal{G}$ be a countable class of functions $f: \R^{n\times d} \to \R^{n}$ of the form $f(\bs{x}) = \big(f_1(\bs{x}), \ldots, f_n(\bs{x})\big)^{\top}$, where each component function $f_i:\R^{n\times d} \to \R$ satisfies $\|f_i\|_{\infty} \leq F$ for some constant $F > 0$. Let $\alpha=6/5$ and let $\nu=3600F^2\log|\cG|/\pi$ with $\log|\cG|\geq 1$. For $\ell\in\mathbb{N}$, define
$$S_{\ell}=\left \{f\in\cG\,:\,\alpha^{\ell-1}\nu<n\left \Vert f-f^{*}\right \Vert^{2}_{n}\leq \alpha^{\ell}\nu\right \},$$ and 
$$Z_{\ell}=\underset{f\in S_{\ell}}{\sup}\left(\pi\left \Vert f-f^*\right \Vert^{2}_{n}-\left\Vert f-f^*\right \Vert^{2}_{n,{\bs\omega}}\right).$$ 
Then, for each $\ell\in\mathbb{N},$ 
\begin{equation*}
\mathbb{P}\left(Z_{\ell}>\frac{\pi\nu\alpha^{\ell-1}}{2n}\right )\leq \exp \left (-5.6\alpha^{\ell}\log|\cG|\right ).
	\end{equation*}
\end{lemma}
\begin{proof}
We first provide an upper bound on $\mathbb{E}[Z_{\ell}]$ and then show that $Z_{\ell}$ concentrates around its expectation. Let $\eta_{f}=\sum_{i=1}^n(\pi-\omega_{i})\left(f_{i}\left({\bs{x}}\right)-f^*_i\left({\bs{x}}\right)\right)^{2}.$ By definition of $Z_{\ell}$, we have
$$Z_{\ell}=\underset{f\in S_{\ell}}{\sup}\dfrac{1}{n}\sum_{i=1}^n(\pi-\omega_{i})\left(f_{i}\left({\bs{x}}\right)-f^{*}_i\left({\bs{x}}\right)\right)^{2}=\underset{f\in S_{\ell}}{\sup}\dfrac{\eta_f}{n}.$$
Observe that
$$\big|(\pi-\omega_i)\big(f_{i}\left({\bs{x}}\right)-f^{*}_i\left({\bs{x}}\right)\big)^2\big|\leq 4F^2,$$
and for $f\in S_{\ell}$,
\begin{align}
\sum_{i=1}^n\Var \left [(\omega_{i}-\pi)\left(f_{i}\left({\bs{x}}\right)-f^{*}_i\left({\bs{x}}\right)\right)^{2}\right ]&\leq 4F^{2}\pi(1-\pi)\sum_{i=1}^n\left(f_{i}\left({\bs{x}}\right)-f^{*}_i\left({\bs{x}}\right)\right)^{2}\nonumber\\
&\leq 4F^{2}\pi\nu\alpha^{\ell}.\label{bound_variance_GineN}
\end{align}
Applying Bernstein's inequality (Lemma~\ref{bernstein}), we derive that for all $f\in S_{\ell}$,
$$\P\left(\eta_{f}\geq t\right )\leq \exp\left(\dfrac{-t^{2}/2}{4F^{2}\pi\nu\alpha^{\ell}+4F^{2}t/3}\right).$$
For any ${\cT}\geq \pi\nu\alpha^{\ell}/5$, the union bound gives	
\begin{align*}
\E[nZ_{\ell}] & \leq\int_{0}^{\infty} \P(nZ_{\ell} \geq t) \, d t\\
&\leq {\cT}+\int_{{\cT}}^{\infty} \P(nZ_{\ell}\geq t) \, dt\\
& \leq {\cT}+|\cG|\int_{{\cT}}^{\infty} \exp\left(\dfrac{-t^{2}/2}{4F^{2}\pi\nu\alpha^{\ell}+4F^{2}t/3}\right)dt\\
&\leq {\cT}+|\cG|\int_{{\cT}}^{\infty} \exp\left(\dfrac{-t}{44F^{2}}\right)dt\\
&={\cT}+44|\cG|F^{2}e^{-{\cT}/(44F^{2})}.
\end{align*}
Taking ${\cT}=\pi\nu \alpha^{\ell}/5$ and using the facts that $\nu=3600F^{2}\log|\cG|/\pi$ and $\log|\cG|\geq1$, we can deduce that
\begin{equation}\label{expect-bound}
\E[Z_{\ell}] \leq \dfrac{5\pi\nu \alpha^{\ell}}{24n},\quad\mbox{for any\ }\ell\in\mathbb{N}.    
\end{equation}

Next, we show that $Z_{\ell}$ concentrates around its expectation by applying Talagrand's concentration inequality (Theorem~\ref{Talagrand}). For each $\ell$, we apply Theorem~\ref{Talagrand} with $$\cZ=\sup_{f\in S_{\ell}}\eta_f=nZ_{\ell},$$ which implies $K=4F^{2}$. Combining \eqref{bound_variance_GineN} and \eqref{expect-bound}, we also know that
$$V_{n}=8F^{2}\E\cro{\cZ}+4F^{2}\pi\nu\alpha^{\ell}\leq\frac{5F^{2}\pi\nu \alpha^{\ell}}{3}+4F^{2}\pi\nu\alpha^{\ell}\leq 6F^{2}\pi\nu\alpha^{\ell}.$$
Hence, applying Theorem~\ref{Talagrand} with $t=\dfrac{5\pi\nu \alpha^{\ell}}{24}$, we obtain
\begin{align*}
\mathbb{P}\left (nZ_{\ell} > \frac{\pi\nu \alpha^{\ell-1}}{2}\right)&\leq\P\left ( nZ_{\ell}-n\E[Z_{\ell}]>\dfrac{5\pi\nu \alpha^{\ell}}{24}\right)\\
&\leq \exp\bigg(-\frac{25\pi\nu\alpha^{\ell}}{24(24^2+90)F^2}\bigg)\\
&\leq \exp \left (-5.6\alpha^{\ell}\log|\cG|\right ),
\end{align*}
which completes the proof.    
\end{proof}

\section{Proof of main theorems}\label{main-proof}
\subsection{Auxiliary results}
To prepare for the proof of Theorem~\ref{main}, we first state two preliminary lemmas.
\begin{lemma}\label{general-holder}
Let $U_1, \ldots, U_r$ be nonnegative random variables. Then,
\[
\mathbb{E}\left[ \prod_{k=1}^r U_k \right] \leq \prod_{k=1}^r \left( \mathbb{E}[U_k^{r}] \right)^{1/r}.
\]
\end{lemma}
\begin{proof}
The statement is a consequence of the extension of H\"older's inequality to several functions. To verify this, one should choose the indices in H\"older's inequality to be $p_1=\ldots=p_r=r$ which then gives $1/p_1+\ldots +1/p_r=1.$
\end{proof}
\begin{lemma}\label{lem:bounded-mgf}
Let $V_1,\ldots,V_k$ be independent, real-valued random variables with $\mathbb E[V_i]=0$ and $|V_i|\leq M$ almost surely for each $i$. Then, for $0\leq \lambda<3/M$,
\[
\mathbb E\!\left[\exp\!\left(\lambda\sum_{i=1}^k V_i\right)\right]
\ \leq\ \exp\!\left\{\frac{\lambda^2}{2(1-\lambda M/3)}\sum_{i=1}^k \mathbb E[V_i^2]\right\}.
\]
\end{lemma}
\begin{proof}
Using $j!\geq 2\cdot 3^{\,j-2}$ for $j\geq 2$, we have for $|t|<3$,
$e^t\leq 1+t+t^2\big/ [2(1-|t|/3)]$. This implies that for any $\lambda$ satisfying $0\leq \lambda<3/M$,
\[
e^{\lambda V_i}\leq 1+\lambda V_i+\frac{\lambda^2 V_i^2}{2(1-\lambda M/3)}.
\]
Taking expectations on both sides and using $\mathbb E[V_i]=0$ gives
\begin{align*}
\mathbb E\!\left[e^{\lambda V_i}\right]
\leq1+\frac{\lambda^2}{2(1-\lambda M/3)}\mathbb E[V_i^2]\leq \exp\!\left\{\frac{\lambda^2}{2(1-\lambda M/3)}\mathbb E[V_i^2]\right\},
\end{align*}
where the last inequality uses $1+u\leq e^u$ for all $u\in\mathbb R$. By independence, this bound extends to the sum $\sum_{i=1}^kV_i$.
\end{proof}
\subsection{Proof of Theorem~\ref{main}}\label{thm-1-proof}
For any estimator $\widetilde f\in\cF$, define 
\begin{equation}\label{define-delta-tilde}
\Delta_{n}^{\cF}\left(\widetilde f,f^*\big|\bX\right)=\E_{{\bs\varepsilon},{\bs\omega}}\bigg[\frac{1}{n} \sum_{i\in \Omega}\Big(Y_{i}-\widetilde{f}_i(\mathbf X)\Big)^{2}-\inf_{f \in \mathcal{F}} \frac{1}{n} \sum_{i\in\Omega}\Big(Y_{i}-{f}_i(\mathbf X)\Big)^{2}\ \big|\ \bX\bigg],   
\end{equation}
and consequently, $$\Delta_{n}^{\cF}\left(\widetilde f,f^*\right)=\E_{\bX}\cro{\Delta_{n}^{\cF}\left(\widetilde f,f^*\big|\bX\right)}.$$ 
\begin{proof}[Proof of Theorem~\ref{main}]
We may restrict to the case $\log \mathcal{N}_{\delta} \leq n$. Since $\cR(\widetilde{f},f^*)\leq 4F^{2}$, the upper bound holds trivially when $\log \mathcal{N}_{\delta} \geq n$. To verify that the lower bound is also valid in this case, let 
$$\bar{f}\in\argmin_{f\in \mathcal{F}}\frac{1}{n}\sum_{i=1}^{n}\omega_i\big(Y_i-{f}_i(\bX)\big)^2$$ be an empirical risk minimizer over $\cF$. Observe that
\begin{align}
&\E_{{\bs{\varepsilon}},{\bs\omega}}\cro{\|\widetilde{f}-f^*\|^2_{n,{\bs\omega}}}-\E_{{\bs{\varepsilon}},{\bs\omega}}\cro{\|\bar{f}-f^*\|^2_{n,{\bs\omega}}}\nonumber\\
&=\Delta_{n}^{\cF}\left(\widetilde f,f^*\big|\bX\right)+\mathbb E_{{\bs{\varepsilon}},{\bs\omega}}\left[\frac{2}{n}\sum_{i=1}^{n} \varepsilon_{i}\omega_{i} \widetilde{f}_i\left(\bX\right)\right]-\mathbb E_{{\bs{\varepsilon}},{\bs\omega}}\left[\frac{2}{n} \sum_{i=1}^{n} \varepsilon_{i}\omega_{i} \bar{f}_i\left(\bX\right)\right],\label{eq:trivial_lb}
\end{align}
which implies that almost surely
\begin{align*}
\Delta_{n}^{\cF}\left(\widetilde f,f^*\big|\bX\right)&\leq\E_{{\bs{\varepsilon}},{\bs\omega}}\cro{\|\widetilde{f}-f^*\|^2_{n,{\bs\omega}}}+\left|\mathbb E_{{\bs{\varepsilon}},{\bs\omega}}\left[\frac{2}{n}\sum_{i=1}^{n} \varepsilon_{i}\omega_{i} \widetilde{f}_i\left(\bX\right)\right]\right|+\left|\mathbb E_{{\bs{\varepsilon}},{\bs\omega}}\left[\frac{2}{n} \sum_{i=1}^{n} \varepsilon_{i}\omega_{i} \bar{f}_i\left(\bX\right)\right]\right|\\
&\leq4F^2+4F\E_{\bs\varepsilon}\left(\frac{1}{n}\sum_{i=1}^n|\varepsilon_i|\right)\\
&\leq8F^2.
\end{align*}
Thus $\Delta_{n}^{\cF}(\widetilde f,f^*)\leq8F^2$. Since $m,F\geq1$, the lower bound also holds for $\log \mathcal{N}_{\delta} \geq n$.

In the following, we consider the case when $1\leq\log \mathcal{N}_{\delta} \leq n$. The proof proceeds in five steps that are denoted by (I)-(V).
\begin{itemize}
\item Step (I): Conditionally on $\bX$, for any estimator $\widetilde f\in\cF$, we bound 
$$\left| \mathbb{E}_{{\bs{\varepsilon}},{\bs\omega}} \left[ \frac{2}{n} \sum_{i=1}^{n} \varepsilon_{i} \omega_{i}\widetilde f_i(\mathbf{X})\right] \right|\leq6\sqrt{\pi}\delta+4\sqrt{\frac{ \E_{{\bs{\varepsilon}},{\bs\omega}}\cro{\|\widetilde{f}-f^*\|^2_{n,{\bs\omega}}}\log \mathcal{N}_{\delta}}{n}}.$$
\item Step (II): Conditionally on $\bX$, we show for any $\varepsilon\in(0,1]$, 
$$\E_{{\bs{\varepsilon}},{\bs\omega}}\cro{\|\widetilde{f}-f^*\|^2_{n,{\bs\omega}}}\hspace{-2pt}\leq \hspace{-2pt}(1+\varepsilon)\hspace{-3pt} \cro{\inf_{f \in \mathcal{F}} \pi \|f - f^{*}\|_{n}^{2} + 6\delta \sqrt{\pi} +\hspace{-2pt} \frac{4(1+\varepsilon)\log \mathcal{N}_{\delta}}{n \varepsilon}\hspace{-2pt} + \Delta_{n}^{\cF}\left(\widetilde f,f^*\big|\bX\right)}.$$ 
\item Step (III): Conditionally on $\bX$, we relate $\E_{{\bs{\varepsilon}},{\bs\omega}}[\|\widetilde{f}-f^*\|^2_{n,{\bs\omega}}]$ to $\E_{{\bs{\varepsilon}},{\bs\omega}}[\|\widetilde{f}-f^*\|^2_{n}]$. Via isometry, we prove that for any $\varepsilon>0$, 
\begin{align*}
&\E_{{\bs{\varepsilon}},{\bs\omega}}\cro{\|\widetilde{f}-f^*\|^2_{n}}\\&\leq2(1+\varepsilon)\hspace{-3pt}\cro{\inf_{f\in \mathcal{F}}\|f-f^*\|_n^2+\hspace{-3pt}\frac{1+\varepsilon}{\varepsilon}\frac{1804F^2\log \mathcal{N}_{\delta}}{n\pi}\hspace{-2pt} + \hspace{-2pt}\dfrac{4F^{2}}{\mathcal{N}_{\delta}}\hspace{-1pt}+\hspace{-2pt}\frac{6\delta}{\sqrt{\pi}}\hspace{-1pt}+\hspace{-2pt}12F\delta+\hspace{-2pt}\frac{\Delta_{n}^{\cF}\left(\widetilde f,f^*\big|\bX\right)}{\pi}}.   \end{align*}
\item Step (IV): We connect $\mathcal{R}(\widetilde{f}, f^*)$ and $\E_{{\bs{\varepsilon}},{\bs\omega},\bX}[\|\widetilde{f}-f^*\|^2_{n}]$ via
\begin{align*}
(1-\varepsilon)\E_{{\bs{\varepsilon}},{\bs\omega},\bX}&\cro{\|\widetilde{f}-f^*\|^2_{n}}-\frac{13m^2F^2\log \mathcal{N}_{\delta}}{n\varepsilon}-16\delta F\leq\mathcal{R}\left(\widetilde{f}, f^*\right)\\ &\leq (1+\varepsilon) \left(\E_{{\bs{\varepsilon}},{\bs\omega},\bX}\cro{\|\widetilde{f}-f^*\|^2_{n}}+\frac{10(1+\varepsilon)m^2 F^{2}}{\varepsilon} \frac{\log \mathcal{N}_{\delta}}{n}+12\delta F\right).
\end{align*}
\item Step (V): For any $\varepsilon\in(0,1]$, we prove the lower bound $$\E_{{\bs{\varepsilon}},{\bs\omega}}\cro{\|\widetilde{f}-f^*\|^2_{n}}\geq (1-\varepsilon) \left( \Delta_{n}^{\cF}\left(\widetilde f,f^*\big|\bX\right) - \frac{4\log \mathcal{N}_{\delta}}{n \varepsilon} - 12\delta \sqrt{\pi}\right).$$
\end{itemize}
We then obtain the asserted lower bound of the theorem by taking $\mathbb{E}_{\bX}$ of Step (V) and combining this with Step (IV), and the corresponding upper bound by taking $\mathbb{E}_{\bX}$ of Step (III) and combining this with Step (IV).

\textbf{Step (I): Bounding the expectation of noise terms.}
For any estimator $\widetilde f\in\cF$ whose $i$-th component is denoted by $\widetilde f_{i}$, we first show that
\begin{equation}\label{eq:I}
\left| \mathbb{E}_{{\bs{\varepsilon}},{\bs\omega}} \left[ \frac{2}{n} \sum_{i=1}^{n} \varepsilon_{i} \omega_{i}\widetilde f_{i}(\mathbf{X})\right] \right|\leq6\sqrt{\pi}\delta+4\sqrt{\frac{ \E_{{\bs{\varepsilon}},{\bs\omega}}\cro{\|\widetilde{f}-f^*\|^2_{n,{\bs\omega}}}\log \mathcal{N}_{\delta}}{n}}.
\end{equation}
		
\textit{ Step 1.1: Using a covering argument.}
By the definition of $\cF_{\delta}$ as a $\delta$-covering of $\mathcal{F}$, for any (random) estimator $\widetilde f\in\cF$, there exists a (random) function $f'\in \cF_{\delta}$ such that for every index $i \in \{1,\ldots,n\}$,
\begin{equation}\label{cover-delta}
\|\widetilde f_i - f'_i\|_{\infty}\leq \delta.
\end{equation}
This implies that 
$$\left| \mathbb{E}_{{\bs{\varepsilon}},{\bs\omega}}\cro{\sum_{i=1}^{n} \varepsilon_{i} \omega_{i} \left(\widetilde f_i(\mathbf{X})- f'_i(\mathbf{X})\right)}\right| \leq \delta \mathbb{E}_{{\bs{\varepsilon}},{\bs\omega}}\cro{\sum_{i=1}^{n} \omega_i |\varepsilon_i|}.$$
Since $\mathbb{E}_{{\bs{\varepsilon}},{\bs\omega}}(\omega_{i}|\varepsilon_i|) \leq\pi$, for each $i$, we obtain
\begin{equation}\label{cover-diff}
\left| \mathbb{E}_{{\bs{\varepsilon}},{\bs\omega}}\cro{\sum_{i=1}^{n} \varepsilon_{i} \omega_{i} \left(\widetilde f_i(\mathbf{X})- f'_i(\mathbf{X})\right)}\right|\leq n \pi\delta.    
\end{equation}

\textit{Step 1.2: Using  Gaussian concentration.} 
Recall that ${\bs{\omega}}=(\omega_1,\ldots,\omega_n)^{\top}$. For any fixed $f\in\cF_{\delta}$, we define the random variable  
$$\xi_{f}= \frac{\sum_{i=1}^{n} \varepsilon_i \omega_i \cro{f_i(\mathbf{X}) - f_i^*(\mathbf{X})}}{\sqrt{n}\|f^*-f\|_{n,{\bs\omega}}}.$$
Conditionally on $\bs{\omega}$ and with $\mathbf{X}$ fixed, since the $\varepsilon_i \sim \mathcal{N}(0,1)$ are mutually independent, it follows that $\xi_{f}\sim \mathcal{N}(0,1)$. Applying Lemma C.1 in \cite{Schmidt-Hieber} to the second inequality, we can derive that 
\begin{equation}\label{bound-2-by-cover}
\E_{{\bs{\varepsilon}},{\bs\omega}}\cro{\xi^{2}_{f'}}=\E_{\bs{\omega}}\cro{\E_{\bs{\varepsilon}}\left(\xi^{2}_{f'}\mid\bs{\omega}\right)
}\leq\E_{\bs\omega}\cro{\E_{\bs{\varepsilon}}\left(\max_{f\in\cF_{\delta}}\xi^{2}_f\mid\bs\omega\right)}\leq3 \log \mathcal{N}_{\delta}+1.    
\end{equation}
Using the Cauchy-Schwarz inequality and \eqref{bound-2-by-cover} yields
\begin{align}
\mathbb{E}_{{\bs{\varepsilon}},{\bs\omega}}\left[\|f'-f^*\|_{n,{\bs\omega}}|\xi_{f'}| \right]&\leq\sqrt{\mathbb{E}_{{\bs{\varepsilon}},{\bs\omega}}\left[\|f'-f^*\|_{n,{\bs\omega}}^2\right]} \sqrt{\E_{{\bs{\varepsilon}},{\bs\omega}}\cro{\xi^{2}_{f'}}}\nonumber\\
&\leq\sqrt{\mathbb{E}_{{\bs{\varepsilon}},{\bs\omega}}\left[\|f'-f^*\|_{n,{\bs\omega}}^2\right]} \sqrt{3\log \mathcal{N}_{\delta} + 1}\nonumber\\
&\leq \left (\sqrt{\mathbb{E}_{{\bs{\varepsilon}},{\bs\omega}}\left[\|\widetilde f-f^*\|_{n,{\bs\omega}}^2\right]}+\sqrt{\pi}\delta\right ) \sqrt{3\log \mathcal{N}_{\delta}+1}.\label{step1-b1}
\end{align}
Observe that since $f^*$ is fixed, then for each $i=1,\ldots,n,$
\begin{align}\label{exp-zero}
\mathbb{E}_{{\bs{\varepsilon}},{\bs\omega}}\left[\varepsilon_{i} \omega_{i} f^{*}_i(\mathbf{X})\right]=\pi\mathbb{E}_{{\bs\varepsilon}}\left[\varepsilon_{i}f^{*}_i(\mathbf{X})\right]=0.
\end{align}
Thus, following from \eqref{cover-diff} and \eqref{exp-zero}, we have for any estimator $\widetilde f$,
\begin{align}
\left|\mathbb{E}_{{\bs{\varepsilon}},{\bs\omega}}\left[ \frac{2}{n} \sum_{i=1}^{n}\varepsilon_{i} \omega_{i}\widetilde f_i(\mathbf{X})\right] \right|
&=\left|\mathbb{E}_{{\bs{\varepsilon}},{\bs\omega}}\left[ \frac{2}{n} \sum_{i=1}^{n} \varepsilon_{i} \omega_{i}\cro{\widetilde f_i(\mathbf{X})-f_i^*(\mathbf{X})\big)}\right] \right|\nonumber\\
&\leq2\pi \delta +\left| \mathbb{E}_{{\bs{\varepsilon}},{\bs\omega}}\left[ \frac{2}{n} \sum_{i=1}^{n} \varepsilon_{i} \omega_{i}\cro{f'_i(\mathbf{X})-f_i^*(\mathbf{X})}\right] \right|\nonumber\\
&\leq2\pi \delta +\frac{2}{\sqrt{n}}\mathbb{E}_{{\bs{\varepsilon}},{\bs\omega}}\left[\|f'-f^*\|_{n,{\bs\omega}}|\xi_{f'}| \right].\label{step1-b2}  \end{align}
Plugging \eqref{step1-b1} into \eqref{step1-b2} yields
\begin{align*}
\left| \mathbb{E}_{{\bs{\varepsilon}},{\bs\omega}}\left[ \frac{2}{n} \sum_{i=1}^{n} \varepsilon_{i} \omega_{i}\widetilde f_i(\mathbf{X})\right] \right|&\leq2\pi\delta+2\left (\sqrt{\mathbb{E}_{{\bs{\varepsilon}},{\bs\omega}}\left[\|\widetilde f-f^*\|_{n,{\bs\omega}}^2\right]}+\sqrt{\pi}\delta\right ) \sqrt{\frac{3\log \mathcal{N}_{\delta}+1}{n}}\\
&\leq2\pi\delta+4\sqrt{\frac{ \mathbb{E}_{{\bs{\varepsilon}},{\bs\omega}}\left[\|\widetilde f-f^*\|_{n,{\bs\omega}}^2\right]\log \mathcal{N}_{\delta}}{n}}+4\sqrt{\pi}\delta\\
&\leq6\sqrt{\pi}\delta+4\sqrt{\frac{ \mathbb{E}_{{\bs{\varepsilon}},{\bs\omega}}\left[\|\widetilde f-f^*\|_{n,{\bs\omega}}^2\right]\log \mathcal{N}_{\delta}}{n}},    
\end{align*}
where the second inequality follows from $1 \leq \log \mathcal{N}_{\delta} \leq n$, and the last from $0<\pi\leq1$. This completes Step (I).

\textbf{ Step (II): Bounding the term $\mathbb{E}_{{\bs{\varepsilon}},{\bs\omega}}[\|\widetilde f-f^*\|_{n,{\bs\omega}}^2]$.}
The goal of this step is to show that for any (random) estimator $\widetilde{f} \in \mathcal{F}$ and any $\varepsilon\in(0,1]$, 
$$\mathbb{E}_{{\bs{\varepsilon}},{\bs\omega}}\left[\|\widetilde f-f^*\|_{n,{\bs\omega}}^2\right] \leq (1+\varepsilon) \cro{\inf_{f \in \mathcal{F}} \pi\|f - f^{*}\|_{n}^{2} + 6\delta \sqrt{\pi} + \frac{4(1+\varepsilon)\log \mathcal{N}_{\delta}}{n \varepsilon} + \Delta_{n}^{\cF}\left(\widetilde f,f^*\big|\bX\right)}.$$

\textit{Step 2.1: Decomposing the empirical risk difference.}
By the definition of ${\Delta}_{n}^{\mathcal{F}}(\widetilde{f}, f^*|\bX)$ in \eqref{define-delta-tilde}, for any deterministic function $f \in \mathcal{F}$, we have
\begin{align*}
\E_{{\bs{\varepsilon}},{\bs\omega}}\cro{\frac{1}{n} \sum_{i=1}^{n}\omega_{i}\big(Y_{i}-\widetilde{f}_i(\mathbf X)\big)^{2}}&\leq\E_{{\bs{\varepsilon}},{\bs\omega}}\cro{\frac{1}{n} \sum_{i=1}^{n}\omega_{i}\big(Y_{i}-f_i(\mathbf X)\big)^{2}}+\Delta_{n}^{\cF}\left(\widetilde f,f^*\big|\bX\right) 
\end{align*}
and therefore
\begin{align}
&\mathbb{E}_{{\bs{\varepsilon}},{\bs\omega}}\left[\|\widetilde f-f^*\|_{n,{\bs\omega}}^2\right]\nonumber\\
&=\mathbb{E}_{{\bs{\varepsilon}},{\bs\omega}}\cro{\frac{1}{n}\sum_{i=1}^n\omega_i\big(\widetilde f_i(\bX)-f^*_i(\bX)\big)^2}\nonumber\\
&=\mathbb{E}_{{\bs{\varepsilon}},{\bs\omega}}\cro{\frac{1}{n}\sum_{i=1}^n\omega_i\big(\widetilde f_i(\bX)-Y_i+\varepsilon_i\big)^2}\nonumber\\
&=\mathbb{E}_{{\bs{\varepsilon}},{\bs\omega}}\cro{\frac{1}{n}\sum_{i=1}^n\omega_i\left(\big(\widetilde f_i(\bX)-Y_i\big)^2+2\varepsilon_i\big(\widetilde f_i(\bX)-Y_i\big)+\varepsilon_i^2\right)}\nonumber\\
&\leq\mathbb{E}_{{\bs{\varepsilon}},{\bs\omega}}\cro{\frac{1}{n}\sum_{i=1}^n\omega_i\left(\big(f_i(\bX)-Y_i\big)^2+2\varepsilon_i\big(\widetilde f_i(\bX)-Y_i\big)+\varepsilon_i^2\right)}+\Delta_{n}^{\cF}\left(\widetilde f,f^*\big|\bX\right)\nonumber\\
&\leq\mathbb{E}_{{\bs{\varepsilon}},{\bs\omega}}\cro{\frac{1}{n}\sum_{i=1}^n\omega_i\left(\left(f_i(\mathbf X)-f_i^*(\mathbf X)-\varepsilon_i\right)^2+2\varepsilon_i\widetilde{f}_i(\mathbf X)-\varepsilon_i^2\right)}+\Delta_{n}^{\cF}\left(\widetilde f,f^*\big|\bX\right)\nonumber\\
&\leq\pi\|f - f^*\|_{n}^{2}+\E_{{\bs{\varepsilon}},{\bs\omega}}\left[\frac{2}{n} \sum_{i=1}^{n} \varepsilon_{i}\omega_{i} \widetilde{f}_i(\mathbf X)\right]+\Delta_{n}^{\cF}\left(\widetilde f,f^*\big|\bX\right).\label{step-21-b}
\end{align}

\textit{Step 2.2: Applying noise expectation bound.} 
For any real numbers $a,b\geq0$ and any parameter $\varepsilon>0$,
$$2 \sqrt{ab} \leq \frac{\varepsilon}{1+\varepsilon} a + \frac{1+\varepsilon}{\varepsilon} b,$$
which implies that
\begin{align*}
4\sqrt{\frac{\mathbb{E}_{{\bs{\varepsilon}},{\bs\omega}}\left[\|\widetilde f-f^*\|_{n,{\bs\omega}}^2\right]\log \mathcal{N}_{\delta}}{n}}&\leq \frac{\varepsilon}{1+\varepsilon} \mathbb{E}_{{\bs{\varepsilon}},{\bs\omega}}\left[\|\widetilde f-f^*\|_{n,{\bs\omega}}^2\right] + \frac{1+\varepsilon}{\varepsilon} \frac{4\log \mathcal{N}_{\delta}}{n}. \end{align*}
Using the result \eqref{eq:I} from Step (I), we derive for any $\varepsilon>0$,
$$\left|\mathbb{E}_{{\bs\omega},{\bs\varepsilon}} \left[ \frac{2}{n} \sum_{i=1}^{n}\varepsilon_i\omega_i \widetilde f_{i}(\mathbf X)\right] \right|\leq\frac{\varepsilon}{1+\varepsilon} \mathbb{E}_{{\bs{\varepsilon}},{\bs\omega}}\left[\|\widetilde f-f^*\|_{n,{\bs\omega}}^2\right]+\frac{1+\varepsilon}{\varepsilon} \frac{4\log \mathcal{N}_{\delta}}{n}+ 6\sqrt{\pi}\delta.$$
It therefore follows from \eqref{step-21-b} that for any $\varepsilon>0,$
\begin{align*}
&\mathbb{E}_{{\bs{\varepsilon}},{\bs\omega}}\left[\|\widetilde f-f^*\|_{n,{\bs\omega}}^2\right]\leq (1+\varepsilon) \left[ \inf_{f\in \mathcal{F}}\pi\|f-f^*\|_n^2+ \frac{1+\varepsilon}{\varepsilon}\frac{4\log \mathcal{N}_{\delta}}{n} + 6\sqrt{\pi}\delta+\Delta_{n}^{\cF}\left(\widetilde f,f^*\big|\bX\right)\right].
\end{align*}

\textbf{Step (III): Bounding the estimation risk via isometry.} In this step, we relate $\mathbb{E}_{{\bs{\varepsilon}},{\bs\omega}}[\|\widetilde f-f^*\|_{n,{\bs\omega}}^2]$ to $\mathbb{E}_{{\bs{\varepsilon}},{\bs\omega}}[\|\widetilde f-f^*\|_{n}^2]$ and show for any $\varepsilon>0,$
\begin{align*}
&\mathbb{E}_{{\bs{\varepsilon}},{\bs\omega}}\cro{\|\widetilde f-f^*\|_{n}^2}\\
&\leq2(1+\varepsilon)\cro{\inf_{f\in \mathcal{F}}\|f-f^*\|_n^2+ 1804\frac{(1+\varepsilon)F^2\log \mathcal{N}_{\delta}}{\varepsilon n\pi} + \dfrac{4F^{2}}{\mathcal{N}_{\delta}}+\frac{6\delta}{\sqrt{\pi}}+12F\delta+\frac{\Delta_{n}^{\cF}\left(\widetilde f,f^*\big|\bX\right)}{\pi}}.    
\end{align*}

Define $\cC_1$ as the event that
$$\pi n\|\widetilde{f}-f^*\|_n^2 >3600F^2\log\mathcal{N}_\delta+6F\pi n\delta.$$ Observe that
\begin{align}
\bE_{\bs\varepsilon,\bs\omega}\cro{\|\widetilde f-f^*\|^{2}_{n}}=\bE_{\bs\varepsilon,\bs\omega}\left [\|\widetilde f-f^*\|^{2}_{n}\1(\mathcal{C}_1)\right]+\bE_{\bs\varepsilon,\bs\omega}\left [\|\widetilde f-f^*\|^{2}_{n}\1\left(\mathcal{C}_1^{c}\right)\right]\label{n-risk-b}    
\end{align}
and
\begin{equation}\label{step3-c-1-comple}
\bE_{\bs\varepsilon,\bs\omega}\left [\|\widetilde f-f^*\|^{2}_{n}\1\left(\mathcal{C}_1^{c}\right)\right]\leq\frac{3600F^2\log\mathcal{N}_\delta}{\pi n}+6F\delta.    
\end{equation}
To bound $\bE_{\bs\varepsilon,\bs\omega}\left [\|\widetilde f-f^*\|^{2}_{n}\1(\mathcal{C}_1)\right]$, we apply Lemma~\ref{lm:isometry}. Recall that $f'\in\cF_{\delta}$ is the (random) function closest to $\widetilde{f}$, meaning that for all $i=1,\ldots,n$,
$$\|f'_i - \widetilde f_{i}\|_{\infty} \leq \delta.$$
This yields
\begin{align}
\|\widetilde f-f^*\|^{2}_{n}&=\frac{1}{n}\sum_{i=1}^n\big(\widetilde f_i(\bX)-f^*_i(\bX)\big)^2\nonumber\\
&=\frac{1}{n}\sum_{i=1}^n\big(\widetilde f_i(\bX)-f'_i(\bX)+f'_i(\bX)-f^*_i(\bX)\big)^2\nonumber\\
&\leq\| f'-f^*\|^{2}_{n}+6F\delta.\label{cover-delta-b}
\end{align}
Let $\cC_2$ denote the event that
$$\pi n\|f'-f^*\|_n^2 >3600F^2\log\mathcal{N}_\delta.$$ By \eqref{cover-delta-b}, event $\mathcal{C}_1$ implies $\mathcal{C}_2$. Combining this with \eqref{cover-delta-b}, we deduce
\begin{align}
\bE_{\bs\varepsilon,\bs\omega}\cro{\|\widetilde f-f^*\|^{2}_{n}\1(\mathcal{C}_1)}&\leq \bE_{\bs\varepsilon,\bs\omega}\cro{\| f'-f^*\|^{2}_{n}\1(\mathcal{C}_1)}+6F\delta\nonumber\\
&\leq\bE_{\bs\varepsilon,\bs\omega}\cro{\| f'-f^*\|^{2}_{n}\1(\mathcal{C}_2)}+6F\delta.\label{step-3-transfer-delta}
\end{align}
Next, we bound $\bE_{\bs\varepsilon,\bs\omega}\cro{\| f'-f^*\|^{2}_{n}\1(\mathcal{C}_2)}$. Defining $\mathcal{E}$ as the event that
$$\|f'-f^*\|^{2}_{n,{\bs\omega}}\geq \frac{\pi\|f'-f^*\|^{2}_{n}}{2},$$ we can rewrite
\begin{align*}
\bE_{\bs\varepsilon,\bs\omega}\left [\|f'-f^*\|^{2}_{n}\1(\mathcal{C}_2)\right]
&=\bE_{\bs\varepsilon,\bs\omega}\left [\|f'-f^*\|^{2}_{n}\1(\cE\cap\mathcal{C}_2)\right]+\bE_{\bs\varepsilon,\bs\omega}\left [\|f'-f^*\|^{2}_{n}\1(\cE^{c}\cap\mathcal{C}_2)\right].
\end{align*}
Applying Lemma \ref{lm:isometry} with $\cG=\cF_{\delta}$, we know that $$\P(\cE^c\cap
\cC_2)\leq\frac{2}{\mathcal{N}_{\delta}},$$ which further implies that
\begin{align}
\bE_{\bs\varepsilon,\bs\omega}\left [\|f'-f^*\|^{2}_{n}\1(\mathcal{C}_2)\right]
&\leq \dfrac{2}{\pi}\,\bE_{\bs\varepsilon,\bs\omega}\cro{\|f'-f^*\|^{2}_{n,\bs\omega}} + \dfrac{8F^{2}}{\mathcal{N}_{\delta}}.\label{step3-b}  
\end{align}
Combining \eqref{step-3-transfer-delta} and \eqref{step3-b}, we obtain
\begin{equation}\label{step3-c-1}
\bE_{\bs\varepsilon,\bs\omega}\cro{\|\widetilde f-f^*\|^{2}_{n}\1(\mathcal{C}_1)}\leq\dfrac{2}{\pi}\,\bE_{\bs\varepsilon,\bs\omega}\cro{\|f'-f^*\|^{2}_{n,\bs\omega}} + \dfrac{8F^{2}}{\mathcal{N}_{\delta}}+6F\delta. \end{equation}
Substituting equations \eqref{step3-c-1-comple} and \eqref{step3-c-1} into \eqref{n-risk-b} gives
\begin{equation}\label{step-final-b}
\bE_{\bs\varepsilon,\bs\omega}\cro{\|\widetilde f-f^*\|^{2}_{n}}\leq\dfrac{2}{\pi}\bE_{\bs\varepsilon,\bs\omega}\cro{\|f'-f^*\|^{2}_{n,\bs\omega}}+\dfrac{8F^{2}}{\mathcal{N}_{\delta}}+\frac{3600F^2\log\mathcal{N}_\delta}{\pi n}+12F\delta.
\end{equation}
Moreover, observe that with $F\geq1,$
\begin{align}\label{delta-conne}
&\Delta_{n}^{\cF}\left(f',f^*\big|\bX\right)-\Delta_{n}^{\cF}\left(\widetilde f,f^*\big|\bX\right)\nonumber\\
&=\E_{\bs\varepsilon,\bs\omega}\cro{\frac{1}{n}\sum_{i=1}^n\omega_i\Big(\big(Y_i-f'_i(\bX)\big)^2-\big(Y_i-\widetilde f_i(\bX)\big)^2\Big)}\nonumber\\
&=\E_{\bs\varepsilon,\bs\omega}\cro{\frac{1}{n}\sum_{i=1}^n\omega_i\Big(2Y_i\big(\widetilde f_i(\bX)-f'_i(\bX)\big)+\big(f'_i(\bX)\big)^2-\big(\widetilde f_i(\bX)\big)^2\Big)}\nonumber\\
&=\E_{\bs\varepsilon,\bs\omega}\cro{\frac{1}{n}\sum_{i=1}^n\omega_i\Big(2\big(f_i^*(\bX)+\varepsilon_i\big)\big(\widetilde f_i(\bX)-f'_i(\bX)\big)+\big(f'_i(\bX)\big)^2-\big(\widetilde f_i(\bX)\big)^2\Big)}\nonumber\\
&\leq4\pi F\delta+2\delta\E_{\bs\varepsilon,\bs\omega}\cro{\frac{1}{n}\sum_{i=1}^n\omega_i|\varepsilon_i|}\nonumber\\
&\leq6\pi F\delta.
\end{align}
By applying the result from Step (II) with $\widetilde f=f'$ and using \eqref{delta-conne}, we obtain from \eqref{step-final-b} that
\begin{align*}
&\bE_{\bs\varepsilon,\bs\omega}\cro{\|\widetilde f-f^*\|^{2}_{n}}\\
&\leq2(1+\varepsilon)\cro{\inf_{f\in \mathcal{F}}\|f-f^*\|_n^2+ \frac{1+\varepsilon}{\varepsilon}\frac{4\log \mathcal{N}_{\delta}}{n\pi} + \frac{6\delta}{\sqrt{\pi}}+\frac{\Delta_{n}^{\cF}(\widetilde f,f^*\big|\bX)}{\pi}+6F\delta}\\
&\quad+\dfrac{8F^{2}}{\mathcal{N}_{\delta}}+\frac{3600F^2\log\mathcal{N}_\delta}{\pi n}+12F\delta\\
&\leq2(1+\varepsilon)\cro{\inf_{f\in \mathcal{F}}\|f-f^*\|_n^2+ \frac{1+\varepsilon}{\varepsilon}\frac{1804F^2\log \mathcal{N}_{\delta}}{n\pi} + \dfrac{4F^{2}}{\mathcal{N}_{\delta}}+\frac{6\delta}{\sqrt{\pi}}+12F\delta+\frac{\Delta_{n}^{\cF}(\widetilde f,f^*\big|\bX)}{\pi}}.
\end{align*}

\textbf{Step (IV): Relating the prediction error to the population risk.} 
This step yields the following two-sided bound for the prediction error of any (random) estimator $\widetilde f$. More precisely, we show that for any $\varepsilon \in (0,1]$,
\begin{align*}
(1-&\varepsilon)\E_{{\bs{\varepsilon}},{\bs\omega},\bX}\cro{\|\widetilde{f}-f^*\|^2_{n}}-\frac{20F^2m\log \mathcal{N}_{\delta}}{n\varepsilon}-\frac{15m^2F^2(\log\mathcal{N}_\delta)^{3/4}}{n^{3/4}}-16\delta F\leq\mathcal{R}\left(\widetilde{f}, f^*\right)\\  
&\leq (1+\varepsilon) \left(\E_{{\bs{\varepsilon}},{\bs\omega},\bX}\cro{\|\widetilde{f}-f^*\|^2_{n}}+\frac{15(1+\varepsilon)m F^{2}}{\varepsilon} \frac{\log \mathcal{N}_{\delta}}{n}+ \frac{15m^2 F^{2}(\log \mathcal{N}_{\delta})^{3/4}}{n^{3/4}}+12\delta F\right).
\end{align*}
Recall that all results from the previous three steps are established conditionally on the fixed design points $X_1,\ldots,X_n$.
		
\textit{Step 4.1:} 
Generate random vectors ${X}_1^{\prime},\ldots,{X}_n^{\prime}$ that have the same joint distribution as $X_1,\ldots,X_n$ and are independent of the original sample. We denote the resulting feature matrix by $\mathbf{X}'$, with rows $(X'_1)^\top,\ldots, (X'_n)^\top$. 

For any $f\in\cF\cup\cF_{\delta}$ and all $i=1,\ldots,n$, define 
$$s_{i}^{(f)}(\bX)=\big(f_i(\mathbf{X})-f^{*}_i(\mathbf{X})\big)^{2}\in \cro{0,4F^2},$$
(since $\|f_i\|_{\infty},\|f_i^*\|_{\infty}\leq F$) and
$$Z_i^{(f)}=\E_{\bX'}[s_{i}^{(f)}(\bX')]-s_{i}^{(f)}(\bX)\in \cro{-4F^2,4F^2},\quad D_f=\frac{1}{n}\sum_{i=1}^nZ_i^{(f)}.$$ For an estimator $\widetilde f$, recall that the prediction error was defined as $$\mathcal{R}\left(\widetilde f,f^*\right)=\mathbb{E}_{{\bs{\varepsilon}}, {\bs{\omega}},\bX, \bX'}\cro{\frac{1}{n}\sum_{i=1}^{n} s_{i}^{(\widetilde f\hspace{2pt})}(\bX')}.$$ Define the deviation term
$$D=\E_{{\bs{\varepsilon}}, {\bs{\omega}},\bX}[D_{\widetilde f}\hspace{2pt}]=\mathcal R(\widetilde{f}, f^*)-\E_{{\bs{\varepsilon}},{\bs\omega},\bX}[\|\widetilde{f}-f^*\|^2_{n}].$$

\textit{Step 4.2: Bounding $|D|.$} We claim that for any given $f$, under Assumption~\ref{ass-m}, there exists an integer $r\leq m(m-1)+1$ and a (disjoint) partition $\cP_1\cup \cdots\cup \cP_r =\{1,\ldots,n\}$ such that $i,j\in\cP_\ell$ if $s_i^{(f)}(\bX)$ and $s_j^{(f)}(\bX)$ depend on different rows of $\bX,$ that is, they do not share a single covariate vector. 

To prove the claim, we build a new graph with vertices $1,\ldots,n,$ where vertex $i$ and $j$ are connected by an edge if and only if $s_i^{(f)}(\bX)$ and $s_j^{(f)}(\bX)$ depend on at least one shared covariate vector. 

Every vertex in this new graph has edge degree bounded by $\leq m(m-1).$ To see this, we consider node $i$ corresponding to $s_i^{(f)}(\bX)$. By condition (i) of Assumption~\ref{ass-m}, $s_i^{(f)}(\mathbf{X})$ can depend on at most $m$ covariate vectors in $\mathbf{X}$. According to condition (ii), each of those $m$ covariate vectors can itself depend on at most $m-1$ of $s_j^{(f)}(\mathbf{X})$ with $j \neq i$. Therefore, the degree of any node $i$ in the new graph is at most $m(m-1)$. Since node $i$ was chosen arbitrarily, this bound holds for every node.

We now examine a vertex coloring of the new graph. This is an assignment of a color to each vertex with the constraint that neighboring vertices must receive different colors. Using a greedy sequential coloring scheme guarantees that at most $r \leq \Delta_{\max} + 1$ colors are sufficient, where $\Delta_{\max}$ denotes the maximum vertex degree. In our case, this yields the bound $r \leq m(m-1) + 1$. Grouping all vertices of the same color into one set yields the disjoint partition $\mathcal{P}_1 \cup \cdots \cup \mathcal{P}_r$ of $\{1,\ldots,n\}$ with the properties stated above, thereby proving the claim.

In particular, the claim implies that $D_f$ can be rewritten as
\begin{equation}\label{sepa-Df}
D_f=\frac{1}{n}\sum_{\ell=1}^{r}V_{\ell}^{(f)}\quad\mbox{where}\quad V_{\ell}^{(f)}=\sum_{i\in\cP_\ell}Z_i^{(f)},
\end{equation}
and each $V_{\ell}^{(f)}$ is a sum of independent random variables.

For each block $\ell=1, \dots, r$, the random variables $\{Z_i^{(f)} : i \in \mathcal{P}_\ell\}$ satisfy $|Z_i^{(f)}| \leq 4F^2$ almost surely. Applying Lemma~\ref{lem:bounded-mgf} to the sum $V_{\ell}^{(f)} = \sum_{i \in \mathcal{P}_\ell} Z_i^{(f)}$ yields, for any $\lambda$ such that $0 \leq 4F^2 \lambda r < 3$,
\begin{equation}\label{each-block-b}
\E\cro{\exp\left(\lambda rV_{\ell}^{(f)}\right)}\leq\exp\cro{\frac{\lambda^2 r^2}{2(1-4F^2\lambda r/3)}\sum_{i\in\cP_{\ell}}\E\cro{(Z_i^{(f)})^2}}.    
\end{equation}
Substituting the blockwise bound \eqref{each-block-b} into Lemma~\ref{general-holder}, we bound, for any $\lambda$ satisfying $0\leq4F^2\lambda r<3$,
\begin{align}
&\E\cro{\exp\left(\lambda\sum_{\ell=1}^rV_{\ell}^{(f)}\right)}\nonumber\\
&\leq \prod_{\ell=1}^r \left( \E\cro{\exp\left(\lambda rV_{\ell}^{(f)}\right)} \right)^{\frac{1}{r}}\nonumber\\
&\leq\prod_{\ell=1}^r \left( \exp\cro{\frac{\lambda^2 r}{2(1-4F^2\lambda r/3)}\sum_{i\in\cP_{\ell}}\E\cro{(Z_i^{(f)})^2}} \right)\nonumber\\
&\leq \exp\cro{\frac{\lambda^2 r}{2(1-4F^2\lambda r/3)}\sum_{i=1}^n\E\cro{(Z_i^{(f)})^2}}.\label{exp-bb}
\end{align}
Observe that 
\begin{equation}\label{b-n-exp-s}
\sum_{i=1}^n\E\cro{(Z_i^{(f)})^2}\leq\sum_{i=1}^n\E\cro{\left(s_{i}^{(f)}(\bX)\right)^2}\leq4F^2n\cR\left(f, f^*\right).   \end{equation}
Plugging \eqref{b-n-exp-s} into \eqref{exp-bb} and using Markov's Inequality, we deduce that for any $t > 0$ and any $\lambda$ satisfying $0\leq\lambda < 3/(4rF^2)$,
\begin{equation}\label{exp-b-lambda}
\P\left(nD_f\geq t\right) \leq e^{-\lambda t}\hspace{2pt}\E\cro{\exp\left(\lambda\sum_{\ell=1}^rV_{\ell}^{(f)}\right)}\leq \exp\cro{-\lambda t + \frac{4F^2\lambda^2rn\cR\left(f, f^*\right)}{2(1 - 4rF^2\lambda/3)}}.    
\end{equation}
Taking $$\lambda=\frac{t}{4F^2rn\cR\left(f, f^*\right) + 4rF^2t/3}$$ in the right-hand side of \eqref{exp-b-lambda} yields the one-sided tail bound
\begin{equation}\label{n-times-b}
\P\left(nD_f\geq t\right)\leq\exp\cro{ -\frac{t^2}{8F^2rn\cR\left(f, f^*\right)+8F^2rt/3}}.
\end{equation}
Since the same argument applies to $-D_f$, substituting $t\to nt$ yields the two-sided inequality
\begin{equation}\label{final-exp-b}
\P\left(|D_f| \geq t\right) \leq 2\exp\cro{-\frac{n t^2}{8F^2 r \cR(f, f^*) + 8F^2 r t/3}}.
\end{equation}
Using $\sqrt{a+b}\leq\sqrt{a}+\sqrt{b}$, for all $a,b\geq0$, we deduce from \eqref{final-exp-b} that for all $u>0$, 
\begin{equation}\label{two-sided-s-b}
\mathbb{P}\left( |D_f|\geq\sqrt{\frac{8F^2ru\cR(f, f^*)}{n}} + \frac{8F^2ru}{3n} \right) \leq 2e^{-u}.    
\end{equation}
Taking a union bound over $\cF_{\delta}$ and using the fact that $f' \in \mathcal{F}_{\delta}$, we obtain from \eqref{two-sided-s-b} that for all $u>0$,
$$\P\left(|D_{f'}|\geq\sqrt{\frac{8F^2 ru\cR(f', f^*)}{n}} + \frac{8F^2ru}{3n}\right)\leq2\cN_{\delta}e^{-u}.$$
Set $$B_1=\sqrt{\frac{8rF^2\cR\left(f',f^*\right)}{n}}\quad\mbox{and}\quad B_2=\frac{8F^2r}{3n},$$ and define
$$G(u)=B_1u^{1/2}+B_2u.$$ Taking $u_0=\log(2\cN_{\delta})$, it then follows from the integration that
\begin{align}
\E\cro{|D_{f'}|}&=\int_{0}^{\infty}\P\left\{|D_{f'}|>t\right\}dt\nonumber\\  
&=\int_{0}^{\infty}\P\left\{|D_{f'}|>G(u)\right\}G'(u) \, du\nonumber\\
&\leq\int_{0}^{u_0}G'(u) \, du+\int_{u_0}^{\infty}2\cN_{\delta}e^{-u}G'(u) \, du\nonumber\\
&\leq G(u_0)+2\cN_{\delta}\int_{u_0}^{\infty}e^{-u}\left(\frac{B_1}{2\sqrt{u}}+B_2\right)du\nonumber\\
&\leq1.7B_1 \sqrt{\log\mathcal{N}_\delta} +2.7B_2 \log\mathcal{N}_\delta,\label{exp-df'-b}
\end{align}
where we use $1\leq\log\cN_{\delta}\leq n$ in the last inequality. Using the fact that $\|f'_i-\widetilde f_i\|_{\infty}\leq\delta$ for all $i$, we obtain from \eqref{exp-df'-b} that
\begin{align}
|D|&\leq\E\cro{|D_{\widetilde f}\hspace{2pt}|}\nonumber\\
&\leq\E\cro{|D_{f'}|}+8\delta F\nonumber\\
&\leq1.7B_1 \sqrt{\log\mathcal{N}_\delta} +2.7B_2 \log\mathcal{N}_\delta+8\delta F\nonumber\\
&\leq\frac{7.2rF^2\log\mathcal{N}_\delta}{n}+4.81F\sqrt{\frac{r\cR(f',f^*)\log\mathcal{N}_\delta}{n}}+8\delta F\nonumber\\
&\leq\frac{7.2rF^2\log\mathcal{N}_\delta}{n}+4.81F\sqrt{\frac{r\cro{\cR(\widetilde f,f^*)+4\delta F}\log\mathcal{N}_\delta}{n}}+8\delta F.\label{bound-absolute-d}
\end{align}
\textit{Step 4.3: Applying a quadratic bound.} Let $a,b,c,d$ be positive real numbers satisfying $|a-b|\leq2c\sqrt{a}+d.$ Then, for any $\varepsilon\in(0,1],$ the following inequality holds
\begin{equation}\label{aux-inequ}
(1-\varepsilon) b - d - \frac{c^2}{\varepsilon} \leq a \leq (1+\varepsilon) (b + d) + \frac{(1+\varepsilon)^2}{\varepsilon} c^2.    
\end{equation}
Applying \eqref{aux-inequ} with $$a =\mathcal{R}\left(\widetilde{f},f^*\right)+4\delta F, \  b = \E_{{\bs{\varepsilon}},{\bs\omega},\bX}\cro{\|\widetilde{f}-f^*\|^2_{n}},\  c = \frac{4.81F\sqrt{r\log \mathcal{N}_{\delta}}}{2\sqrt{n}}, \  d = \frac{7.2rF^2\log\mathcal{N}_\delta}{n}+12\delta F,$$ using the fact that $r\leq m^2$, we derive from \eqref{bound-absolute-d} that
\begin{align*}
(1-\varepsilon)\E_{{\bs{\varepsilon}},{\bs\omega},\bX}\cro{\|\widetilde{f}-f^*\|^2_{n}}&-\frac{13m^2F^2\log \mathcal{N}_{\delta}}{n\varepsilon}-16\delta F\leq\mathcal{R}\left(\widetilde{f}, f^*\right)\\ &\leq (1+\varepsilon) \left(\E_{{\bs{\varepsilon}},{\bs\omega},\bX}\cro{\|\widetilde{f}-f^*\|^2_{n}}+\frac{10(1+\varepsilon)m^2 F^{2}}{\varepsilon} \frac{\log \mathcal{N}_{\delta}}{n}+12\delta F\right),
\end{align*}
where the upper bound employs the inequality $2 \leq(1+\varepsilon) / \varepsilon$.

\textbf{Step (V): Lower bound for $\E_{{\bs{\varepsilon}},{\bs\omega}}[\|\widetilde{f}-f^*\|^2_{n}]$.} In this step, we will show that for any $\varepsilon\in(0,1]$, $$\E_{{\bs{\varepsilon}},{\bs\omega}}\cro{\|\widetilde{f}-f^*\|^2_{n}}\geq (1-\varepsilon) \left( \Delta_{n}^{\cF}\left(\widetilde f,f^*\big|\bX\right) - \frac{4\log \mathcal{N}_{\delta}}{n \varepsilon} - 12\delta \sqrt{\pi} \right).$$

Let $\bar{f}=(\bar{f}_1,\ldots,\bar{f}_n)^{\top}$ be any global empirical risk minimizer over $\cF$; that is,
$$\bar{f}\in\argmin_{f \in \mathcal{F}}\sum_{i=1}^{n}\omega_i\big(Y_i-{f}_i(\bX)\big)^2.$$
Using \eqref{eq:trivial_lb} and the result from Step (I), we can deduce that
\begin{align}
&\E_{{\bs{\varepsilon}},{\bs\omega}}\cro{\|\widetilde{f}-f^*\|^2_{n,{\bs\omega}}} - \E_{{\bs{\varepsilon}},{\bs\omega}}\cro{\|\bar{f}-f^*\|^2_{n,{\bs\omega}}}\nonumber\\
&\geq\Delta_{n}^{\cF}\left(\widetilde f,f^*\big|\bX\right)-4\sqrt{\frac{ \left(\E_{{\bs{\varepsilon}},{\bs\omega}}\cro{\|\widetilde{f}-f^*\|^2_{n,{\bs\omega}}}+\E_{{\bs{\varepsilon}},{\bs\omega}}\cro{\|\bar{f}-f^*\|^2_{n,{\bs\omega}}}\right) \log \mathcal{N}_{\delta}}{n}} - 12\sqrt{\pi}\delta.\label{step-5-ineq} 
\end{align}
For any $\varepsilon\in(0,1)$ and $a,b\geq0$, using the inequality
\begin{equation}\label{pre-tool}
2 \sqrt{ab} \leq \frac{\varepsilon}{1-\varepsilon} a + \frac{1-\varepsilon}{\varepsilon} b,    
\end{equation}
we bound
\begin{equation}\label{s5-ine2}
4\sqrt{\frac{\E_{{\bs{\varepsilon}},{\bs\omega}}\cro{\|\widetilde{f}-f^*\|^2_{n,{\bs\omega}}} \log \mathcal{N}_{\delta}}{n}} \leq \frac{\varepsilon}{1-\varepsilon} \E_{{\bs{\varepsilon}},{\bs\omega}}\cro{\|\widetilde{f}-f^*\|^2_{n,{\bs\omega}}}+\frac{4(1-\varepsilon) \log \mathcal{N}_{\delta}}{n \varepsilon}. \end{equation}
Moreover, taking $\varepsilon=1/2$ in \eqref{pre-tool} yields the bound
\begin{equation}\label{s5-ine3}
4\sqrt{\frac{\E_{{\bs{\varepsilon}},{\bs\omega}}\cro{\|\bar{f}-f^*\|^2_{n,{\bs\omega}}} \log \mathcal{N}_{\delta}}{n}}\leq\E_{{\bs{\varepsilon}},{\bs\omega}}\cro{\|\bar{f}-f^*\|^2_{n,{\bs\omega}}} + \frac{4\log\mathcal{N}_{\delta}}{n}. \end{equation}
By definition, $$\E_{{\bs{\varepsilon}},{\bs\omega}}[\|\widetilde{f}-f^*\|^2_{n}]\geq\E_{{\bs{\varepsilon}},{\bs\omega}}[\|\widetilde{f}-f^*\|^2_{n,{\bs\omega}}].$$ Substituting \eqref{s5-ine2} and \eqref{s5-ine3} into \eqref{step-5-ineq}, we obtain
\begin{equation}\label{omega-n-risk-bound}
\E_{{\bs{\varepsilon}},{\bs\omega}}[\|\widetilde{f}-f^*\|^2_{n}]\geq 
\E_{{\bs{\varepsilon}},{\bs\omega}}\cro{\|\widetilde{f}-f^*\|^2_{n,{\bs\omega}}}\geq(1-\varepsilon) \left(\Delta_{n}^{\cF}\left(\widetilde f,f^*\big|\bX\right)- \frac{4\log\mathcal{N}_{\delta}}{n \varepsilon}-12\sqrt{\pi} \delta\right).    
\end{equation}
When $\varepsilon=1$, the lower bound holds trivially.

In conclusion, the lower bound follows by combining Step (IV) with the expectation over $\bX$ of Step (V)'s result, and the upper bound follows by combining Step (IV) with the expectation over $\bX$ of Step (III)'s result.
\end{proof}
\subsection{Proof of Corollary~\ref{kappa-b}}
\begin{proof}
Consider the feedforward part of the network class that is denoted by $\cF(L_2,\bs{p},s,\infty)$. Using the identity (19) in \cite{Schmidt-Hieber} to remove inactive nodes, we deduce that
\begin{equation}\label{refor-deep-relu}
\cF(L_2,\bs{p},s,\infty)=\mathcal{F}\big(L_2, (p_0, p_1 \wedge s, p_2 \wedge s, \ldots, p_{L_2} \wedge s, p_{L_2+1}), s,\infty\big).
\end{equation}
Using this identity and choosing $\delta = 1/n$ in Proposition~\ref{cover-whole} yields a metric entropy bound for the whole class
\begin{align}
&\log\mathcal{N}\big(1/n,\mathcal{F}(\bT, L_1, L_2, {\bs p},s,F),\|\cdot\|_{\infty}\big)\nonumber\\
&\leq\log\mathcal{N}\big(1/n,\mathcal{F}(\bT, L_1, L_2, {\bs p}, s,\infty),\|\cdot\|_{\infty}\big)\nonumber\\
&=\log\mathcal{N}\big(1/n,\mathcal{F}(\bT, L_1, L_2, (p_0, p_1 \wedge s, p_2 \wedge s, \ldots, p_{L_2} \wedge s, p_{L_2+1}), s,\infty\big),\|\cdot\|_{\infty})\nonumber\\
&\leq(d^2L_1+L_1+s+1)\log\left[2^{2L_2+5}nL_1(L_1+L_2+2)(\|\bT\|_{1,\infty}\vee1)^{L_1}d^{L_1+2}s^{2L_2}\right]
\nonumber\\
&\leq18 (d^2L_1+s)\log\left[nL_1(L_1+L_2+2) (\|\bT\|_{1,\infty}\vee d)^{L_1+1}s^{L_2}\right],\label{cl-bound}   
\end{align}
where the last inequality is due to $s\geq2$ and $L_2\geq1$. Under the condition $$\cN\big(1/n, \cF(\bT, L_1, L_2, \bs{p}, s,F),\|\cdot\|_{\infty}\big) \geq Cn,$$ with some numerical constant $C>0$, and setting $\delta=1/n$, it follows from $F\geq1$ and $\pi\in(0,1]$ that 
\begin{equation*}
\frac{24F\delta}{\sqrt{\pi}}+\frac{4F^2}{\cN_{\delta}}\leq \frac{24F}{n\pi}+\frac{4F^2}{Cn}\leq\left(24+\frac{4}{C}\right)\frac{F^2}{n\pi}.
\end{equation*}
Together with \eqref{cl-bound} and Theorem~\ref{main} (and the conditions $m,F \geq 1$ and $0 < \pi \leq 1$), this proves the assertion.
\end{proof}

\subsection{Proof of Theorem~\ref{converge-rate}}
\begin{proof}
The proof is based on applying Corollary~\ref{kappa-b}. We first verify that 
\begin{equation}\label{cover-lower-condition}
\cN\Big(\hspace{2pt}\frac 1n,\cF(\bS_{\mathbf{A}},L_1,L_2,{\bs{p}}_n,s_n,F),\|\cdot\|_{\infty}\Big)\gtrsim n.    
\end{equation}
Observe that the function class includes constant functions; that is, for any fixed $v\in[0,1]$, there exists a function in $\cF(\bS_{\mathbf{A}},L_1,L_2,{\bs{p}}_n,s_n,F)$ that outputs $v$ at all nodes, given any input. This can be realized by setting all parameters except the bias in the final layer of the ReLU network to zero and taking ${\bs b}_{L_2} = v$. In particular, let $\Gamma= \lceil n/3 \rceil$ and define the set of functions $\cM=\{f^{(0)}, \dots, f^{(\Gamma-1)}\}$ by  
\[
f^{(j)} \equiv \frac{3j}{n}, \quad \text{for all } j = 0,\dots,\Gamma-1.
\]
With $F\geq1$ and $s_n\geq2$, we have $\cM \subseteq \cF(\bS_{\mathbf{A}},L_1,L_2,{\bs{p}}_n,s_n,F)$ and the functions satisfy
$$\left\|f^{(i)}-f^{(j)}\right\|_{\infty}\geq\frac{3}{n}>\frac{2}{n},\quad\mbox{for\ all\ }i\not=j.$$ Then any sup-norm ball of radius $1/n$ can contain at most one function from $\cM$. We obtain
$$\cN\Big(\frac 1n,\cF(\bS_{\mathbf{A}},L_1,L_2,{\bs{p}}_n,s_n,F),\|\cdot\|_{\infty}\Big)\geq\frac{n}{3},$$ which verifies \eqref{cover-lower-condition}.

Taking $$N_i=\left\lceil n^{\frac{t_i}{2\alpha_i^*+t_i}}\right\rceil,\quad N=\max_{i=0,\ldots,q}N_i$$ and the network parameters $L_1, L_2,\bs{p}_n,s_n, F$ satisfying Theorem~\ref{converge-rate}, for sufficiently large $n$, Lemma~\ref{overall-approx} yields
\begin{equation}\label{appro-error-rate}
\inf_{f\in\cF(\bS_{\mathbf{A}},L_1,L_2,{\bs{p}}_n,s_n,F)}\|f-f^*\|_{\infty}\leq C_{q,\beta,{\bs{d}},{\bs{t}},{\bs{\alpha}},M,F}\max_{i=0,\ldots,q} n^{-\frac{\alpha_i^*}{2\alpha_i^*+t_i}}.     
\end{equation}
Additionally, under the given conditions, for all sufficiently large $n$, we can bound
\begin{align}
&(d^2 L_1 + s)\cro{\log\big( n L_1(L_1 + L_2)\big)+(L_1+1)\log\big(\|\bS_{\mathbf{A}}\|_{1,\infty}\vee d\big)+L_2\log s}\nonumber\\ 
&\leq C_{d,L_1,\|\bS_{\mathbf{A}}\|_{1,\infty}}sL_2\log(nsL_2),\label{variance-rate}
\end{align}
where $C_{d,L_1,\|\bS_{\mathbf{A}}\|_{1,\infty}}$ is a positive constant depending only on $d$, $L_1$, and $\|\bS_{\mathbf{A}}\|_{1,\infty}$.

Substituting \eqref{appro-error-rate} and \eqref{variance-rate} into Corollary~\ref{kappa-b} with $\varepsilon=1$ gives that for sufficiently large $n$,
\begin{align*}
\mathcal{R}(\widehat{f}, f^*)&\leq C_{q,\beta,{\bs{d}},{\bs{t}},{\bs{\alpha}},M,F}\max_{i=0,\ldots,q}n^{-\frac{2\alpha_i^*}{2\alpha_i^*+t_i}}+C_{d,L_1,\|\bS_{\mathbf{A}}\|_{1,\infty}}\frac{m^2F^2}{\pi}\frac{sL_2\log(nsL_2)}{n}\\
&\leq C_{q,\beta,{\bs{d}},{\bs{t}},{\bs{\alpha}},M,F}\max_{i=0,\ldots,q}n^{-\frac{2\alpha_i^*}{2\alpha_i^*+t_i}}+C_{q,\beta,{\bs d},{\bs t},{\bs{\alpha}},M,L_1,\|\bS_{\mathbf{A}}\|_{1,\infty}}\frac{m^2F^2}{\pi}\frac{N}{n}\log^3 n\\
&\leq C_{q,\beta,{\bs d},{\bs{t}},{\bs{\alpha}},M,F,L_1,\|\bS_{\mathbf{A}}\|_{1,\infty}}\frac{m^2\log^3 n}{\pi}\max_{i=0,\ldots,q}n^{-\frac{2\alpha_i^*}{2\alpha_i^*+t_i}},   
\end{align*}
which completes the proof.
\end{proof}

\section{Proof of the covering number bound}\label{proof-covering}
To prove Proposition~\ref{cover-whole}, we first show that for any two GCNs with matrix parameters and reweighting coefficients differing by at most $\varepsilon$, their outputs, for any given input, differ by at most a value proportional to $\varepsilon$.

Recall that for a matrix $\mathbf{M}=(M_{i,j})$, its row-sum norm is given by $$\|\mathbf{M}\|_{1,\infty}=\max_i \sum_j |M_{i,j}|.$$

\begin{lemma}\label{cover-gcn}
Let $\mathcal{G}{(L_1,\bT)}$ denote the class of functions defined in \eqref{gcn-l-t}. If $g, h \in \mathcal{G}{(L_1,\bT)}$ are two GCNs with corresponding weight matrices $\bW_{\ell}^g, \bW_{\ell}^h$ and reweighting coefficients $\gamma_{\ell}^g,\gamma_{\ell}^h$ satisfying
$$\left\|\bW_{\ell}^g - \bW_{\ell}^h\right\|_{\infty} \leq \varepsilon,\quad |\gamma_{\ell}^g-\gamma_{\ell}^h|\leq \varepsilon,\quad\mbox{for all\ }\ell \in\{1, \ldots, L_1\},$$ 
then for any ${\bs{x}}\in[0,1]^{n\times d}$,
$$\|g({\bs{x}})-h({\bs{x}})\|_{\infty}\leq (L_1^2+L_1)\big(\|\bT\|_{1,\infty}\vee1 \big)^{L_1}d^{L_1}\varepsilon.$$
\end{lemma}	
\begin{proof}
We denote by $H_{\bT,g}^{(\ell)}$ and $H_{\bT,h}^{(\ell)}$ the functions corresponding to the $\ell$-th layer of networks $g$ and $h$, respectively. First, we show by induction that for each $\ell=1,\ldots,L_1$,
\begin{equation}\label{induc-b}
\left\|H_{\bT,g}^{(\ell)}({\bs{x}})-H_{\bT,h}^{(\ell)}({\bs{x}})\right\|_{\infty}\leq\ell \|\bT\|_{1,\infty}^{\ell}d^{\ell}\varepsilon.
\end{equation}
For $\ell=1$, we can deduce from \eqref{gcn-def} that 
\begin{align*}
\left\|H_{\bT,g}^{(1)}({\bs{x}})-H_{\bT,h}^{(1)}({\bs{x}})\right\|_{\infty}&=\left\|\bT{\bs{x}}\big(\gW_1^g-\gW_1^h\big)\right\|_{\infty}\leq\|\bT\|_{1,\infty}d\varepsilon.
\end{align*}
Assuming the claim holds for $\ell-1\in\{1,\ldots,L_1-1\}$, we prove it now for $\ell$. To this end,
\begin{align*}
&\left\|H_{\bT,g}^{(\ell)}({\bs{x}})-H_{\bT,h}^{(\ell)}({\bs{x}})\right\|_{\infty}\\
&=\left\|\bT H_{\bT,g}^{(\ell-1)}({\bs{x}})\bW_{\ell}^g-\bT H_{\bT,h}^{(\ell-1)}({\bs{x}}) \bW_{\ell}^h\right\|_{\infty}\\
&\leq\left\|\bT\big(H_{\bT,g}^{(\ell-1)}({\bs{x}})-H_{\bT,h}^{(\ell-1)}({\bs{x}})\big) \bW_{\ell}^g\right\|_{\infty}+\left\|\bT H_{\bT,h}^{(\ell-1)}({\bs{x}})\big(\bW_{\ell}^g-\bW_{\ell}^h\big)\right\|_{\infty}\\
&\leq\|\bT\|_{1,\infty}\cro{(\ell-1)\|\bT\|_{1,\infty}^{\ell-1}d^{\ell-1}\varepsilon}d+\|\bT\|_{1,\infty}\big(\|\bT\|_{1,\infty}^{\ell-1}d^{\ell-1}\big)d\varepsilon\\
&\leq\ell\|\bT\|_{1,\infty}^{\ell}d^{\ell}\varepsilon.
\end{align*}
Hence, by definition \eqref{gcn-l-t} and the fact that  $|\gamma_{\ell}^g|,|\gamma_{\ell}^h|\leq1$, for $\ell=1\ldots,L_1$,
\begin{align*}
&\|g({\bs{x}})-h({\bs{x}})\|_{\infty}\\
&=\left\|\sum_{\ell=1}^{L_1}\gamma_{\ell}^gH_{\bT,g}^{(\ell)}({\bs{x}})-\sum_{\ell=1}^{L_1}\gamma_{\ell}^hH_{\bT,h}^{(\ell)}({\bs{x}})\right\|_{\infty}\\
&\leq\left\|\sum_{\ell=1}^{L_1}\gamma_{\ell}^gH_{\bT,g}^{(\ell)}({\bs{x}})-\sum_{\ell=1}^{L_1}\gamma_{\ell}^gH_{\bT,h}^{(\ell)}({\bs{x}})\right\|_{\infty}+\left\|\sum_{\ell=1}^{L_1}\gamma_{\ell}^gH_{\bT,h}^{(\ell)}({\bs{x}})-\sum_{\ell=1}^{L_1}\gamma_{\ell}^hH_{\bT,h}^{(\ell)}({\bs{x}})\right\|_{\infty}\\
&\leq\sum_{\ell=1}^{L_1}\left\|H_{\bT,g}^{(\ell)}({\bs{x}})-H_{\bT,h}^{(\ell)}({\bs{x}})\right\|_{\infty}+\left|\gamma_{\ell}^g-\gamma_{\ell}^h\right|\cdot\left\|\sum_{\ell=1}^{L_1}H_{\bT,h}^{(\ell)}({\bs{x}})\right\|_{\infty}\\
&\leq L_1^2\big(\|\bT\|_{1,\infty}\vee1 \big)^{L_1}d^{L_1}\varepsilon+L_1\big(\|\bT\|_{1,\infty}\vee1 \big)^{L_1}d^{L_1}\varepsilon\\
&\leq L_1(L_1+1)\big(\|\bT\|_{1,\infty}\vee1 \big)^{L_1}d^{L_1}\varepsilon.
\end{align*}
\end{proof}

\begin{proof}[Proof of Proposition~\ref{cover-whole}]
Let $g, h \in \mathcal{F}(\bT, L_1, L_2,{\bs{p}}, s,\infty)$ be two networks such that their weight matrices and bias vectors differ by at most \(\varepsilon\) in each entry. By definition, for each $j\in[n]$ we have
$$g_j=g_0\circ g_{1,j}\quad\mbox{and}\quad h_j=h_0\circ h_{1,j},$$ where $g_0,h_0\in\cF(L_2,{\bs{p}},s,\infty)$ and $g_1,h_1\in\cG(L_1,\bT)$. We first show that for any $j\in[n]$ and any ${\bs x}\in[0,1]^{n\times d}$,
$$\big|g_0\circ g_{1,j}({\bs x})-h_0\circ h_{1,j}({\bs x})\big|\leq L_1(L_1+L_2+2)\cro{\prod_{k=0}^{L_2+1}(p_{k}+1)}\big(\|\bT\|_{1,\infty}\vee1 \big)^{L_1}d^{L_1}\varepsilon.$$ 

For any ${\bs x}\in[0,1]^{n\times d}$, the triangle inequality yields
\begin{equation}\label{decompo-inequa}
|g_0\circ g_{1,j}({\bs x})-h_0\circ h_{1,j}({\bs x})|\leq|g_0\circ g_{1,j}({\bs x})-h_0\circ g_{1,j}({\bs x})|+|h_0\circ g_{1,j}({\bs x})-h_0\circ h_{1,j}({\bs x})|.    
\end{equation}
Let $g^i_{1,j}({\bs x})$ denote the $i$-th component of $g_{1,j}({\bs x})\in\R^{1\times d}$. Observe that by the definition of $g_{1,j}$, for all $i,$
\begin{equation}\label{bound-gij}
\left|g_{1,j}^i({\bs x})\right|\leq\sum_{\ell=1}^{L_1}\|\bT\|_{1,\infty}^{\ell}d^{\ell}.
\end{equation}
Since all parameters in the weight matrices and shift vectors of $\cF(L_2,{\bs{p}},s,\infty)$ lie within $[-1,1]$, it follows from the proof of Lemma~5 in \cite{Schmidt-Hieber} that the function $h_0$ is Lipschitz with a Lipschitz constant bounded by $\prod_{k=0}^{L_2}p_k$. Together with Lemma~\ref{cover-gcn}, we obtain
\begin{align}
|h_0\circ g_{1,j}({\bs x})-h_0\circ h_{1,j}({\bs x})|&\leq\left(\prod_{k=0}^{L_2}p_k\right)\left|g_{1,j}({\bs x})-h_{1,j}({\bs x})\right|_{\infty}\nonumber\\
&\leq\left(\prod_{k=0}^{L_2}p_k\right)\cro{L_1(L_1+1)\big(\|\bT\|_{1,\infty}\vee1\big)^{L_1}d^{L_1}\varepsilon}.\label{sum-p-2}
\end{align}
Applying the argument from the proof of Lemma 5 in \cite{Schmidt-Hieber} (specifically, the last step in the chain of inequalities on page 15) together with \eqref{bound-gij}, we obtain
\begin{align}
|g_0\circ g_{1,j}({\bs x})-h_0\circ g_{1,j}({\bs x})|\leq\varepsilon(L_2+1)\cro{\prod_{k=0}^{L_2+1}(p_{k}+1)}\left(\sum_{\ell=1}^{L_1}\|\bT\|_{1,\infty}^{\ell}d^{\ell}\right).\label{sum-p-1}
\end{align}
Thus, substituting \eqref{sum-p-2} and \eqref{sum-p-1} into \eqref{decompo-inequa}, it follows that
\begin{align*}
&\big|g_0\circ g_{1,j}({\bs x})-h_0\circ h_{1,j}({\bs x})\big|\\
&\leq\varepsilon(L_2+1)\cro{\prod_{k=0}^{L_2+1}(p_{k}+1)}\left(\sum_{\ell=1}^{L_1}\|\bT\|_{1,\infty}^{\ell}d^{\ell}\right)+ \left(\prod_{k=0}^{L_2}p_k\right)\cro{L_1(L_1+1)\big(\|\bT\|_{1,\infty}\vee1\big)^{L_1}d^{L_1}\varepsilon}\\
&\leq L_1(L_1+L_2+2)\cro{\prod_{k=0}^{L_2+1}(p_{k}+1)}\big(\|\bT\|_{1,\infty}\vee1\big)^{L_1}d^{L_1}\varepsilon.
\end{align*}
This implies that for any $\delta>0$, constructing a $\delta$-covering of $\cF(\bT, L_1, L_2, \bs{p}, s,\infty)$ only requires discretizing the network parameters with grid size $$\rho_{\delta}=\frac{\delta}{L_1(L_1+L_2+2)\cro{\prod_{k=0}^{L_2+1}(p_{k}+1)}\big(\|\bT\|_{1,\infty}\vee1\big)^{L_1}d^{L_1}}.$$ When $\delta\leq1$ and $L_1\geq1$, it follows that $\rho_{\delta}\leq1$. The GCN part has $(d^2+1)L_1$ parameters in total and all parameters take values in $[-1,1]$. For the class $\cF(L_2,{\bs p},s,\infty)$, the total number of parameters is bounded by $$\sum_{k=0}^{L_2}(p_{k}+1)p_{k+1}\leq\prod_{k=0}^{L_2+1}(p_{k}+1).$$ To pick $s$ non-zero parameters, there are at most $[\prod_{k=0}^{L_2+1}(p_{k}+1)]^s$ combinations. Therefore, we obtain that
\begin{align*}
&\mathcal{N}\big(\delta,\cF(\bT, L_1, L_2, \bs{p}, s,\infty\big),\|\cdot\|_{\infty})\\
&\leq\cro{\sum_{s^*\leq s}\left(\frac{2}{\rho_{\delta}}\cro{\prod_{k=0}^{L_2+1}(p_{k}+1)}\right)^{s^*}}\left(\frac{2}{\rho_{\delta}}\right)^{(d^2+1)L_1}\\
&\leq\left(\frac{2}{\rho_{\delta}}\cro{\prod_{k=0}^{L_2+1}(p_{k}+1)}\right)^{s+1}\left(\frac{2}{\rho_{\delta}}\right)^{(d^2+1)L_1}\\
&\leq\left(\frac{2L_1(L_1+L_2+2)(\|\bT\|_{1,\infty}\vee1)^{L_1}d^{L_1}}{\delta}\cro{\prod_{k=0}^{L_2+1}(p_{k}+1)}^2\right)^{(d^2+1)L_1+s+1}.    
\end{align*}
Taking logarithms yields the result.
\end{proof}

\section{Proofs for the approximation theory}\label{appro-proofs}
\subsection{Proof of Lemma~\ref{gcn-approx}}
\begin{proof}
It suffices to show that, under the given condition, for every $f\in\cF_0(\beta,k,\bT)$, one can find a function $g\in\cG(L_1,\bT)$ such that $f\equiv g$.
Recall that for any $g\in\mathcal{G}{(L_1,\bT)}$ and any input ${\bs x}\in\cro{0,1}^{n\times d}$, the output takes the form
\begin{equation*}
g({\bs x})=\sum_{\ell=1}^{L_1}\gamma_{\ell}H^{(\ell)}_{\bT,g}({\bs x}),
\end{equation*}
where for $\ell\in\{1,\ldots,L_1\}$,
\begin{equation*}
H^{(\ell)}_{\bT,g}({\bs x})=\bT H^{(\ell-1)}_{\bT,g}({\bs x})\bW_{\ell}^g,   
\end{equation*}
with $H^{(0)}_{\bT,g}({\bs x})={\bs x}$. Here, $\bW_{\ell}^g\in\R^{d\times d}$ are the parameter matrices associated with the function $g$, with each of its entries lying in $[-1,1]$. Under the given conditions, every target function $f\in\cF_0(\beta,k,\bT)$ admits the representation
$$f({\bs x})= \sum_{i=1}^k \theta_i \bT^i{\bs x},\quad\mbox{where}\quad|\theta_i|\leq\beta\leq1.$$
For $k\leq L_1$, the parameter choice $\bW_{\ell}^{g}=\mathbf{I}_d$, for $1\leq\ell\leq k$, and $\bW_{\ell}^{g}={\bf 0}$, for $\ell>k$ yields for $1\leq\ell\leq k$,
\begin{align*}
H^{(\ell)}_{\bT,g}({\bs x})&=\bT H^{(\ell-1)}_{\bT,g}({\bs x})\mathbf{I}_d\\
&=\bT\bT^{\ell-1}{\bs x}\mathbf{I}_d\\
&=\bT^{\ell}{\bs x} 
\end{align*}    
and 
$$H^{(\ell)}_{\bT,g}({\bs x})={\bf 0},\quad\mbox{for}\quad \ell>k.$$ Let $\gamma_{\ell} = \theta_{\ell}$ for $1 \leq \ell \leq k$ and $\theta_{\ell} = 0$ otherwise. Summing the outputs from $\ell=1$ to $L_1$ then completes the proof.
\end{proof}

\subsection{Proof of Lemma~\ref{overall-approx}}
To prove Lemma~\ref{overall-approx}, we apply a direct consequence of Theorem 1 in \cite{Schmidt-Hieber} (page 1891), incorporating the necessary adjustments detailed in the correction note \cite{Schmidt-Hieber_Vu_2024_Correction}. It shows that deep neural networks can effectively approximate finite compositions of H\"older smooth functions.
\begin{lemma}\label{dnn-approx-lemma}
Let $t_i \in \mathbb{N}$ and $\alpha_i > 0$, for $i = 0, \ldots, q$ and let $\alpha_i^*$ be defined as in \eqref{effect-smooth}. Set $m_i=\lceil (\alpha_i+t_i)\log_2 n/(2\alpha_i^*+t_i) \rceil$, $L'_i = 8 + (m_i+5)(1 + \lceil \log_2(t_i \vee \alpha_i) \rceil)$, and $s_i= 141(t_i + \alpha_i + 1)^{3+t_i}(m_i + 6).$ Let $\varphi^{*} \in \cG(q, \bs{d}, \bs{t}, \bm{\alpha}, K)$ and let $Q_0=1$, $Q_i=(2K)^{\alpha_i}$ for $i \in [q-1]$, and $Q_q=K(2K)^{\alpha_q}$.
For any $N_i \in \mathbb{N}$ such that $N_i \geq (\alpha_i+1)^{t_i} \vee (Q_i+1)e^{t_i}$, there exists $h \in \cF\big(\overline{L}, (d, 6rN, \ldots, 6rN, 1), s,\infty\big)$ with 
$$\overline L=3q + \sum_{i=0}^q L_i',\quad r= \max_{i=0,\ldots,q} d_{i+1}(t_i + \lceil \alpha_i \rceil),\quad N=\max_{i=0,\ldots,q}N_i,\quad s\leq\sum_{i=0}^qd_{i+1}(s_iN+4),$$
and a positive constant $C$ depending only on $q,{\bs d},{\bs t},{\bs\alpha},K$, such that
\[
\sup_{{\bs x}\in[0,1]^d}|\varphi^{*}({\bs{x}}) - h({\bs{x}})|\leq C\cro{\sum_{i=0}^{q}\left(N_i^{-\frac{\alpha_i}{t_i}}+N_in^{-\frac{\alpha_i+t_i}{2\alpha_i^*+t_i}}\right)^{\prod_{\ell=i+1}^{q}(\alpha_{\ell}\wedge1)}}.
\]
\end{lemma}

The proof of Lemma~\ref{overall-approx} relies moreover on the following stability lemma for compositions of H\"older smooth functions. It quantifies how input perturbations propagate through the compositional structure.
\begin{lemma}\label{stability-com}
For $0 \leq i \leq q$, let $h_i = (h_{i,1},\dots,h_{i,d_{i+1}})^{\top}$ be defined on $[0,1]^{d_i}$, where each $h_{i,j} \in \mathcal{H}_{t_i}^{\alpha_i}([0,1]^{t_i}, Q_i)$, $Q_i\geq1$ and $d_{q+1}=1$. Then, for any functions $u,v:\;\cD\to[0,1]^{d_0}$,
$$\big\|h_q\circ\ldots\circ h_0\circ u-h_q\circ\ldots\circ h_0\circ v\big\|_{\infty}\leq Q_{q}\left(\prod_{\ell=0}^{q-1}Q_{\ell}^{(\alpha_{\ell+1}\wedge1)}\right)\|u-v\|_{\infty}^{\prod_{\ell=0}^{q}(\alpha_{\ell}\wedge1)}.$$
\end{lemma}
\begin{proof}
The proof is based on a slight modification of the proof of Lemma~3 in \cite{Schmidt-Hieber}. Define $H_i=h_i\circ\ldots\circ h_0$. It follows that
\begin{align}
\|H_i\circ u-H_i\circ v\|_{\infty}&= \|h_i\circ H_{i-1}\circ u-h_i\circ H_{i-1}\circ v\|_{\infty}\nonumber\\
&\leq Q_i\left(\|H_{i-1}\circ u-H_{i-1}\circ v\|_{\infty}\right)^{\alpha_i\wedge1}.\label{com-b-i}
\end{align}
Applying \eqref{com-b-i} repeatedly and using the fact that $Q_i\geq1$ yields the result.
\end{proof}

\begin{proof}[Proof of Lemma~\ref{overall-approx}]
For any vector or matrix $\bf v$ and constants $\cK_1, \cK_2$, we write $\cK_1{\bf v} + \cK_2$ to denote the entrywise affine transformation $v_j \mapsto \cK_1 v_j + \cK_2$ for all entries $v_j$. For any function $f^*$ whose $j$-th component   $f^*_j = \varphi^{*} \circ \psi^*_{\mathbf{A},j}$ with $\psi^*_{\mathbf{A}}\in \cF_{\rho}(\beta,k,S_{\mathbf{A}})$, we can rewrite it as $$f_j^*=\varphi^{*}_{\beta}\circ\psi^*_{\beta,j},\quad\mbox{where}\ \varphi^{*}_\beta(\bs{z})=\varphi^{*}(\beta\bs{z}),\ \psi^*_{\beta,j}(\bs{x})=\frac{\psi^*_{\mathbf{A},j}(\bs{x})}{\beta}.$$
Consequently, $\psi^*_{\beta}\in\cF_0(1,k,S_{\mathbf{A}})$ if $\psi^*_{\mathbf{A}}\in\cF_0(\beta,k,S_{\mathbf{A}})$.
Lemma~\ref{gcn-approx} implies that there exists $g \in \cG(L_1,S_{\mathbf{A}})$ with $L_1\geq k$ such that for all $j=1,\ldots,n$,
\begin{equation}\label{gvn-part-error}
\|g_j-\psi^*_{\beta,j}\|_{\infty}\leq\rho<1.
\end{equation}
Given that $\psi^*_{\mathbf{A},j}(\bs{x}) \in [-M,M]^{d}$ for any $\bs{x} \in [0,1]^{n \times d}$, the $(j,i)$-th entry of $g(\bs{x}) \in \R^{n \times d}$, denoted by $g_j^i(\bs{x})$, satisfies $|g_j^i(\bs{x})| \leq M/\beta+\rho$. In the sequel, we denote $\mathcal{J} = M/\beta + \rho$ to simplify the notation. Observe that both $\bar\psi_{\beta,j}^* = \psi_{\beta,j}^*/(2\cJ) + 1/2$ and $\bar g = g/(2\cJ) + 1/2$ are functions on $[0,1]^{n\times d}$ with outputs in $[0,1]^{n\times d}$. 

Define $g^{*\beta}_{0}({\bs z})=g^*_{0}(\beta{\bs z})$. Then, under the representation $\varphi^{*} = g_q^* \circ \cdots \circ g_0^*$, where each $g^*_{i,j} \in \cH_{t_i}^{\alpha_i}([a_i,b_i]^{t_i},K)$ with $K\geq1$, we have $$\varphi^{*}_{\beta} = g_q^* \circ \cdots \circ g^{*\beta}_{0}.$$ Define 
$$h_0^{*\beta} = \frac{g_0^{*\beta}(2\cJ\cdot-\cJ)}{2K}+\frac{1}{2}=\frac{g_0^{*}(2\beta\cJ\cdot-\beta\cJ)}{2K}+\frac{1}{2},$$ and for $i=1,\ldots,q-1,$ $$h_i^* = \frac{g_i^*(2K\cdot - K)}{2K}+\frac{1}{2},\quad h_q^* = g_q^*(2K\cdot - K).$$ 
As a consequence, for $\phi=\psi_{\beta,j}^*$ or $g_j$, we have
\begin{equation}\label{equi-represent}
\varphi^{*}_{\beta}\circ\phi = g_q^* \circ \cdots \circ g_0^{*\beta}\circ\phi = h_q^* \circ \cdots \circ h_0^{*\beta}\circ\left(\frac{\phi}{2\cJ}+\frac{1}{2}\right).
\end{equation}
By definition, we have $h_{0,j}^{*\beta} \in\cH^{\alpha_0}_{t_0}([0,1]^{t_0},(2\beta\cJ)^{\alpha_0})$, $h_{i,j}^* \in\cH^{\alpha_i}_{t_i}([0,1]^{t_i},(2K)^{\alpha_i})$ for $i=1,\ldots,q-1$, and $h_{q,j}^* \in\cH^{\alpha_q}_{t_q}([0,1]^{t_q},K(2K)^{\alpha_q})$. Observe that $\beta\cJ\geq M\geq 1$ and $K\geq1$. Then, combining \eqref{gvn-part-error} with Lemma~\ref{stability-com} yields
\begin{align}
\|\varphi^{*}_{\beta}\circ g_j-\varphi^{*}_{\beta}\circ \psi^*_{\beta,j}\|_{\infty}&=\left\|g_q^* \circ \cdots \circ g_0^{*\beta}\circ g_j-g_q^* \circ \cdots \circ g_0^{*\beta}\circ \psi^*_{\beta,j}\right\|_{\infty}\nonumber\\
&=\left\|h_q^* \circ \cdots \circ h_0^{*\beta}\circ\bar g_j-h_q^* \circ \cdots \circ h_0^{*\beta}\circ \bar \psi^*_{\beta,j}\right\|_{\infty}\nonumber\\
&\leq(2\beta\cJ)^{\alpha_0}K\prod_{\ell=1}^{q} (2K)^{\alpha_{\ell}}\left\|\bar g_j-\bar \psi^*_{\beta,j}\right\|^{\prod_{\ell=0}^{q}(\alpha_\ell \wedge 1)}_{\infty}\nonumber\\
&\leq C\cdot\rho^{\prod_{\ell=0}^{q}(\alpha_\ell \wedge 1)},\label{appr-error-1}
\end{align}
where $C> 0$ depends only on $q$, $\beta$, $\bs{\alpha}$, $M$ and $K$.

Define $\tilde\varphi^{*}_{\beta} =\varphi^{*}_{\beta}(2\cJ\cdot - \cJ)$. Then $\tilde \varphi^{*}_{\beta} \in \cG(q, \bs{d}, \bs{t}, \bm{\alpha}, K')$ where $K'$ depends only on $\beta$, $M$, $K$ and $\bm{\alpha}$. According to Lemma~\ref{dnn-approx-lemma}, for any $N_i \in \mathbb{N}$ such that $N_i \geq (\alpha_i+1)^{t_i} \vee (Q_i+1)e^{t_i}$, there exists $\tilde h \in \cF\big(\overline{L}, (d, 6rN, \ldots, 6rN, 1), s\big)$ with
\begin{equation}\label{parameter-condition}
\overline{L} \asymp C_{q,\bs{t},\bs{\alpha}} \log_2 n, \quad 
r = C_{\bs{d},\bs{t},\bs{\alpha}}, \quad N =\max_{i=0\ldots,q}N_i, \quad 
s \leq C_{q,\bs{d},\bs{t},\bs{\alpha}}N \log_2 n,\end{equation}
such that
\begin{equation}\label{direct-result}
\sup_{\bs{x} \in [0,1]^d} |\tilde \varphi^{*}_{\beta}(\bs{x}) - \tilde h(\bs{x})| \leq C_{q,\bs{d},\bs{t},\bs{\alpha},K'}\cro{\sum_{i=0}^{q}\left(N_i^{-\frac{\alpha_i}{t_i}}+N_in^{-\frac{\alpha_i+t_i}{2\alpha_i^*+t_i}}\right)^{\prod_{\ell=i+1}^{q}(\alpha_{\ell}\wedge1)}},
\end{equation}
where all constants depend only on the parameters appearing in their subscripts. The bound \eqref{direct-result} further implies that if we define $h = \tilde h(\cdot/(2\cJ)+1/2)$, then
\begin{align*}
\|h\circ g_j-\varphi^{*}_{\beta}\circ  g_j\|_{\infty}&=\|\tilde h\circ\bar g_j-\tilde \varphi^{*}_{\beta}\circ\bar  g_j\|_{\infty}\\
&\leq C_{q,\bs{d},{\bs t},{\bs\alpha},K'}\cro{\sum_{i=0}^{q}\left(N_i^{-\frac{\alpha_i}{t_i}}+N_in^{-\frac{\alpha_i+t_i}{2\alpha_i^*+t_i}}\right)^{\prod_{\ell=i+1}^{q}(\alpha_{\ell}\wedge1)}}.
\end{align*}
Observe that the function $\zeta(\bs{z})=\bs{z}/(2\cJ)+1/2$ for $\bs{z}=(z_1,\ldots,z_d)^{\top}\in\R^d$ can be realized via a ReLU network where all parameters are bounded in $[-1,1]$. The construction proceeds coordinate-wise. For each $z_i$ with $i = 1, \ldots, d$, we implement the transformation using two layers. In the first layer, we use $2\lceil1/\cJ\rceil+1$ neurons to compute
\begin{align*}
a_1^{(i)}=\cdots=a_{\lceil1/\cJ\rceil}^{(i)} = \operatorname{ReLU}(1 \cdot z_i + 0),\quad a_{\lceil1/\cJ\rceil+1}^{(i)}=\cdots=a_{2\lceil1/\cJ\rceil}^{(i)} = \operatorname{ReLU}(-1 \cdot z_i + 0), \end{align*}
and $$a_{2\lceil1/\cJ\rceil+1}^{(i)} = \operatorname{ReLU}(0 \cdot z_i + 0.5).$$
In the first layer, we use in total 
$\cM_1=2d(2\lceil1/\cJ\rceil+1)$ parameters. In the second layer, we proceed
\begin{align*}
\zeta_i &= \frac{1}{2}\sum_{k=1}^{\lceil1/\cJ\rceil-1}a_{k}^{(i)}+\left(\frac{1}{2\cJ}-\frac{\lceil1/\cJ\rceil-1}{2}\right)a_{\lceil1/\cJ\rceil}^{(i)}\\
&+\left(-\frac{1}{2}\right) \sum_{k=\lceil1/\cJ\rceil+1}^{2\lceil1/\cJ\rceil-1}a_{k}^{(i)}+\left(-\frac{1}{2\cJ}+\frac{\lceil1/\cJ\rceil-1}{2}\right) \cdot a_{2\lceil1/\cJ\rceil}^{(i)} + 1 \cdot a_{2\lceil1/\cJ\rceil+1}^{(i)}.    
\end{align*}
This yields $\zeta_i = z_i/(2\cJ) + 1/2$ for all $z_i \in \R$, with the networks parameters all lying within $[-1,1]$. Observe that the output layer from the preceding construction can be combined with the input layer of the network $\cF\big(\overline{L}, (d, 6rN, \ldots, 6rN, 1), s\big)$. This yields 
\begin{equation}\label{construct-network}
h \in \cF\big(\overline{L}+1, (d,(2\lceil1/\cJ\rceil+1)d, 6rN, \ldots, 6rN, 1), \overline s\big),    
\end{equation}
where $$\overline s=s +\cM_1 + ((2\lceil1/\cJ\rceil+1)d+1)6rN,$$ $\overline{L}$, $r$, and $s$ are as specified in \eqref{parameter-condition}. 

Consequently,
\begin{align}
\|h \circ g_j - \varphi^{*}_{\beta} \circ g_j\|_{\infty} \leq C_{q,\bs{d},\bs{t},\bs{\alpha},K'}\cro{\sum_{i=0}^{q}\left(N_i^{-\frac{\alpha_i}{t_i}}+N_in^{-\frac{\alpha_i+t_i}{2\alpha_i^*+t_i}}\right)^{\prod_{\ell=i+1}^{q}(\alpha_{\ell}\wedge1)}}. \label{appr-error-2}
\end{align}
Combining \eqref{appr-error-1} with \eqref{appr-error-2}, we find that there exists a network $h$ in the class \eqref{construct-network} such that
\begin{align*}
\|h\circ g_j-\varphi^{*}_{\beta}\circ\psi^*_{\beta,j}\|_{\infty}&\leq\|h\circ g_j-\varphi^{*}_{\beta}\circ  g_j\|_{\infty}+\|\varphi^{*}_{\beta}\circ g_j-\varphi^{*}_{\beta}\circ \psi^*_{\beta,j}\|_{\infty}\\
&\leq C_{q,\beta,\bs{d},{\bs t},{\bs\alpha},M,K}\cro{\sum_{i=0}^{q}\left(N_i^{-\frac{\alpha_i}{t_i}}+N_in^{-\frac{\alpha_i+t_i}{2\alpha_i^*+t_i}}\right)^{\prod_{\ell=i+1}^{q}(\alpha_{\ell}\wedge1)}\hspace{-3pt}+\rho^{\prod_{i=0}^{q}(\alpha_{i}\wedge1)}}.
\end{align*}
Define $h_F=(h\vee{-F})\wedge F$ truncating the network output to $[-F,F]$. Observe that with $F\geq K$, 
\begin{align*}
\|h_F\circ g_j-\varphi^{*}_{\beta}\circ\psi^*_{\beta,j}\|_{\infty}&\leq\|h\circ g_j-\varphi^{*}_{\beta}\circ\psi^*_{\beta,j}\|_{\infty}\\
&\leq C_{q,\beta,\bs{d},{\bs t},{\bs\alpha},M,K}\cro{\sum_{i=0}^{q}\left(N_i^{-\frac{\alpha_i}{t_i}}+N_in^{-\frac{\alpha_i+t_i}{2\alpha_i^*+t_i}}\right)^{\prod_{\ell=i+1}^{q}(\alpha_{\ell}\wedge1)}\hspace{-10pt}+\rho^{\prod_{i=0}^{q}(\alpha_{i}\wedge1)}}.
\end{align*}
This completes the proof.
\end{proof}

\subsection{Proof of Lemma~\ref{t-s-app-error}}
\begin{proof}
Let $d_{\max} = d_{\bT}+d_{\bS_{\mathbf{A}}}$ and $\cA=\|\bT\|_{1,\infty}\vee\|\bS_{\mathbf{A}}\|_{1,\infty}$. Given that $\|\bT - \bS_{\mathbf{A}}\|_{\operatorname{F}} \leq \tau$ and each row has at most $d_{\max}$ nonzero entries, for any $i = 1, \ldots, n$, the Cauchy-Schwarz inequality gives
\begin{equation}\label{row-b-t-s}
\|\bT - \bS_{\mathbf{A}}\|_{1,\infty} \leq \max_{i\in [n]} \, \sum_{j=1}^n |(\bT - \bS_{\mathbf{A}})_{i,j}| \leq \sqrt{d_{\max}} \max_{i\in [n]} \,  \sqrt{\sum_{j=1}^n |(\bT - \bS_{\mathbf{A}})_{i,j}|^2} \leq \sqrt{d_{\max}}\tau. \end{equation}
We now prove that for any positive integer $i$ and row index $j \in\{1,\ldots,n\}$,
\[
\left|\big((\bT^i-\bS_{\mathbf{A}}^i){\bs x}\big)_{j,\cdot}\right|_{\infty} \leq i\sqrt{d_{\max}}\tau\cA^{i-1}.
\]
Decomposing
\(
\bT^i - \bS_{\mathbf{A}}^i = \sum_{\ell=0}^{i-1}\bT^{\ell}(\bT-\bS_{\mathbf{A}})\bS_{\mathbf{A}}^{i-1-\ell},
\)
yields
\begin{align}
\left|\big((\bT^i-\bS_{\mathbf{A}}^i){\bs x}\big)_{j,\cdot}\right|_{\infty}=\left|{\bs e}_j^{\top}(\bT^i-\bS_{\mathbf{A}}^i){\bs x}\right|_{\infty}\leq\sum_{\ell=0}^{i-1}\left| {\bs e}_j^\top\bT^{\ell}(\bT-\bS_{\mathbf{A}})\bS_{\mathbf{A}}^{i-1-\ell}{\bs x}\right|_{\infty},
    \label{eq.94fbwd}
\end{align}
where ${\bs e}_j$ represents the $j$-th standard basis vector. For each $\ell$, let ${\bs v}^{(\ell)} ={\bs e}_j^\top \bT^{\ell}\in\R^{1\times n}$. It follows that for each term
\begin{align}
\left|{\bs e}_j^\top \bT^{\ell}(\bT-\bS_{\mathbf{A}})\bS_{\mathbf{A}}^{i-1-\ell}{\bs x}\right|_{\infty}&\quad \leq \left|{\bs v}^{(\ell)}(\bT-\bS_{\mathbf{A}})\bS_{\mathbf{A}}^{i-1-\ell}\right|_1 \cdot \|{\bs x}\|_{\infty} \nonumber\\
&\quad \leq\left|{\bs v}^{(\ell)}(\bT-\bS_{\mathbf{A}})\bS_{\mathbf{A}}^{i-1-\ell}\right|_1.\label{1-norm-b}
\end{align}
From \eqref{1-norm-b}, we can further bound the right-hand side 
\begin{align}
\left|{\bs v}^{(\ell)}(\bT-\bS_{\mathbf{A}})\bS_{\mathbf{A}}^{i-1-\ell}\right|_1 
&= \sum_{m=1}^n \left|\sum_{h=1}^n  {\bs v}_h^{(\ell)}\cro{\sum_{p=1}^n (\bT-\bS_{\mathbf{A}})_{h,p}(\bS_{\mathbf{A}}^{i-1-\ell})_{p,m}}\right|\nonumber\\
&\leq \sum_{m=1}^n \sum_{h=1}^n\left|{\bs v}_h^{(\ell)}\right|\cro{ \sum_{p=1}^n \left|(\bT-\bS_{\mathbf{A}})_{h,p}\right| \cdot \left|(\bS_{\mathbf{A}}^{i-1-\ell})_{p,m}\right|}\nonumber\\
&\leq\left|{\bs v}^{(\ell)}\right|_1\cdot\|\bT-\bS_{\mathbf{A}}\|_{1,\infty}\cdot\|\bS_{\mathbf{A}}^{i-1-\ell}\|_{1,\infty}\nonumber\\
&\leq \sqrt{d_{\max}}\tau\cdot\left|{\bs e}_j^\top \bT^{\ell}\right|_1\cdot\|\bS_{\mathbf{A}}^{i-1-\ell}\|_{1,\infty},\label{bound-ej}
\end{align}
where the last inequality follows from \eqref{row-b-t-s}. We claim that $|{\bs e}_j^{\top} \bT^{\ell}|_1 \leq \|\bT\|_{1,\infty}^{\ell}$ for all $\ell \geq 0$. When $\ell = 0$, the bound is immediate since $|{\bs e}_j^{\top}|_1 = 1 \leq \|\bT\|_{1,\infty}^0 = 1$. Now assume the claim holds for some $\ell - 1 \geq 0$, i.e.,
\[
\left|{\bs e}_j^{\top}\bT^{\ell-1} \right|_1=\sum_{h=1}^{n}|{\bs v}_h^{(\ell-1)}| \leq \|\bT\|_{1,\infty}^{\ell-1}.
\]
We show that it also holds for $\ell$. Observe that
\[
\left|{\bs e}_j^{\top}\bT^{\ell} \right|_1 = \left| \left({\bs e}_j^{\top}\bT^{\ell-1} \right)\bT \right|_1=\sum_{m=1}^n\left|\sum_{h=1}^n{\bs v}_h^{(\ell-1)} \bT_{h,m}\right| 
\leq\sum_{h=1}^n\left|{\bs v}_h^{(\ell-1)}\right|\left(\sum_{m=1}^n\left|\bT_{h,m}\right|\right)\leq \|\bT\|_{1,\infty}^{\ell}.
\]
This completes the induction and proves the claim. Similarly, we can show that $\|\bS_{\mathbf{A}}^{i-1-\ell}\|_{1,\infty}\leq\|\bS_{\mathbf{A}}\|_{1,\infty}^{i-1-\ell}$, for $0\leq\ell\leq i-1$. Plugging these two bounds into \eqref{bound-ej} yields
\begin{align*}
\left|{\bs e}_j^{\top}\bT^{\ell}(\bT-\bS_{\mathbf{A}})\bS_{\mathbf{A}}^{i-1-\ell}{\bs x}\right|_{\infty}\leq\sqrt{d_{\max}}\tau\cA^{i-1}.
\end{align*}
Together with \eqref{eq.94fbwd}, this gives
$$\left|\big((\bT^i-\bS_{\mathbf{A}}^i){\bs x}\big)_{j,\cdot}\right|_{\infty}\leq i\cA^{i-1}\sqrt{d_{\max}}\tau.$$
Hence, we conclude for any $j\in[n]$,
\begin{align*}
\left|\Big(\sum_{i=1}^{L_1}\theta_i\big(\bT^i-\bS_{\mathbf{A}}^i\big){\bs x}\Big)_{j,\cdot}\right|_{\infty}&\leq\sum_{i=1}^{L_1}|\theta_i|\left|\big((\bT^i-\bS_{\mathbf{A}}^i){\bs x}\big)_{j,\cdot}\right|_{\infty}\\ 
&\leq\sqrt{d_{\max}}\tau\left(\sum_{i=1}^{L_1}|\theta_i|\cA^{i-1}i\right).
\end{align*}
\end{proof}

\bibliographystyle{plain}
\bibliography{bib}  

\end{document}